\documentclass[twocolumn,letterpaper]{IEEEtran}
\usepackage[left=0.75in,top=0.72in,right=0.75in,bottom=0.78in]{geometry}
\usepackage{psfrag}
\usepackage{amsmath}
\allowdisplaybreaks 

\usepackage{hyperref}
\usepackage{amssymb}
\usepackage{graphicx}
\usepackage{subfigure}
\usepackage{verbatim}
\usepackage[thinlines,thicklines]{easybmat}
\usepackage{latexsym}
\usepackage[dvipsnames]{xcolor}
\usepackage{cite}
\usepackage{bm}
\usepackage{stmaryrd}
\usepackage{multirow}
\usepackage{textcomp}
\usepackage[linesnumbered,ruled,vlined]{algorithm2e}
\usepackage{mathtools}
\usepackage{bookmark}
\usepackage{dsfont}
\usepackage{csquotes}
\usepackage{tikz}
\usetikzlibrary{backgrounds, fit}
\definecolor{emerald}{rgb}{0.31, 0.78, 0.47}

\newtheorem{thm}{Theorem}
\newtheorem{lem}{Lemma}

\newtheorem{rem}{Remark}

\newcommand{\qed}{\hfill \mbox{~\rule[-1pt]{6pt}{6pt}}}
\newenvironment{proof}{\noindent { \textit{Proof}:} }{\qed \\}

\def\atan2{\mathrm{atan2}}

\begin{document}

\title{ Tactile Sensing-based Secure and Gentle\\  Grasping of Unknown Objects}


\author{Feng Han\thanks{F. Han is with the Department of Mechanical Engineering, New York Institute of Technology, Old Westbury, NY 11568 USA (email: {fhan03@nyit.edu})} }

\maketitle

\begin{abstract}

Reliable robotic manipulation requires maintaining a firm grasp while strictly avoiding both contact slip and structural damage to the object. Satisfying these inherent \textit{secure-and-gentle} grasping constraints is challenging for unknown objects. We present a novel online data-driven predictive control framework for safe dynamic robotic grasping of unknown objects. The proposed system features a pair of compact, single-channel tactile-sensing modules that characterize the gripper-object dynamic interaction in real time, entirely bypassing the need for prior knowledge or offline training. A finite-state-machine-based grasping force probing mechanism is developed to establish the initial grasping force and explicitly identify grasping force constraints. Through online dynamics learning, a Gaussian Process (GP) model captures gripper-object interaction dynamics. The model uncertainties are further used to form robust grasping constraints, which are embedded into a model predictive control scheme. We demonstrate error convergence in a probabilistic sense. The proposed control framework runs at 100~Hz in real time to dynamically manipulate and handle objects rather than static holding. Extensive experimental evaluations across diverse contact surface topologies and various types of objects validate superior performance and efficiency compared to baseline strategies. The proposed method is able to manipulate objects ranging from delicate items, such as raw tofu and marshmallows, to rough objects, such as rocks.
\end{abstract}


\section{Introduction}
\label{Sec_Introduction}


Robotic grasping and manipulation offer great efficiency and dexterity across a broad spectrum of automated tasks. Notable examples include automated fruit harvesting in modern agriculture~\cite{Lehnert2017RAL, Uppalapati2020Berry, Ansari2026Tomato}, dynamic food sorting and picking in the food industry~\cite{Zhongkui2021Circular, Baggetta2025Food, Zhang02092025}, and the delicate and ultra-precise grasping and transporting of live organs during robotic surgery~\cite{Roh2021Gripper, Mykhailyshyn2024Tissue}. With prior knowledge and offline training, these grasping tasks can be achieved safely. However, reliable robotic grasping of unknown objects remains one of the most fundamental yet challenging objectives in autonomous robotics~\cite{Xu2024, Palleschi2023TRO, phanomchoeng2023grasping}. When a robotic gripper interacts with unseen objects, the \textit{inherent but mathematically unknown grasping force constraints} must be satisfied. The grasping force should provide secure but delicate contact to avoid contact slip, damage, or breakdown~\cite{Xu2024, Dong2019ICRA, Yunha2025TMECH}. These constraints are related to the object's mass, material compliance, friction, and contact interface topology, etc., which are completely unknown. Even the same object under different contact interface topologies with the same gripper will result in different interaction dynamics~\cite{Romano2011TRO}. For instance, grasping a metal cube along diagonal portions or parallel edges requires significantly different forces.


Human grasping and manipulation of unknown objects take advantage of the coordination of vision and haptic-based force feedback~\cite{robotics9030050, Palleschi2023TRO, kang2026Force, Zheng2022Grasping}. The advanced biostructure of human fingers is capable of adapting to the object's geometry, size, surface topologies, etc. When handling objects out of the line of sight, human manipulation can still achieve success by relying fundamentally on blind haptic reflexes~\cite{Johansson2009Coding, Zheng2022Grasping}. During this process, humans learn the unseen object's mechanical properties, \enquote{feel} the interaction mechanism online, and then adjust the applied force to reach a secure-and-gentle grasp~\cite{Zheng2022Grasping, Romano2011TRO}.

For robotic grasping, accurate sensing plays a significant role in characterizing the robotic gripper-object interaction mechanism and dynamic grasping control. Sensing technology can be categorized into contactless (i.e., vision-based) and contact (i.e., tactile/haptic-based) groups. Vision-based machine learning agents can recognize and label object types and prescribe a grasping force~\cite{Beddow2025TMECH, kang2026Force, Yuan2017GelSight, Xu2024, Yunha2025TMECH, Zhao2025TASE}. However, such pipelines heavily depend on a pre-trained network over massive object datasets~\cite{calli2015ycb, kasaei2023mvgrasp}. Vision systems inherently suffer from critical degradation in occluded spaces and an inability to monitor visually unrecognized micro-slip at the interface~\cite{Hu2021RAL, Li2025Grasping, Dong2019ICRA, Yunha2025TMECH}. Furthermore, the computational cost of vision-based grasping is usually demanding, and the real-time control frequency is relatively low~\cite{Tian2019Manipulation, Xu2024}. For instance, in~\cite{Tian2019Manipulation}, video processing is used to handle tactile sensing feedback, and the control frequency can only reach $1$~Hz. Tactile sensing provides an alternative by perceiving physical interactions through direct contact~\cite{Zhou2024SoftGripper, Beddow2025TMECH, Romano2011TRO, Lu2022RAL}. Nevertheless, tactile sensing requires accurate calibration, and haptic feedback is inherently susceptible to distortions caused by complex contact surface topologies, localized stress concentrations, and vibrations~\cite{Friedl2021Evaluation, Piacenza2020TMECH}. The uncertainties in sensing propagate into interaction dynamics characterization and control design. Under dynamic grasping with a time-varying reference force, the control system design must respect secure-and-gentle grasping constraints~\cite{Xu2024, Tian2019Manipulation}.


To overcome the challenges in unknown object grasping, we propose a multi-phase, vision-free interaction dynamics learning and predictive control architecture. The proposed tactile mechatronics can adapt to different contact interface topologies and detect potential object slip or yielding while preserving the fidelity of the measurement. Additionally, a Gaussian Process (GP)-based calibration model accurately captures the interaction force from raw measurements with bounded uncertainties. The case-specific interaction mechanism is characterized online through a haptic-based finite state machine (FSM). The object geometry, stiffness, and grasping force constraints are identified systematically. Through online dynamics learning, a GP-based non-parametric interaction dynamics model is constructed. Together with the identified constraints, we design a robust model predictive control (MPC) scheme by incorporating both calibration and modeling uncertainties, ensuring grasping force control convergence and strict constraint satisfaction. The proposed control is data-efficient and does not depend on prior knowledge or datasets. Compared to the low-frequency control of vision-tactile sensing, the real-time control frequency of the proposed control reaches $100$~Hz.

\subsection{Related Work}
\label{sec:related_work}


Robotic grasping and manipulation encompass multifaceted innovations, including advanced gripper design, sensing and perception, actuation, modeling, and control. In the following, we review the most relevant work in the literature.

\subsubsection{Robotic Grasping}


Robotic grasping has been used in various applications due to its dexterity, ranging from fruit harvesting in agriculture~\cite{Lehnert2017RAL, Uppalapati2020Berry, Ansari2026Tomato} and food-feeding in caregiving~\cite{robotics9030050, Keely2024IROS} to raw material picking in manufacturing. For different applications, various gripper mechanisms have been developed, including rigid~\cite{Yunha2025TMECH, Romano2011TRO, Hua2023Grasping} and soft grippers~\cite{Zhou2024SoftGripper, Lu2022RAL}. For instance, a compliant three-finger gripper is designed with embedded strain gauges in~\cite{Beddow2025TMECH}, while the work in~\cite{Ansari2026Tomato} designs a hybrid robotic gripper combining six soft auxetic fingers with a rigid exoskeleton. Soft grippers can provide direct lifting force, whereas parallel-jaw grippers usually apply a normal contact force, with lifting force generated indirectly through friction~\cite{Romano2011TRO}.

\subsubsection{Tactile Sensing}

Emerging tactile sensing technology provides direct information about the gripper-object contact interface at the interface~\cite{Calandra2017Feeling, Si2022Grasp, Yin2022Multimodal}, which resembles advanced human grasping feedback. To provide rich contact information, multi-tactile sensing arrays have been designed and integrated into gripper systems~\cite{Romano2011TRO, Xu2024, kang2026Force, Yin2022Multimodal}. However, the high nonlinearity of tactile sensors requires advanced calibration to address hysteresis, dead zones, saturation, etc. A popular approach in the literature is to use machine learning-based methods to interpret raw measurements and bypass the calibration process. Tactile sensing, in principle, measures pressure and is thus sensitive to the gripper-object contact profile, area size, slip, and vibrations, leading to inherent uncertainties and case-specific feedback~\cite{Friedl2021Evaluation, Piacenza2020TMECH}. Integrated with vision systems, visual-tactile sensing combines the advantages of both modalities, providing rich feedback information for decision-making and control design~\cite{Xu2024, Yunha2025TMECH, Tian2019Manipulation}. However, the computational costs become very demanding, resulting in a low real-time control frequency (i.e., less than $50$~Hz).

\subsubsection{Grasping Control}

A critical component in robotic grasping is contact force design and control. Since interaction dynamics are case-specific, most existing works either ignore the hidden nonlinear gripper-object interaction dynamics or embed them implicitly into the learning process~\cite{Beddow2025TMECH}. The most popular approach designs the grasping force based on prior datasets~\cite{Yunha2025TMECH, calli2015ycb, kasaei2023mvgrasp}. Empowered by machine learning techniques, robotic grasping can perceive and recognize objects, with grasping forces designed accordingly to enable object-specific strategies. In~\cite{Zhou2024SoftGripper, Taunyazov2020Event}, raw tactile feedback is used for object recognition within a neural network. However, when grasping unknown, delicate objects, failures are commonly observed. Secure-and-gentle grasping requires adequate force without aggressive contact. The boundary grasping forces vary depending on the object's mass, compliance, surface roughness, and contact topology~\cite{Zheng2022Grasping, Romano2011TRO, kang2026Force}, and they must be identified online. Meanwhile, most works in the literature focus on object grasp-and-hold tasks~\cite{Li2025Gentle, Romano2011TRO} rather than dynamic grasping force control, which is critical in the presence of environmental uncertainties.


Additionally, PID, MPC~\cite{Xu2024, Tian2019Manipulation}, and machine learning-based control strategies~\cite{Zhou2024SoftGripper, Beddow2025TMECH, Chen2024Robot, Zeng2023Multifingered, kang2026Force} have been proposed in the literature. Physical model-based control primarily focuses on gripper kinematic movement control rather than explicit gripper-object interaction dynamics~\cite{Xu2024, Zhou2024SoftGripper}. The MPC approach enforces constraints on gripper movement and contact force to ensure safe and secure grasping. In~\cite{Xu2024}, an MPC is designed using gripper kinematics and vision-tactile sensing feedback. However, the constraints are only imposed on the kinematic movement of the gripper. Because the interaction dynamics remain unknown, there are no constraints involving the interaction force. Machine learning-based control is also popular given the complicated gripper-object interaction dynamics. In~\cite{Zhou2024SoftGripper}, raw tactile sensing data are directly processed in a neural network without calibration. End-to-end reinforcement learning neglects interaction dynamics, and the grasping force is designed through reward shaping~\cite{Beddow2025TMECH, Chen2024Robot}. Learning from human grasping is another alternative~\cite{Zeng2023Multifingered, kang2026Force}. However, these learning mechanisms rely on offline training and have low transparency and adaptation to unknown objects.


\subsection{Contributions and Organization}
The major contributions of the proposed online grasping dynamics control framework are summarized as follows:
\begin{enumerate}
    \item The proposed tactile sensing module is compact, containing only a single-channel thin-film force-sensitive resistor (FSR). The module does not require an advanced tactile sensing array design nor significant alteration of the grippers and supportive mechatronics. Through GP-based calibration, the tactile sensing provides high-fidelity measurements regardless of the tactile contact interface topologies.

    \item  The proposed novel finite-state machine safely and explicitly identifies the object's geometry, compliance, and object-specific grasping force constraints online in real time, which lays a solid foundation for secure and gentle grasping control. During this process, no prior knowledge or vision inspections are needed.

    \item We explicitly learn the gripper-object interaction dynamics online. An MPC-based controller ensures force control performance and robust constraint enforcement within the information-theoretic envelope of the dynamics learning uncertainties and tactile sensing uncertainties. We provide a comprehensive theoretical proof to establish probabilistic stability and convergence, bounding the true but unknown force control error within a compact residual set. To the best of the authors' knowledge, no previous work has explicitly identified and incorporated the secure-and-gentle grasping force constraints of unknown objects into control design.

    \item We validate the framework through extensive experiments, demonstrating superior control accuracy, slip prevention, and structural protection over baseline architectures across diverse types of objects, ranging from raw tofu to metal blocks. The control system reaches a real-time $100$~Hz frequency.
\end{enumerate}

The remainder of this paper is organized as follows. Sec.~\ref{Sec_overview} outlines the system modeling and formalizes the robust dynamic grasping problem. Sec.~\ref{Sec_tactile_sensing} details the tactile sensor mechatronics design and calibration. Sec.~\ref{Sec_state_machine} identifies the initial grasping force and inherent secure-and-gentle grasping force constraints. Sec.~\ref{Sec_MPC_control} formulates the online interaction dynamics learning and robust MPC scheme and delivers the analysis for error convergence. Sec.~\ref{Sec_Results} presents the experimental setup, sensor configurations, and performance benchmarks. Finally, limitations and future research avenues are discussed in Sec.~\ref{Sec_discussion}, and concluding remarks are provided in Sec.~\ref{Sec_conclusion}.

\begin{figure}
 \centering
 \includegraphics[width=6cm]{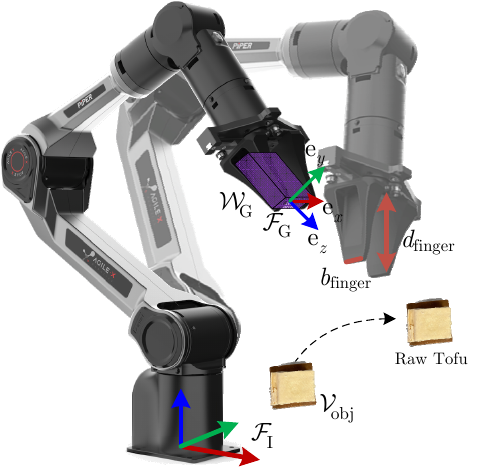}
 \vspace{-2mm}
 \caption{Tactile sensing enabled dynamic grasping of unknown objects.}\label{Fig_Robot}
 \vspace{-2mm}
\end{figure}

\section{Overview}
\label{Sec_overview}

\subsection{System Configuration}


The robotic grasping system is shown in Fig.~\ref{Fig_Robot}, which includes a multi-degree-of-freedom (DOF) dexterous manipulator, a rigid two-finger gripper, and a tactile sensing system. The configuration variable of the grasping system is denoted as $\bm\Theta = [\bm\Theta_m^\top\; \theta_\text{G}]^\top \in \mathbb{R}^{m+1}$, where $\bm\Theta_m = [\theta_1\; \cdots\; \theta_m]^\top$ contains the $m$-joint angles of the manipulator and $\theta_\text{G}$ is the angular position of the gripper motor. To characterize the dynamic grasping process, we establish two reference frames: the inertial frame $\mathcal{F}_\text{I}$ and the gripper body frame $\mathcal{F}_\text{G}$. The inertial frame is fixed at the base of the manipulator, and $\mathcal{F}_\text{G}$ is attached to the geometric center of the parallel fingertips of the gripper. The spatial orientation of $\mathcal{F}_\text{G} = \{\mathbf{e}_x, \mathbf{e}_y, \mathbf{e}_z\}$ is configured such that $\mathbf{e}_x$ is collinear with the line connecting the actuation centers of the left and right gripper fingers, $\mathbf{e}_z$ points outward from the fingers, and $\mathbf{e}_y$ completes the right-handed orthogonal coordinate system. Let $d_{\text{finger}}$ denote the maximum vertical height of the fingers along the local $\mathbf{e}_z$-axis, and $b_{\text{finger}}$ represent the longitudinal length of the fingers along the $\mathbf{e}_y$-axis.

The spatial pose of the moving gripper frame $\mathcal{F}_{\text{G}}$ relative to $\mathcal{F}_{\text{I}}$ is determined by the manipulator's DH parameters. The forward kinematics of the position of $\mathcal{F}_{\text{G}}$ is given by
\begin{subequations}
\label{Eq_Arm_Kine}
\begin{align}
 \bm{p}_{\text{G}} &= \bm {p}_{\text{G}}(\bm \Theta_m), \\
 \dot{\bm{p}}_{\text{G}} &= \mathbf J_m\dot{\bm \Theta}_m \triangleq \frac{\partial\bm {p}_{\text{G}}(\bm \Theta_m)}{\partial \bm \Theta_m} \dot{\bm \Theta}_m,
 \end{align}
\end{subequations}
where $\mathbf J_m = \frac{\partial\bm {p}_{\text{G}}(\bm \Theta_m)}{\partial \bm \Theta_m} \in \mathbb{R}^{3\times 6}$ is the position Jacobian matrix of the gripper. The gripper's movement in $\mathcal{F}_\text{G}$ is expressed as
\begin{equation}
\label{Eq_Gripper_Width}
\begin{bmatrix}
 x \\
 y \\
 z
\end{bmatrix} =
\begin{bmatrix}
 f(\theta_\text{G}) \\
 0 \\
 0
\end{bmatrix},\;
 \theta_\text{G}\in[\theta_{\text{G},\min}, \theta_{\text{G},\max}],
\end{equation}
where $f: \mathbb R \rightarrow \mathbb R^+$ captures the linear motion of the gripper. The local operational space bounded between the fingers, denoted as $\mathcal{W}_{\text{G}} \subset \mathbb{R}^3$, is formally defined by the compact set
\begin{equation}
   \mathcal{W}_{\text{G}} = \left\{\bm w= \begin{bmatrix} x \\ y \\ z \end{bmatrix} \middle| \begin{aligned} -\frac{1}{2}f(\theta_{\text{G}}) &\leq x \leq \frac{1}{2}f(\theta_{\text{G}}) \\ -\frac{1}{2}b_{\text{finger}} &\leq y \leq \frac{1}{2}b_{\text{finger}} \\ -d_{\text{finger}} &\leq z \leq 0  \end{aligned} \right\},
\end{equation}
which can be further mapped back into $\mathcal{F}_\text{I}$ through the coordinate transformation
\begin{equation}\label{Eq_workspace_inertial}
   ^{\text{I}}\mathcal{W}_{\text{G}} = \left\{ \bm{p}\mid \bm{p} = \bm{p}_{\text{G}} + {^{\text{I}}\mathbf{R}_{\text{G}}} \bm w, \bm w \in \mathcal{W}_{\text{G}}(\theta_{\text{G}}) \right\},
\end{equation}
where $^{\text{I}}\mathbf{R}_{\text{G}}(\bm{\Theta}_m) \in \text{SO}(3)$ represents the rotation matrix defining the spatial orientation of $\mathcal{F}_{\text{G}}$ relative to $\mathcal{F}_{\text{I}}$.


The gripper-object interaction dynamics registered at the contact interface are expressed in $\mathcal F_\text{G}$ as
\begin{equation}\label{Eq_Dyna_Inter}
\bm {F} =
\begin{bmatrix}
  F_x \\
  F_y \\
  F_z
\end{bmatrix}
=
\begin{bmatrix}
 1 \\
 \nu_y \\
 \nu_z
\end{bmatrix} f_\text{inter}(x,\dot x, \ddot x), \quad x\leq W_\text{obj},
\end{equation}
where $W_\text{obj}$ is the object width (free of deformation) in the $\mathbf{e}_x$ direction and $F_x \in \mathbb{R}^+$ defines the compressive normal grasping force along $\mathbf e_x$. $F_y, F_z\in \mathbb{R}$ represent the lateral and vertical static friction forces with coefficients $\nu_y, \nu_z>0$. The interaction dynamics~\eqref{Eq_Dyna_Inter} simplify the contact mechanism while preserving the complex, nonlinear characteristics of the interactive object through the single hidden scalar function $f_{\text{inter}}: \mathbb{R}^3 \to \mathbb{R}$.

\begin{figure*}
  \centering
   \includegraphics[width=16cm]{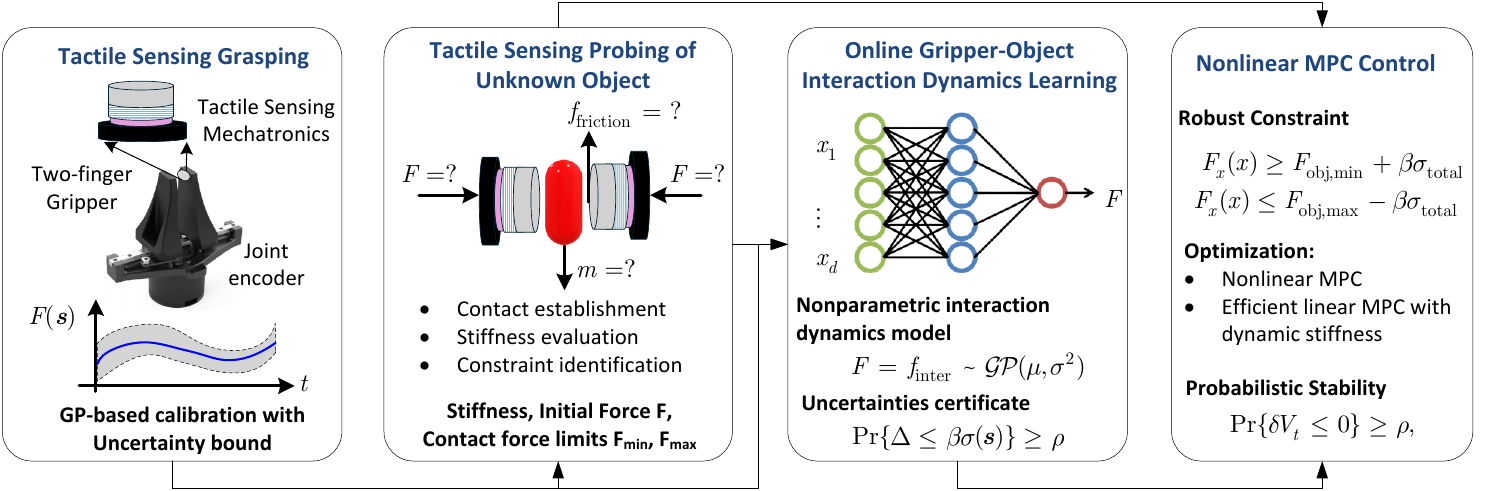}
   \caption{Overview of the proposed vision-free active grasping framework.}
  \label{Fig_Overview}
  \vspace{-2mm}
\end{figure*}

%

Let $\nu \in \mathbb{R}^+$ be the isotropic maximum static friction coefficient, satisfying $\sqrt{\nu_y^2 + \nu_z^2} \leq \nu$. To prevent slippage across the contact interface while preventing object damage during dynamic grasping, the grasping force should satisfy the following constraints:
\begin{subequations}\label{Eq_Consrtrains}
\begin{align}
    &\mathds{C}_\text{secure}: F_x > F_{\text{obj}, \min} \triangleq \frac{\sqrt{F_y^2 + F_z^2}}{\nu} = \frac{mg}{\nu},\\
   & \mathds{C}_\text{gentle}: F_x < F_{\text{obj}, \max},
  \end{align}
\end{subequations}
where $m \in \mathbb{R}^+$ represents the unknown mass, and $F_{\text{obj}, \min}, F_{\text{obj}, \max} \in \mathbb{R}^+$ define the unknown minimum and maximum grasping force thresholds before failure occurs.

\subsection{Problem Formulation}


The target of this work is to build a tactile grasping control system for unknown objects while respecting the unknown, underlying constraints. To establish mathematical tractability and define the scope of this work, we first make the following assumptions regarding the dynamic grasping process.
\begin{enumerate}
   \item [$\mathcal{A}_1$] Initially, the object resides within the workspace of the gripper, such that $\mathcal{V}_{\text{obj}} \cap\, ^{\text{I}}\mathcal{W}_{\text{G}} \neq \emptyset$, where $\mathcal{V}_{\text{obj}}$ in $\mathcal F_{\text{I}}$ represents the volume occupied by the object's geometry.


   \item [$\mathcal{A}_2$] The grasping system has zero \textit{a priori} knowledge of the object, including its mass $m$, friction characteristics, internal material stiffness $\kappa$, and surface topology. However, the object's mass and stiffness are bounded as $m \in (0, m_{\max}]$ and $\kappa \in [\kappa_{\min}, \kappa_{\max}]$.


   \item[$\mathcal{A}_3$] The tactile sensor only registers the single-channel interaction load along the $\mathbf{e}_x$ direction, despite the effective contact region potentially spanning a continuous surface area.

   \item [$\mathcal{A}_4$] The system imposes no geometric or distribution constraints on the contact topology between the object and the tactile sensor.
\end{enumerate}


The assumption that the object initially resides within the intersecting volume of the workspace is reasonable in the sense that tactile-sensing grippers can search for the object without vision and reconstruct its 3-dimensional structure~\cite{Zhang2026Object, Pan2026IROS, Comi2024DeepSDF}. Furthermore, $\mathcal{A}_2$ is made based on the grasping capability (the maximum grasping force) of the physical system. Meanwhile, $\mathcal{A}_3$ and $\mathcal{A}_4$ are realistic and practical. Due to the limited area of the gripper fingers, the tactile sensor must be compact and capable of adapting to various contact surface topologies. Given the above system setup and assumptions, we present the problem statement.


\textbf{Problem statement:} Given a completely unknown object and the tactile-sensing grasping system under Assumptions~$\mathcal{A}_1$ through~$\mathcal{A}_4$, this work aims to design an online machine learning-based dynamic grasping control framework to manipulate and transport the object while respecting the unknown constraints in~\eqref{Eq_Consrtrains}.

\begin{figure*}
  \centering
   \includegraphics[width=18cm]{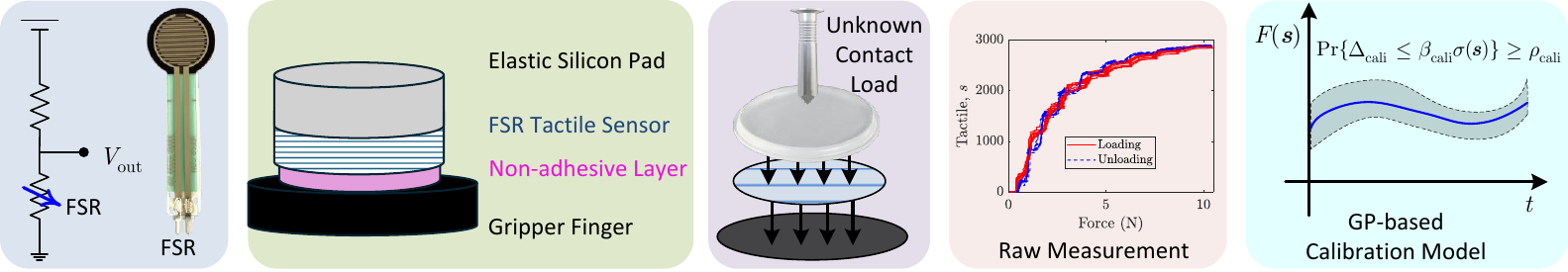}
   \vspace{-5mm}
   \caption{The tactile sensing module. The FSR sensor is placed into a voltage divider circuit. With the design, the point load will be distributed among the active area of the FSR sensor. The raw measurement is highly nonlinear. Through GP-based calibration, we map the raw data in real time to contact force, with bounded error.}
  \label{Fig_Gripper_Tacitle}
  \vspace{-3mm}
\end{figure*}

\subsection{Overview of the Proposed Approach}


To resolve the challenges of parametric uncertainties ($\mathcal{A}_2$), completely unobserved shear and friction states ($\mathcal{A}_3$), and unmodeled contact topologies ($\mathcal{A}_4$), we propose a vision-free active grasping architecture. The proposed approach tightly couples tactile mechatronics, constraint identification, online interaction dynamics characterization, and predictive control; see the schematic diagram in Fig.~\ref{Fig_Overview}. The framework is organized into three core pillars:
\begin{enumerate}

   \item \textbf{Tactile Sensing Mechatronics:} To interface directly with unknown object profiles and contact topologies ($\mathcal{A}_4$), we design a grasping interaction sensing module using affordable tactile sensors ($<\$5$). The designed sensing module automatically adapts to various contact types. The severe nonlinearities and cross-coupling inherent in raw sensor measurements are systematically resolved using a machine learning-based calibration.


   \item \textbf{Finite State Machine for Object Probing and Constraint Identification:} Because the minimum physical slip threshold $F_{\text{obj}, \min}$ and structural destruction limit $F_{\text{obj}, \max}$ are unknown parameters ($\mathcal{A}_2, \mathcal{A}_3$), an FSM is deployed to manage operational execution phases and identify these unknown constraints.


   \item \textbf{Online Interaction Dynamics Learning and MPC Control:} To accommodate unmodeled elastic material variations ($\mathcal{A}_2$), a data-driven non-parametric model learns the gripper-object interaction dynamics $f_{\text{inter}}$ online. This adaptive interaction model is embedded directly inside a nonlinear MPC engine alongside the identified constraints. The model uncertainty is explicitly incorporated into the control design to ensure robust predictive control.

\end{enumerate}


By integrating these three strategies, the proposed architecture bypasses the need for global vision networks or exact material profiles, achieving highly robust and safe dynamic grasping of unmodeled payloads.

\section{Gripper Tactile Sensing System}
\label{Sec_tactile_sensing}


In this section, we present the mechatronic design of a compact tactile sensing system for grasping control. The tactile sensing system provides real-time force feedback on the gripper-object interaction.

\subsection{Tactile Sensor Mechatronics}


The structural architecture of the tactile sensing perception layer is detailed in Fig.~\ref{Fig_Gripper_Tacitle}. Each parallel finger is integrated with a dedicated tactile sensing module featuring a single-channel FSR. Under external normal pressure, the spatial distribution of internal transduction varies, thereby altering the material's conductivity. This dynamic conductance change is subsequently mapped to a continuous analog signal and measured via a standard voltage divider circuit. The raw analog signal is then sampled using a microcontroller.


Using an FSR-based tactile sensing system presents several key challenges. First, the FSR can only measure normal force, rendering the tangential shear components completely unobservable ($\mathcal{A}_3$). Second, the transducer substrate is inherently sensitive to the spatial pressure distribution across its active area. Because it yields a single-channel measurement, the resulting output suffers from structural ambiguity: the system cannot distinguish whether a change in the electrical signal is driven by an alteration in the absolute normal force magnitude or by a shifting contact topology and pressure distribution profile. Under $\mathcal{A}_4$, this coupling severely distorts the feedback loop, as the contact profile may dynamically switch between point, line, or irregular surface patch contacts. Furthermore, the raw signal is corrupted by multiple nonlinearly coupled noise sources, including hysteresis, viscoelastic recovery, and material fatigue.


To resolve issues arising from diverse physical contact types and to minimize mechanical noise, a specialized medium-elasticity elastomeric pad is fabricated and overmolded directly onto the top of the circular FSR active area, while a non-adhesive film is attached to the bottom. This arrangement relaxes the FSR's contact with the gripper, thereby preventing internal pre-loading. When encountering distinct contact topologies, such as point contact from a sharp edge, this compliant interface undergoes rapid elastic deformation. It redistributes concentrated stresses across the active sensing area, ensuring that irregular geometric boundaries do not result in localized sensor punctures or premature signal saturation. Furthermore, this composite design inherently possesses the structural ability to damp high-frequency tactile noise induced by micro-slip transitions and dynamic structural vibrations.

\subsection{GP Tactile Sensing Calibration}


While the elastomeric overlay ensures geometric adaptation, it severely exacerbates the intrinsic hysteresis, viscoelastic, and history-dependent electrical behavior of the sensor measurement~\cite{Donlon2018IROS, Li2024EasyCalib}. Instead of using a traditional analytical model-based calibration approach, a GP-based calibration model that captures the relationship between applied force $F$ and raw measurement $s$ is constructed. The bounded calibration uncertainties are utilized in the control design~\cite{HanRAL2021, Han2025IROS, Han2024JDSMC}.


We deploy GP-based regression mapping. Let the haptic feature vector containing historical data be defined as $\bm{s}_i = [s_i \, s_{i-1} \, \dots \, s_{i-n+1}]^\top \in \mathbb{R}^n$. The calibration model is denoted as
\begin{equation}\label{Eq_GP_Sensor}
  F = f_\text{cali} (\bm{s}) \sim \mathcal{GP}\left(\mu(\bm{s}), \sigma^2(\bm{s})\right),
\end{equation}
where $F = f_{\text{cali}}(\bm{s})$ denotes the true underlying mapping, while $\mu(\bm{s})$ and $\sigma^2(\bm{s})$ represent the mean and variance functions of the GP model, respectively. Offline calibration is executed to collect a training dataset $\mathbb{D}_{\text{cali}} = \{(\bm{s}_i, F_{i})\}_{i=1}^{N}$, where $F_{i} \in \mathbb{R}$ represents the true normal force target corrupted by noise $\epsilon \sim \mathcal{N}(0, \sigma_n^2)$. Let $\mathbf{S} = [\bm{s}_1\, \dots\, \bm{s}_N]^\top \in \mathbb{R}^{N \times n}$ denote the aggregated design matrix, and let $\mathbf{F} = [{F}_{1}\, \dots\, {F}_{N}]^\top \in \mathbb{R}^N$ be the associated target vector.


Given a real-time haptic query vector $\bm{s}$ collected at runtime step $t$, the joint distribution of the training targets $\mathbf{F}$ and the unobserved output force $F$ is Gaussian~\cite{RasmussenGPML}
\begin{equation*}
\begin{bmatrix}
  \mathbf{F} \\
  F
\end{bmatrix} \sim
\mathcal{N}\left( \mathbf{0}, \,
\begin{bmatrix}
\bm{K}(\mathbf{S}, \mathbf{S}) + \sigma_n^2\mathbf{I}_N & \bm{k}(\mathbf{S}, \bm{s}) \\
\bm{k}^\top(\mathbf{S}, \bm{s}) & k(\bm{s}, \bm{s})
\end{bmatrix} \right),
\end{equation*}
where $\mathbf{K}(\mathbf{S}, \mathbf{S}) \in \mathbb{R}^{N \times N}$ represents the training covariance matrix, and $\bm{k}(\mathbf{S}, \bm{s}) \in \mathbb{R}^N$ acts as the cross-covariance vector between the training set and the runtime query. Each entry of $\mathbf{K}$ or $\bm{k}$ is obtained using the kernel function
\begin{equation}\label{Eq_GP_Kernel}
    k(\bm{s}, \bm{s}') = \sigma_f^2 \exp\left( -\frac{1}{2} (\bm{s} - \bm{s}')^\top \mathbf{L}^{-2} (\bm{s} - \bm{s}') \right),
\end{equation}
where $\sigma_f^2$ represents the signal scaling variance, and $\mathbf{L} = \text{diag}(\ell_1, \ell_2, \dots, \ell_n) \in \mathbb{R}^{n \times n}$ represents the length-scale matrix. The hyperparameter $\{\sigma_f^2, \mathbf{L}, \sigma_n^2\}$ is optimized by maximizing the log-marginal likelihood of the training dataset.


By applying standard Gaussian conditioning rules, the predictive distribution is calculated analytically in closed form, yielding the mean and variance~\cite{RasmussenGPML}
\begin{subequations}
\label{Eq_GP_Pred}
 \begin{align}
   \mu(\bm{s}) &= \bm{k}^\top \left[ \mathbf{K} + \sigma_n^2\mathbf{I}_N \right]^{-1} \mathbf{F}, \label{Eq_gp_mean} \\
   \sigma^2(\bm{s}) &= k(\bm{s}, \bm{s}) - \bm{k}^\top \left[ \mathbf{K} + \sigma_n^2\mathbf{I}_N \right]^{-1} \bm{k}, \label{Eq_gp_var}
\end{align}
\end{subequations}
where the implicit dependencies on $\bm{s}$ and $\mathbf{S}$ in $\bm{k} = \bm{k}(\mathbf{S}, \bm{s})$ and $\mathbf{K} = \mathbf{K}(\mathbf{S}, \mathbf{S})$ are omitted for brevity. The computed mean $\mu$ provides an accurate measurement of contact force, while the variance $\sigma^2$ quantifies the localized calibration confidence under the following lemma. The tactile sensing force measurement is denoted as $F_x(\bm{s}) = \mu(\bm{s})$.

\begin{lem}[Theorem 6 in~\cite{GPtheory}]
\label{Lemma_error_bound}
Given the training data set $\mathbb D_{\text{cali}}$, if the
kernel function~\eqref{Eq_GP_Kernel} is chosen such that $F$ has a finite reproducing
kernel Hilbert space norm $\|F\|_k< \infty$, for given $\rho_\text{cali} \in(0,1)$,
the calibration error $\Delta_\text{cali} =\mu (\bm s)- F$ relative to the ground truth ($F$) is bounded probabilistically
\begin{equation}
   \Pr \left\{ | \Delta_\text{cali} | \le \beta_\text{cali} \sigma(\bm s) \right\} \ge \rho_\text{cali} ,
  \label{Eq_lemma_probability}
\end{equation}
where $\beta_\text{cali} = \sqrt{2\|F\|_k^2 + 300\gamma \ln^3 \left(\frac{N + 1}{1 - \rho} \right)}$ and
$\gamma = \max_{\bm{s},\bm{s}' \in \mathbf S} \frac{1}{2} \ln \left| 1 + \sigma_n^{-2}{k}(\bm{s}, \bm{s}') \right|$ measures the maximum information gain.
\end{lem}


With the availability of accurate force measurements, we design the following preliminary PID control scheme to regulate the gripper-object interaction force:
\begin{align}\label{Eq_PID_Control}
  u &= K_p (F_x^d(\bm{s})-F_x) + K_d \frac{d(F_x^d(\bm{s})-F_x)}{dt} \nonumber\\
  &\quad + K_i\int_{0}^{t}(F_x^d(\bm{s})-F_x)\,dt,
\end{align}
where $K_p, K_d, K_i > 0$ are the control gains, $F_x(\bm{s}) = \mu(\bm{s})$ is evaluated based on the tactile sensing system prediction, and $u$ is the gripper displacement. The reference width of the gripper is then commanded to
\begin{equation}\label{Eq_Gripper_Control}
  x(t+1) = x(t) + u
\end{equation}
by regulating the angular position $\theta_\text{G}$ of the motor shaft using~\eqref{Eq_Gripper_Width}. The primitive controller~\eqref{Eq_PID_Control} is necessary when the gripper-object interaction dynamics~\eqref{Eq_Dyna_Inter} is unknown for force regulation. In the following, unless stated otherwise, the term "contact force" refers to the GP prediction mean value rather than the ground truth.

\section{Grasping Force Probing and Constraint Identification}
\label{Sec_state_machine}


Gentle and secure grasping of unknown objects requires appropriate grasping force design and the identification of the constraints $\mathds{C}_\text{secure}$ and $\mathds{C}_\text{gentle}$. In this section, we propose an FSM-based probing strategy to capture the object's geometry and compliance, design the contact force, and identify these constraints.

\subsection{Initial Grasping Force Probing}


Under $\mathcal{A}_1$, the object resides in the workspace of the gripper. Grasping and transporting the object requires identifying the critical initial grasping force. Due to the complete absence of \textit{a priori} physical properties dictated by the blind manipulation $\mathcal A_2$, a constant pre-selected gripping force introduces severe operational risks, which may either break delicate objects or cause slipping in the case of a heavy object. Consequently, the initial grasping force must actively adapt to the object's hidden mechanical characteristics.

\begin{figure}
 \centering
 \includegraphics[width=8cm]{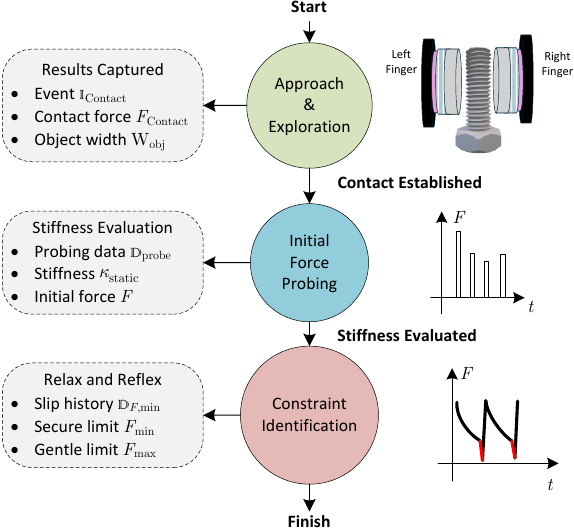}
 \vspace{-2mm}
 \caption{FSM topology for sequential contact initialization, stiffness evaluation, and online tactile safety boundary identification.}
 \label{Fig_FSM}
 \vspace{-3mm}
\end{figure}

\subsubsection{Contact Establishment}


The contact initialization event is formally defined by evaluating the sensory feedback from each parallel finger module. Let $F_{x,l}$ and $F_{x,r}$ represent the contact forces of the left and right fingers, respectively. The gripper closes its fingers gently at a constant speed, and a valid contact event is registered as
\begin{equation*}\label{Eq_contact_event}
  \mathds{I}_{\text{Contact}} =
  \begin{cases}
     1, \text{if } F_{x,l}(\bm s) \ge F_{\text{thr.}} \;\text{and}\; F_{x,r}(\bm s) \ge F_{\text{thr.}}, \\
   0, \text{otherwise},
  \end{cases}
\end{equation*}
where $F_{\text{thr.}} \in \mathbb{R}^+$ denotes the static engagement threshold. Furthermore, the actual contact force and the width of the object are registered at the time instant $t$ when $\mathds{I}_{\text{Contact}}(t) =1$ as
\begin{equation*}
   W_{\text{obj}} = x (t), \;\; F_{\text{Contact}} = F_x(\bm s) = \frac{F_{x,l}(\bm s) + F_{x,r}(\bm s)}{2},
\end{equation*}
where $F_x(\bm s)$ represents the average contact force.

\subsubsection{Grasping Force Estimation}

Once contact is established, the FSM execution loop safely triggers a transition from the exploration phase to the active stiffness probing routine to identify the initial grasping force. This phase aims to estimate the object's material compliance safely. To achieve this, the gripper jaws are commanded to apply multiples of $F_\text{Contact}$ under the controller~\eqref{Eq_PID_Control}. However, the engagement is stopped and restored immediately if a certain further deformation of the object is detected. At each measurement interval $i$, the probing information is updated:
\begin{equation}\label{Eq_probing}
 \mathbb{D}_\text{Probe}\leftarrow \mathbb{D}_{\text{Probe}}\cup \left\{F_{x,i}^-(\bm s), F_{x, i}^+(\bm s), x_{i}^-, x_{i}^+\right\},
\end{equation}
where \enquote{$-$} and \enquote{$+$} represent variables before and after each single probing step.

Since the deformation of the object (namely, the displacement $x_{i}^- - x_{i}^+$) is limited, given the probing data $\mathbb{D}_\text{Probe}$, we can identify the structural stiffness of the object by evaluating the linear relationship between the applied normal force differential and the resulting joint displacement profile across the collected sequence:
\begin{equation}
   \kappa_\text{static} = \frac{1}{|\mathbb{D}_\text{Probe}|} \sum_{i} \frac{F_{x,i}^+(\bm s) - F_{x,i}^-(\bm s)}{|x_{i}^+ - x_{i}^-|}.
  \label{Eq_stiffness}
\end{equation}
The stiffness metric $\kappa_\text{static}$ serves as a direct representation of the unmodeled payload's physical compliance. Under assumption $\mathcal{A}_2$, where no prior geometric or material properties are available, the initial grasping force is designed to scale dynamically based on this identified stiffness parameter:
\begin{equation}
   F = F_{\text{Contact}} + (\kappa_\text{static} - \kappa_{\min}) \frac{F_{\text{Contact}, \max} - F_{\text{Contact}}}{\kappa_{\max} - \kappa_{\min}},
  \label{Eq_Grasp_Force}
\end{equation}
where $F_{\text{Contact},\max} \in \mathbb{R}^+$ is the maximum grasping force provided by the gripper's actuator. Under this mechanism, a larger grasping force is applied to a more rigid object.

\subsection{Grasping Constraint Identification}

The grasping force probing provides a systematic way to detect the mechanical compliance of the object. However, successful grasping and transporting require further satisfaction of $\mathds{C}_\text{secure}$ and $\mathds C_\text{gentle}$. The initial grasping force could be either large (yet not large enough to break the object) or small (leading to slip), since the mass and surface friction remain unchecked. Estimation of the mass and friction situation is a complicated task as the tactile sensing system only provides a single-channel, single-direction measurement. Instead, we directly identify $ F_{\text{obj},\min}$ and $F_{\text{obj},\max}$ for $\mathds{C}_\text{secure}$ and $\mathds{C}_\text{gentle}$ through a relax-and-reflex mechanism.

Under the grasping force, the robotic arm is used to lift and transport the object to the target location. If secure grasping is established, the grasping force inherently satisfies the underlying constraint $\mathds C_{\text{secure}}$. Since $F_x (\bm s)>F_{\text{obj}, \min}$, we start to relax the reference force as $F_{x}^d (t)= \alpha_1 F_x^d(t-1), \alpha_1 \in(0, 1)$, under the controller~\eqref{Eq_PID_Control}. A valid slip event $\mathds{I}_{\text{slip}}$ due to the violation of the constraint $\mathds{C}_\text{slip}$ is registered as
\begin{equation}
  \mathds{I}_{\text{slip}} =
  \begin{cases}
     1,\; \text{if } F_{x}(\bm s) < \lambda F_x^d, \; {\dot F}_x^d <0,\\
   0,\; \text{otherwise},
  \end{cases}
  \label{Eq_slip_event_trigger}
\end{equation}
where $\lambda_1 \in (0, 1)$. The condition ${\dot F}_x^d <0$ rules out slips in the reference grasping force restoration phase due to controller overshoot. To avoid failures due to loss of contact, the reference force immediately increases based on the previously identified minimum force once slip occurs.

%
%
%

At each valid slip, $F_x(\bm s)$ represents the minimum grasping force, since the constraint is validated. Therefore,
 \begin{equation}
   \mathbb{D}_{F_{\min}} \leftarrow \mathbb{D}_{F_{\min}} \cup \left\{ F_x(\bm s(t)) \right\}, \text{if } \mathds I_\text{slip}(t) = 1.
  \label{Eq_memory_update}
\end{equation}
Over multiple trials, the minimum grasping force is identified as
\begin{equation}\label{Eq_F_min}
   F_{\text{obj},\min} = \frac{1}{|\mathbb{D}_{F_{\min}}|} \sum_{F \in \mathbb{D}_{F,\min}} F,
\end{equation}
which completes the constraint $\mathds{C}_\text{secure}$. For $\mathds{C}_\text{gentle}$, the temporary maximum grasping force is defined as
\begin{equation}\label{Eq_F_max_temp}
 \hat{F}_{\text{obj},\max} = \min\{\kappa_{\text{static}} \cdot \epsilon_{\max},\; \xi \cdot F_{\text{obj},\min},\; F_{\text{G}, \max}\},
\end{equation}
where $\epsilon_{\max}\in(0, 1)$ is the maximum deformation, $\xi>1$, and $F_{\text{G}, \max}$ is the maximum grasping force provided by the gripper.

\begin{rem}
  The obtained maximum grasping force~\eqref{Eq_F_max_temp} will not cause significant deformation or damage to hard objects. For soft and deformable objects,~\eqref{Eq_F_max_temp} might be overestimated. Thin films, such as cardboard and aluminum plates, can display spring-like elastic properties within a certain force range. However, when the grasping force approaches the yield limit, the reactive force drops, and permanent and irreversible deformation occurs.
\end{rem}

To verify the temporary maximum grasping force, the grasping force reference is increased slowly toward the temporary maximum force as $F_{x}^d (t)= \alpha_2 F_x^d(t-1), \alpha_2 >1$. We define the event of objects yielding against $F_{x}^d (t)$ as
\begin{equation}
  \mathds{I}_{\text{yield}} =
  \begin{cases}
     1,\; \text{if } F_{x}(\bm s) < \lambda_2 F_x^d,\; {\dot F}_x^d >0, \\
   0,\; \text{otherwise},
  \end{cases}
  \label{Eq_yield_event}
\end{equation}
where $\lambda_2 \in (0, 1)$. The gentle grasping ceiling force is thus registered as
 \begin{equation}
   F_{\text{obj},\max}= F_{x}(\bm s(t)), \text{if } \mathds I_\text{yield}(t) = 1.
  \label{Eq_F_max}
\end{equation}
If no yielding occurs, $F_{\text{obj},\max}=\hat{F}_{\text{obj},\max}$.

\begin{rem}
  The proposed object slip and yield event detection mechanism is applied to oversee the grasping process from the moment contact is established (i.e., when applying the force in~\eqref{Eq_Grasp_Force}) until constraint-enforced control takes over. If slip or yielding is detected during the initial lifting process, the contact force is adjusted automatically.
\end{rem}

All the aforementioned procedures are executed entirely online, enabling the physical constraints to be identified in real time. With these constraints verified dynamically, the manipulation and transport of unmodeled payloads become certifiably safe, paving the way for predictive control formulations. However, the exact underlying interaction dynamics remain unknown. To address this, the subsequent section introduces an online dynamics learning scheme designed to identify the unmodeled interaction $f_{\text{inter}}$. Furthermore, an MPC framework, which replaces the baseline PID controller, is designed to actively track the reference force while explicitly respecting the identified constraints.

\section{Dynamic Grasping via Predictive Contact Control}
\label{Sec_MPC_control}


The transition from baseline control to a predictive framework is motivated by three critical limitations: (1) the object's stiffness may vary dynamically due to nonuniform material compliance; (2) the PID loop ignores the coupled contact dynamics ($f_{\text{inter}}$); and (3) the PID control design does not take constraints into consideration, where overshoot or oscillation can cause constraint violations. To address these deficiencies, this section introduces an integrated model learning and predictive control framework.

\subsection{Online Interaction Dynamics Learning}

The grasping constraints are identified online, and the object is already in the grasping phase. Therefore, the characterization of the interaction dynamics must also be achieved online. This is based on the consideration that $f_\text{inter}$ is a shared dynamic property between the gripper and the object under the current specific contact topology. Instead of traditional system identification, we consider a data-driven method.

To identify the unmodeled interaction dynamics without relying on precise analytical models, a multi-frequency sine-excitation grasping force reference (subject to the identified constraints) is injected into the system under PID control. The resulting state-force pairs are collected as $
\mathbb{D}_{\text{inter}} = \left\{ (\bm{x}_i, F_{x,i}(\bm s)) \right\}_{i=1}^{M}$, where $M \in \mathbb{Z}^+$ represents the total number of valid observations. The input feature vector $\bm{x}_i \in \mathbb{R}^{d}$ structures a localized temporal sequence of normalized displacement depths
\begin{equation}
   \bm{x}_i = \frac{1}{W_{\text{obj}}} \begin{bmatrix} x_i & x_{i-1} & \cdots & x_{i-d+1} \end{bmatrix}^\top \in \mathbb{R}^{d}.
  \label{Eq_input_feature_vector}
\end{equation}
Using $\mathbb{D}_{\text{inter}}$, the nonlinear interaction dynamics $f_{\text{inter}}$ are learned online via GP regression, sharing a mathematical formulation similar to tactile sensor calibration~\eqref{Eq_GP_Sensor} as
\begin{equation}
   F = f_\text{inter} \sim \mathcal{GP}(\mu(\bm x), \sigma^2(\bm x)),
  \label{Eq_gp_interaction}
\end{equation}
where $F = f_\text{inter}$ represents the underlying true relationship. The prediction is denoted as $F_x(\bm x)=\mu(\bm x)$, obtained similarly to~\eqref{Eq_GP_Pred}, and the same type of kernel function is used. The error is denoted as $\Delta_\text{inter}=\mu(\bm x) - F_x(\bm s)$. To accommodate the computational cost, the nonlinear interaction dynamics are learned online in a separate thread alongside the main program execution.

%
%

\begin{rem}\label{Rem_GP_Interaction}
During the grasping phase, particularly when interacting with rigid objects, the physical displacement along $\mathbf{e}_x$ is strictly constrained. Under these conditions, utilizing numerical differentiation to compute velocity $\dot{x}$ and acceleration $\ddot{x}$ introduces significant high-frequency sensory noise and truncation errors. To preserve data fidelity, the proposed GP formulation bypasses explicit time derivatives entirely, leveraging a $d$-step historical horizon of normalized depth states to implicitly capture the underlying system acceleration and damping dynamics. Furthermore, because direct, isolated ground-truth measurements of the internal interaction forces are inherently unavailable during live operations, the physical force readings ($F_x = \mu(\bm s)$) harvested from the calibrated tactile sensors are utilized within the online regression framework. In contrast, during the tactile sensing system calibration phase, the ground-truth force is measured using force sensors.
\end{rem}

Based on Remark~\ref{Rem_GP_Interaction}, the overall error $\Delta_{\text{total}}$ can be explicitly decomposed as
\begin{align}\label{Eq_error_decomposition}
 \Delta_\text{total} &= F_x(\bm x) - F_x=\mu(\bm x) - F_x\nonumber\\
 & = [\mu(\bm x) - F_x(\bm s)+ F_x(\bm s) - F_x]\nonumber\\
 & = [\mu(\bm x) - F_x(\bm s)] + [\mu(\bm s) - F_x]\nonumber\\
 & = \Delta_\text{cali} + \Delta_\text{inter},
\end{align}
where
\begin{itemize}
  \item $F_x$ is the unknown ground truth of the contact force;
  \item $F_x(\bm s) = \mu (\bm s)$ is the tactile sensing model prediction;
  \item $F_x(\bm x) = \mu (\bm x)$ is the gripper-object interaction dynamics prediction.
\end{itemize}

\begin{lem}
\label{Lemma_cascaded_error_bound}
Let the true calibration profile $F = f_{\text{cali}}(\bm{s})$ and the true interaction dynamics $F = f_{\text{inter}}(\bm{x})$ possess finite RKHS norms under their own kernel functions. The overall prediction error $\Delta_{\text{total}}$ relative to the ground truth in \eqref{Eq_error_decomposition} is bounded probabilistically by the composite statistical uncertainty of the cascaded system
\begin{equation}
   \Pr \left\{ | \Delta_{\text{total}}| \le \bar{\beta} \sigma_{\text{total}}(\bm{x}, \bm{s}) \right\} \ge \rho,
  \label{Eq_lemma_cascaded_probability}
\end{equation}
where $\sigma_{\text{total}}(\bm{x}, \bm{s}) = \sqrt{\sigma_{\text{cali}}^2(\bm{s}) + \sigma_{\text{inter}}^2(\bm{x})}$ represents the joint posterior standard deviation, and $\bar{\beta} = \sqrt{2}\max(\beta_{\text{inter}}, \beta_{\text{cali}})$ serves as the conservative unified information-theoretic scaling multiplier for a target confidence probability $\rho \in (0, 1)$.
\end{lem}

\begin{proof}
By applying the triangle inequality directly to the cascaded error decomposition formulated in \eqref{Eq_error_decomposition}, the magnitude of the total prediction error is bounded by
\begin{equation}
   |\Delta_{\text{total}}| \le |\Delta_{\text{inter}}| + |\Delta_{\text{cali}}|. \label{eq:triangle_expansion}
\end{equation}
Similarly to Lemma~\ref{Lemma_error_bound}, for the interaction dynamics model, there exists a scaling parameter $\beta_{\text{inter}}$ such that $\Pr\{|\Delta_{\text{inter}}| \le \beta_{\text{inter}} \sigma(\bm{x})\} \ge \rho_{\text{inter}}$. Substituting both individual bounds into \eqref{eq:linear_combination} yields
\begin{equation}
   |\Delta_{\text{total}}| \le \beta_{\text{inter}} \sigma(\bm{x}) + \beta_{\text{cali}} \sigma(\bm{s}). \label{eq:linear_combination}
\end{equation}
To simplify, we invoke $a + b \le \sqrt{2}\sqrt{a^2 + b^2}$ for any real numbers $a, b \ge 0$. Applying this inequality to the right-hand side of \eqref{eq:linear_combination} leads to
\begin{align}
|\Delta_{\text{total}}| &\le
   \beta_{\text{inter}} \sigma(\bm{x}) + \beta_{\text{cali}} \sigma(\bm{s}) \nonumber\\
   &\le \max(\beta_{\text{inter}}, \beta_{\text{cali}}) \big( \sigma(\bm{x}) + \sigma(\bm{s}) \big) \nonumber \\
   &\le \sqrt{2}\max(\beta_{\text{inter}}, \beta_{\text{cali}}) \sqrt{\sigma_{\text{cali}}^2(\bm{s}) + \sigma_{\text{inter}}^2(\bm{x})} \nonumber \\
   &= \bar{\beta} \sigma_{\text{total}}(\bm{x}, \bm{s}).
\end{align}
Consequently, $|\Delta_{\text{total}}| \le \bar{\beta} \sigma_{\text{total}}(\bm{x}, \bm{s})$ is guaranteed to hold, conditioned on the joint event of tactile sensing and dynamic interaction, with a joint probability of at least $\rho = \rho_{\text{cali}}  \rho_{\text{inter}} \in (0, 1)$. This completes the proof.
\end{proof}

\subsection{Constrained MPC-Based Dynamic Grasping}

Given the learned contact dynamics, we can enforce these predictive envelopes as robust constraints within a model predictive control framework to achieve secure and gentle grasping and manipulation of the object.

In practice, the true physical contact force $F_x$ is unavailable. To account for both the tactile sensing uncertainty $\sigma_{\text{cali}}$ and the interaction dynamics variance $\sigma_{\text{inter}}$, we establish a statistical confidence bound against worst-case prediction deviations. The true contact force can be bounded probabilistically as $F_x \in [\mu(\bm{x}_k) - \beta\sigma_{\text{total}}, \mu(\bm{x}_k) + \beta\sigma_{\text{total}}]$ for some $\beta>0$. Substituting this into the nominal safety bounds yields the tightened, robust force constraint pair directly enforced within the optimization:
\begin{subequations}\label{Eq_Robust_Constraints}
\begin{align}
   &\mathds{C}_{\text{secure}}^{\text{Robust}}:\; F_x(\bm x) \ge F_{\text{obj}, \min} + \beta \sigma_{\text{total}}, \\
   &\mathds{C}_{\text{gentle}}^{\text{Robust}}:\; F_x(\bm x) \le F_{\text{obj}, \max} - \beta \sigma_{\text{total}}.
\end{align}
\end{subequations}

Under this structure, the nonlinear MPC problem is formulated over a finite prediction horizon $H$ as the following constrained optimization:
\begin{subequations}
\label{Eq_NMPC}
\begin{align}
   \min_{\{u_{t+j|t}\}} & \sum_{j=0}^{H-1} \left(w_1 \hat{e}_{t+j|t}^2 + w_2 \Delta \hat{e}_{t+j|t}^2 + w_3 u_{t+j|t}^2 \right) \\
   \text{s. t.}\;\; & F(\bm x_{t+j|t}) = \mu(\bm{x}_{t+j|t}), \\
   & \bm{x}_{t+j|t} = \mathcal{M}\left( u_{t+j|t}\right), \\
   & F_{x,t+j|t} \ge F_{\text{obj}, \min} + \beta \sigma_{\text{total}}(\bm{x}_{t+j|t}, \bm{s}_{t+j|t}), \\
   & F_{x,t+j|t} \le F_{\text{obj}, \max} - \beta \sigma_{\text{total}}(\bm{x}_{t+j|t}, \bm{s}_{t+j|t}), \\
  & -u_{\max} \le u_{t+j|t} \le u_{\max},
\end{align}
\end{subequations}
where $w_1, w_2, w_3\in \mathbb R^+$, the control input $u_{t+j|t}$ represents the optimized incremental displacement of the gripper width, $\hat{e}_{t+j|t} = F_{x,t+j|t}^d - F(\bm x_{t+j|t})= F_{x,t+j|t}^d - \mu( \bm x_{t+j|t})$ denotes the difference between the interaction dynamics predicted force and the desired reference $F_x^d$, $\Delta \hat{e}$ represents the error difference (where the derivative is not taken due to numerical errors and noise), and $\mathcal{M}(\cdot)$ represents the temporal shift mapping operating on the state-space history buffer to construct the feature vector in~\eqref{Eq_input_feature_vector}. The control input is constrained to $[-u_{\max}, u_{\max}]$. The final control input $u_t^*=u_{t|t}$ is transmitted to the gripper motor position controller through~\eqref{Eq_Gripper_Control}.

\begin{algorithm}
	\caption{Dynamic grasping of unknown objects}\label{Alg_overall_strategy}
	
	\SetAlgoLined
	\PrintSemicolon
	
    \KwIn{ Horizons/Weights $H, w_1, w_2, w_3$, $u_{\max}$ etc., Tactile sensing model $f_{\text{cali}}$ }
	\KwOut{Control step $u_t^*$}
	
	\textsf{Phase 1: FSM-based Probing}\;
	\If{Contact Establishment}
	{
        Approach object with constant speed until $\mathds{I}_{\text{Contact}} == 1$\;			
		$W_{\text{obj}} \gets x, \, F_{\text{Contact}} \gets F_{x}(\bm{s})$\;
	}
	\If{Stiffness Probing}
	{
		Squeeze the object\;
        $\mathbb{D}_\text{Probe}\leftarrow \mathbb{D}_{\text{Probe}}\cup \left\{F_{x,i}^-(\bm s), F_{x, i}^+(\bm s), x_{i}^-, x_{i}^+\right\}$\; 			
		Compute $\kappa_{\text{static}}$ via~\eqref{Eq_stiffness} and initial $F$ via~\eqref{Eq_Grasp_Force}\;
	}
	\If{Constraint Identification}
	{
		\For{Relaxation \& Reflex iteration}
		{
			Relax force and evaluate slip trigger $\mathds{I}_{\text{slip}}$\;
			\If{$\mathds{I}_\mathrm{slip} == 1$}
			{
            $\mathbb{D}_{F_{\min}} \leftarrow \mathbb{D}_{F_{\min}} \cup \left\{ F_x(\bm s) \right\}$\;	
            Increase contact force\;			
			}

		}
Update  $F_{\text{obj}, \min}$ via \eqref{Eq_F_max_temp}\;
Verify the ceiling grasping force by~\eqref{Eq_F_max}\;
	}

	\textsf{Phase 2: Interaction Dynamics Learning}\;
	$\mathbb{D}_{\text{inter}} \gets \mathbb{D}_{\text{inter}} \cup \{(\bm{x}_t, F_{x,t}(\bm{s}))\}$ via \eqref{Eq_input_feature_vector}\;
	Train $f_{\text{inter}} \sim \mathcal{GP}(\mu(\bm{x}), \sigma^2(\bm{x}))$ via \eqref{Eq_gp_interaction}\;

	\While{System is operational}
	{
		\textsf{Phase 3: Dynamic Grasping Control}\;
        Update reference force\;
		Update $\mathds C_{\text{gentle}}^{\text{Robust}}$ and $\mathds C_{\text{secure}}^{\text{Robust}}$ by~\eqref{Eq_Robust_Constraints}\;
		\eIf{Computational budget permits Nonlinear MPC}
		{
			Solve nonlinear MPC via \eqref{Eq_NMPC}\;
		}
		{
			Evaluate $\kappa_\text{dyna}$ by~\eqref{Eq_Dyan_Stiffness}\;
            Solve linear MPC via \eqref{Eq_Linear_MPC_Optimization}\;
		}
		 $u_t^* \gets u_{t|t}$\;
	}
\end{algorithm}

\subsection{Practical Consideration of Linear MPC}

The optimization in~\eqref{Eq_NMPC} relies on multi-step forward non-parametric evaluations of both the grasping dynamics prediction and time-varying constraint updates. This might be computationally expensive in the case of limited computational resources and a long prediction horizon. Based on this practical consideration, for cases with a limited computation budget, we develop a variant of the nonlinear MPC. Given the constrained displacement of $x$ upon secure contact, the interaction dynamics are linearized as
\begin{equation*}
   F_{k+1} \approx F_k + \frac{\partial f_{\text{inter}}}{\partial x} \delta x = F_k + \kappa_\text{dyna} u,
\end{equation*}
where $\kappa_\text{dyna} =\frac{\partial f_{\text{inter}}}{\partial x} $ is the dynamic, localized stiffness and is captured via a numerical perturbation of $\epsilon \in \mathbb{R}^+$ to $\bm{x}$:
\begin{equation}\label{Eq_Dyan_Stiffness}
   \kappa_\text{dyna} \approx \frac{\mu(\bm{x} + \bm{\epsilon}) - \mu(\bm{x} - \bm{\epsilon})}{2\epsilon},
\end{equation}
where $\bm{\epsilon} = [\epsilon\; 0\; \dots\; 0]^\top \in \mathbb{R}^d$.

By transforming the contact dynamics into a locally linear space, the optimization problem reduces to a highly efficient convex quadratic programming problem:
\begin{subequations}
\label{Eq_Linear_MPC_Optimization}
\begin{align}
   \min_{\{u_{t+j|t}\}} & \sum_{j=0}^{H-1} \left( w_1 \hat e_{t+j|t}^2 + w_2 \Delta \hat e_{t+j|t}^2 + w_3 u_{t+j|t}^2 \right) \\
   \text{s.t.} \quad & F_{x,t+j|t} = F_{x,t} + \kappa_\text{dyna} \sum\nolimits_{m=0}^{j} u_{t+m|t}, \\
   & F_{x,t+j|t} \ge F_{\text{obj}, \min} + \beta \sigma_{\text{total}}(\bm{x}_{t}, \bm{s}_{t}), \\
   & F_{x,t+j|t} \le F_{\text{obj}, \max} - \beta \sigma_{\text{total}}(\bm{x}_{t}, \bm{s}_{t}), \\
  & -u_{\max} \le u_{t+j|t} \le u_{\max},
\end{align}
\end{subequations}
where the robust uncertainty boundaries are kept frozen starting from the initial evaluation step. The final control input $u_t^*=u_{t|t}$ is transmitted to the motor controller through~\eqref{Eq_Gripper_Control}.

\begin{rem}
 Upon establishing physical contact, the localized deformation depth between the object and the gripper is bounded. The online-estimated local stiffness dynamically captures potential non-uniform compliance variations of the object while simultaneously simplifying the real-time model updating loop. In practice, this linear approximation is activated to preserve control stability in the event of limited computational resources or a nonlinear optimization failure.
\end{rem}

The overall online probing and predictive control strategy is summarized in Algorithm~\ref{Alg_overall_strategy}. Except for tactile sensing calibration, all 3 phases are run online in real time. The linear MPC model is only used if the computation budget is limited.

\begin{figure*}
  \centering
  \subfigure[]{\label{Fig_Robot_Setup}
  \includegraphics[height=3.5cm]{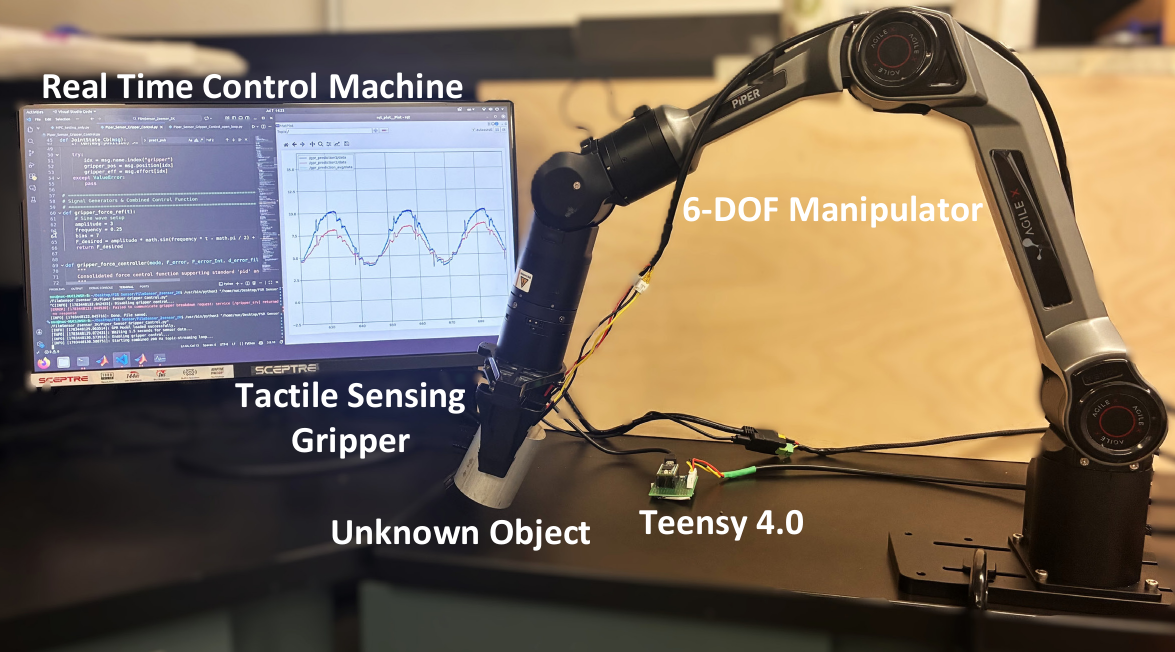}}
  \hspace{-3mm}
  \subfigure[]{\label{Fig_Gripper}
  \includegraphics[height=3.5cm]{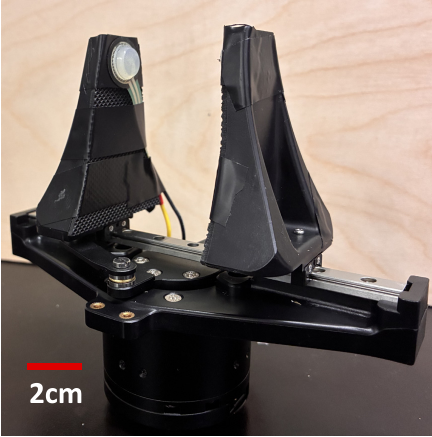}}
    \hspace{-3mm}
  \subfigure[]{\label{Fig_Cali_Loadcell}
  \includegraphics[height=3.5cm]{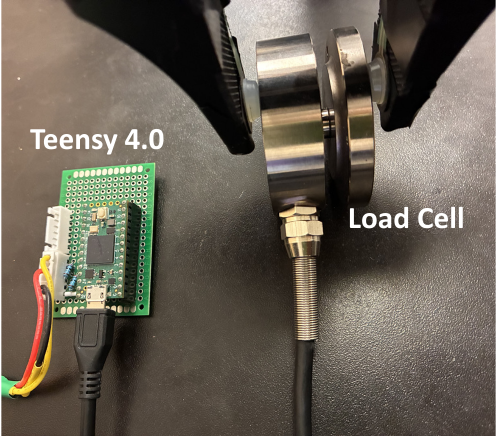}}
  \hspace{-3mm}
  \subfigure[]{\label{Fig_1kg_weight}
  \includegraphics[height=3.5cm]{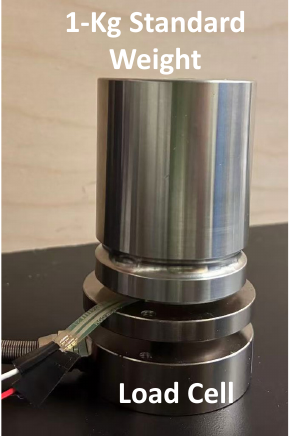}}
  \hspace{-3mm}
  \subfigure[]{\label{Fig_Cali_Contact}
  \includegraphics[height=3.5cm]{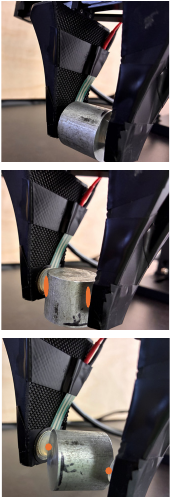}}
  \vspace{-2mm}
  \caption{Experimental setup. (a) Overview of the tactile sensing-enabled grasping system. (b) The gripper with a tactile sensing module. (c) Calibration setup and custom-made circuit mechatronics. (d) Standard weight measurement testing. (e) Demonstration of different contact types, from top to bottom: surface contact, line contact, and point contact.}
  \label{Fig_Setup}
  \vspace{-2mm}
\end{figure*}

\subsection{Convergence Analysis}
\label{Sec_probabilistic_stability}


In this subsection, we show the error convergence under the proposed control. Let the true force control error at $t$ be $e_t = F_{x,t}^d - F_{x,t}$. Based on Remark~\ref{Rem_GP_Interaction}, the actual contact force deviates from the direct GP prediction $\mu(\bm{x}_t)$ due to cascaded calibration and interaction errors. Following~\eqref{Eq_error_decomposition}, we obtain
\begin{equation*}
   F_{x,t} = F_x(\bm{x}_t) - \Delta_{\text{total},t} = \mu(\bm{x}_t) - (\Delta_{\text{cali},t} + \Delta_{\text{inter},t}).
\end{equation*}
Under the MPC control, $\bm{x}_{t+1} = \mathcal{M}(u_t^*)$, and thus the true error dynamics cascade to the subsequent interval as
\begin{align}\label{Eq_Error_Dynamics}
   e_{t+1} &= F_{x,t+1}^d - F_{x,t+1}\nonumber\\
   &=\underbrace{F_{x,t+1}^d - \mu\left(\mathcal{M}( u_{t}^*)\right)}_{\hat{e}_{t+1}} + \Delta_{\text{total},t+1},
\end{align}
where $\mu\left(\mathcal{M}( u_{t}^*)\right)$ represents the force prediction under the control input $u_{t}^*$. Given the above setup, we show the convergence of the true error in the following theorem.

\begin{thm}
\label{Thm_mpc_stability}
Suppose assumptions $\mathcal A_1$ to $\mathcal A_4$ hold. If the MPC optimization~\eqref{Eq_NMPC} generates a valid solution, under the proposed predictive control framework, the grasping force control error decreases with probability $\rho$ until it enters a compact residual bound, and the constraints~\eqref{Eq_Robust_Constraints} are satisfied.
\end{thm}

\begin{proof}
To verify the convergence of the true error $e$, we define a candidate Lyapunov function based on the true error as $V = \frac{1}{2} e^2$. To explicitly link the future tracking energy back to this current state baseline under the high-probability guarantees of Lemma~\ref{Lemma_cascaded_error_bound}, we resolve the nominal updates over the sampling interval. Let
\begin{align*}
\Delta \mu_t &\triangleq \mu\left(\mathcal{M}(u_{t}^*)\right) - \mu(\bm{x}_t), \;\;
\Delta F_{x,t}^d \triangleq F_{x,t+1}^d - F_{x,t}^d,
\end{align*}
be the predicted nominal force increment and reference change, respectively. The net tracking correction command generated by the optimization is $\eta_t \triangleq \Delta F_{x,t}^d-\Delta \mu_t$.

The tracking energy at the next step using~\eqref{Eq_Error_Dynamics} becomes
\begin{align}
V_{t+1} &= \frac{1}{2} \left( F_{x,t+1}^d - \mu\left(\mathcal{M}(u_{t}^*)\right) + \Delta_{\text{total},t+1} \right)^2 \nonumber \\
&= \frac{1}{2} \left(F_{x,t}^d - \mu(\bm{x}_t) + \eta_t + \Delta_{\text{total},t+1} \right)^2 \nonumber \\
&= \frac{1}{2} (\underbrace{F_{x,t}^d - \mu(\bm{x}_t) +\Delta_{\text{total},t}}_{e_t} -\Delta_{\text{total},t}+ \eta_t + \Delta_{\text{total},t+1} )^2 \nonumber \\
&= \frac{1}{2} \left( e_t + \eta_t - \Delta_{\text{total},t} + \Delta_{\text{total},t+1} \right)^2 \nonumber \\
&= V_t + e_t \eta_t + \frac{\eta_t^2}{2} - (e_t + \eta_t)(\Delta_{\text{total},t} - \Delta_{\text{total},t+1}) \nonumber\\&+ \frac{(\Delta_{\text{total},t} - \Delta_{\text{total},t+1})^2}{2}.
\label{Eq_lyapunov_expansion}
\end{align}

By Lemma~\ref{Lemma_cascaded_error_bound}, $| \Delta_{\text{total}}| \le \bar{\beta} \sigma_{\text{total}}(\bm{x}, \bm{s})$ holds with probability at least $\rho$. Therefore,
\begin{align*}
  |\Delta_{\text{total},t} - \Delta_{\text{total},t+1}| \le
  |\Delta_{\text{total},t}|+| \Delta_{\text{total},t+1}| \le
  2\bar{\beta}\sigma_{\text{total}}.
\end{align*}
Applying these bounds to~\eqref{Eq_lyapunov_expansion} yields the following difference condition:
\begin{align*}
\label{Eq_probabilistic_drift}
   \delta V_t \le e_t \eta_t + \frac{\eta_t^2}{2} + 2|e_t + \eta_t| \bar{\beta}\sigma_\text{total} + 2\bar{\beta}^2\sigma_\text{total}^2,
\end{align*}
where $\delta V_t = V_{t+1}-V_t$. According to~\eqref{Eq_Error_Dynamics}, $e_t = \hat{e}_t + \Delta_{\text{total},t}$, and therefore,
\begin{equation*}
e_t \eta_t = \hat{e}_t \eta_t + \Delta_{\text{total},t}\eta_t \le \hat{e}_t \eta_t + |\eta_t|\bar{\beta}\sigma_{\text{total}}.
\end{equation*}
Substituting the above relationship into $\delta V_t$ and applying simplification yields
\begin{equation*}
\Pr\left\{\delta V_t \le \hat{e}_t \eta_t + \frac{\eta_t^2}{2} + 2|\hat{e}_t + \eta_t| \bar{\beta}\sigma_\text{total} + 4\bar{\beta}^2\sigma_\text{total}^2\right\} \ge \rho.
\end{equation*}

Because the MPC actively minimizes the nominal error cost objective function across the preview horizon, it guarantees that $\eta_t$ acts in the opposite direction of the nominal tracking error $\hat{e}_t$ (i.e., $\text{sign}(\eta_t) = -\text{sign}(\hat{e}_t)$), rendering $\hat{e}_t \eta_t \le 0$. Outside of the compact residual band, the strictly negative nominal optimization descent term $\hat{e}_t \eta_t$ dominates these unmodeled cross-coupling disturbances. Over the iterations of the MPC,
$\delta V_t < 0$ until the true error enters a compact bound defined as
\begin{align*}
   \Omega_{e,t} \triangleq \left\{ e_t \;\middle|\; e_t \eta_t + \frac{\eta_t^2}{2} + 2|e_t + \eta_t|\bar{\beta}\sigma_\text{total} \ge -2\bar{\beta}^2\sigma_\text{total}^2 \right\},
\end{align*}
under the probability $\rho$. This completes the proof.
\end{proof}

With the above analysis, the online-identified secure-and-gentle constraints are embedded into the control design to ensure safety and reliability.

\section{Experiment and Analysis}
\label{Sec_Results}

In this section, we present extensive experimental results to validate the proposed grasping mechatronics system and the learning-based predictive control strategy.

\subsection{Experimental Setup}

%
%
%
%
%

The experimental setup is shown in Fig.~\ref{Fig_Setup}. The robotic arm (model Piper from Agilex Robotics) is a 6-DOF manipulator equipped with a two-finger gripper; see Fig.~\ref{Fig_Robot_Setup}. The gripper fingers are driven by a single DC motor through position control. The robotic arm communicates with the host PC via a CAN (Controller Area Network) bus, and the feedback and commands are compiled into standard ROS (Robot Operating System) topics. The gripper dimensions are $d_\text{finger}=70$ mm and $b_\text{finger}=22$ mm, and the maximum width between the two fingers is~$100$ mm. The tactile sensing system is shown in Fig.~\ref{Fig_Gripper}, following the design presented in Sec.~\ref{Sec_tactile_sensing}.

The FSR sensor (model FSR 402) is manufactured by Interlink Electronics, and its active circular sensing region has a diameter of $12.7$ mm. Its thickness is less than $0.5$ mm. The overall thickness of the custom-made tactile sensing module is approximately $4$ mm. Measurements from each FSR sensor are collected through a standard voltage divider circuit with a constant resistor of~$2000$ Ohms using a Teensy 4.0 microcontroller. The raw analog data falls in the range of $[0, 4095]$. According to the user manual, the maximum effective force measurement is around $10$ N. The raw tactile sensing measurements are transmitted to the host machine through a customized ROS topic and synchronized with the robotic arm's measurement and control signals.

\begin{figure*}
  \centering
  \subfigure[]{\label{Fig_Raw_Data}
  \includegraphics[height=4.6cm]{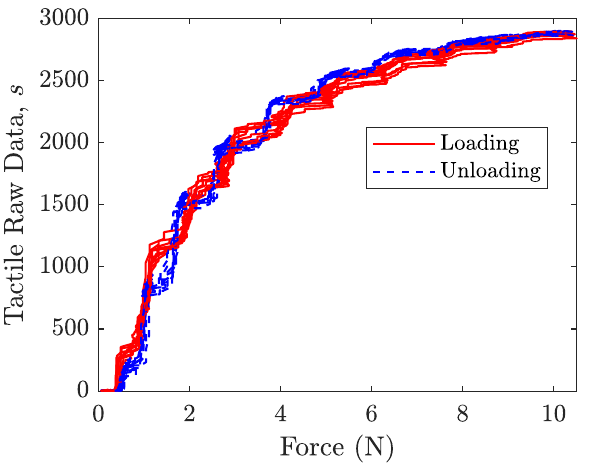}}
  \hspace{-3mm}\vspace{-2mm}
  \subfigure[]{\label{Fig_1kg_Test}
  \includegraphics[height=4.6cm]{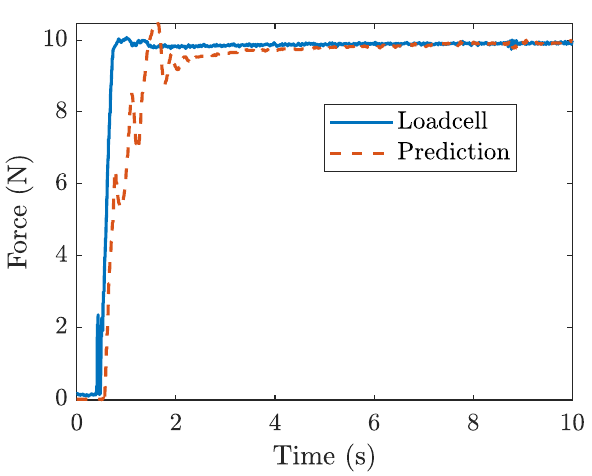}}
  \hspace{-3mm}
  \subfigure[]{\label{Fig_Cali_Valid}
  \includegraphics[height=4.6cm]{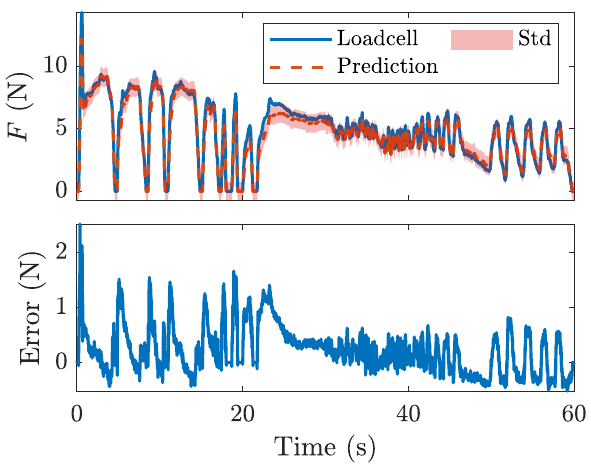}}
  \hspace{-3mm}
  \subfigure[]{\label{Fig_Cylinder}
  \includegraphics[height=4.6cm]{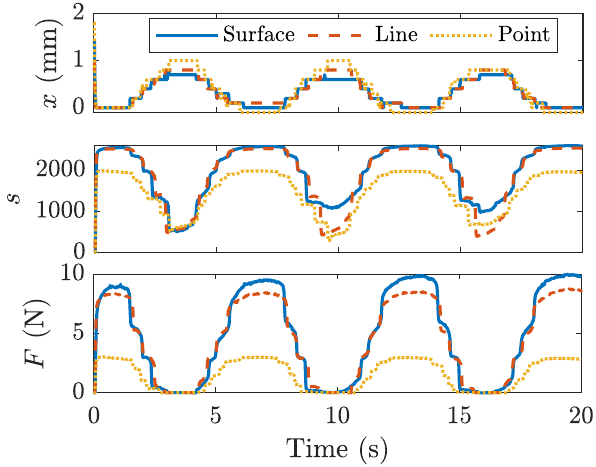}}
  \hspace{-3mm}
  \subfigure[]{\label{Fig_Cylinder_No}
  \includegraphics[height=4.6cm]{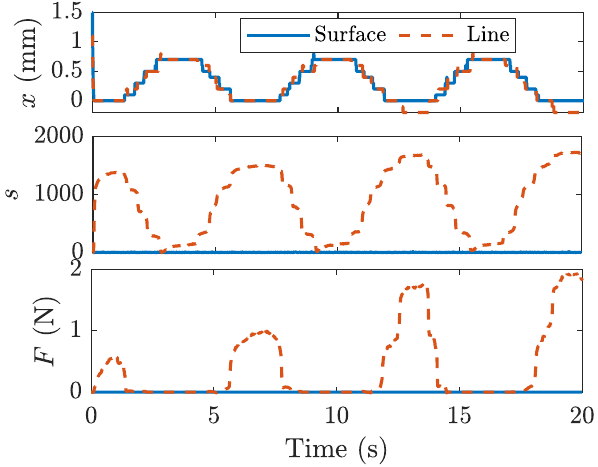}}
  \hspace{-3mm}
  \subfigure[]{\label{Fig_PID}
  \includegraphics[height=4.6cm]{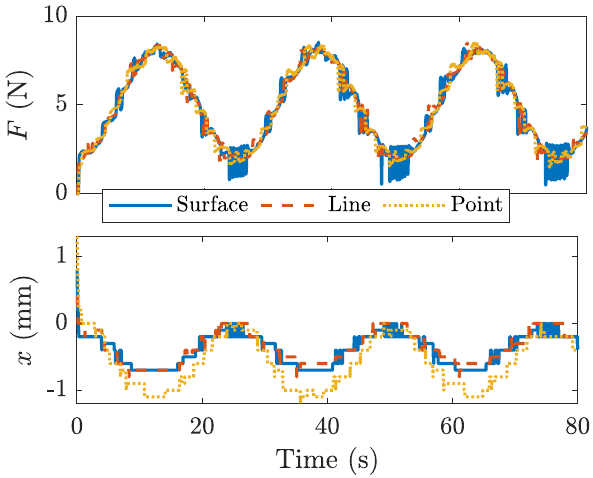}}
  \vspace{-2mm}
  \caption{Tactile sensing calibration and control. (a) Raw measurement. (b) 1-Kg standard weight measurement. (c) Varying force measurement. (d) Open-loop grasping control with different contact types. (e) Open-loop grasping control without contact force redistribution. (f) Closed-loop grasping control. The contact types in (d)-(f) can be found in Fig.~\ref{Fig_Cali_Contact}. To prevent potential damage to the sensors, the point contact test is not performed in (e).}
  \label{Fig_Cali_PID}
  \vspace{-2mm}
\end{figure*}


The GP-based tactile sensing calibration model is trained offline using raw tactile sensing data and force measured by a load cell; see Fig.~\ref{Fig_Cali_Loadcell}. We use the gripper (with the tactile module) to squeeze the industrial load cell and collect the data. The measurement from the load cell is treated as the ground truth. The training dataset contains $2000$ data points selected from a larger dataset, with a depth of $n=5$. The gripper-object interaction dynamics are trained online using a total of $1500$ data points, with a depth of $d=3$.


The control parameters are selected offline using only a foam block and a 3D-printed part. Since we assume that the object is completely unknown during experimental implementation, all the aforementioned setups and parameters are fixed as constants. The primary parameters used are $K_p=2.55\times 10^{-5}$, $K_i=1\times 10^{-6}$, and $K_d= 3\times 10^{-6}$ for the PID controller; $m_{\max}=0.5$~kg, $F_{\text{Contact}, \max}=5$~N, $F_\text{thr.}=1.25$~N, $\kappa_{\min}=200$, $\kappa_{\max}=4000$, $\lambda_1 =0.75$, and $\lambda_2 =0.75$ for the FSM; $H =5$, $w_1=5$, $w_2=2$, $w_3=1$, and $\beta=1$ for the MPC. All experiments are run on a Linux Ubuntu 22 machine (Intel Core i9-14900 x32, Nvidia RTX4060 GPU) with a $100$~Hz real-time control frequency.

\subsection{Results on Tactile Sensing }

The tactile sensing calibration results are shown in Fig.~\ref{Fig_Cali_PID}. It is clear that the raw measurement displays highly nonlinear properties, including dead zones, hysteresis, and saturation; see Fig.~\ref{Fig_Raw_Data}. Upon loading and unloading, the raw measurements under the same force differ. This motivates our use of GP-based calibration. Figs.~\ref{Fig_1kg_Test} and~\ref{Fig_Cali_Valid} show the online prediction of the GP-based calibration model. We place a standard weight of $m=1$~kg on the sensor (which is attached to the load cell; see Fig.~\ref{Fig_1kg_weight}), and the tactile sensing prediction in Fig.~\ref{Fig_1kg_Test} reaches the target value quickly. Note that the vibration in the measurement is due to manual adjustment of the load cell position on the sensor to achieve a stable balance. Furthermore, we use the gripper to dynamically grasp the load cell, and the real-time ground-truth measurement from the load cell alongside the tactile sensing prediction are shown in Fig.~\ref{Fig_Cali_Valid}. The overall prediction error is approximately $1$~N during fast loading/unloading and less than $0.5$~N during firm contact. The ground truth stays within one standard deviation of the GP-based prediction, highlighting the accuracy and bounded uncertainty of the non-parametric calibration model. Given this accuracy, the tactile prediction measurement is directly used in the control design hereafter.

In Figs.~\ref{Fig_Cylinder} and~\ref{Fig_Cylinder_No}, we demonstrate the effectiveness of the tactile-sensing mechatronic design in rejecting the effects of the contact interface topology. The gripper grasps an aluminum cylinder through open-loop control of the gripper width, $x$. Upon contact, the gripper width is commanded to follow a sine wave with the same amplitude and frequency across different contact configurations. We consider surface contact, line contact, and point contact; see Fig.~\ref{Fig_Cali_Contact}. With the designed structure, surface and line contacts do not display significant differences in terms of raw measurements and force predictions, meaning that the contact force is redistributed onto the FSR by the structure. In contrast, in Fig.~\ref{Fig_Cylinder_No}, when direct contact between the object and the FSR occurs, a significant discrepancy is observed; the measurement in surface contact becomes zero, and in line contact it becomes much smaller. These results show that FSR sensors are sensitive to contact types, rendering the redistribution of contact force necessary. Furthermore, Fig.~\ref{Fig_Cylinder} illustrates that the measurement in point contact differs from other types during open-loop control. This does not indicate a failure of the designed structure. Rather, it shows that the point-type interaction dynamics are distinct, such that under the same squeezing depth, the generated force is different; see Fig.~\ref{Fig_PID} for closed-loop control. Under closed-loop PID control, the gripper grasps the cylinder across different contact types by tracking a reference force profile. Fig.~\ref{Fig_PID} shows that the grasping force can be controlled in real time to follow a slow-changing profile using tactile sensing feedback. The gripper width in point contact mode undergoes the largest displacement to generate the same amount of force. Combined with Figs.~\ref{Fig_Cylinder} and~\ref{Fig_Cylinder_No}, these results demonstrate the effectiveness of the tactile mechatronics in handling different contact interface topologies and highlight the uniqueness of the case-specific interaction dynamics between the gripper and objects. Since contact topologies remain unknown until contact is established, these results imply that offline-identified models can be inaccurate.

\begin{figure*}
  \centering
  \subfigure[]{\label{Fig_Copper1}
  \includegraphics[height=2.85cm]{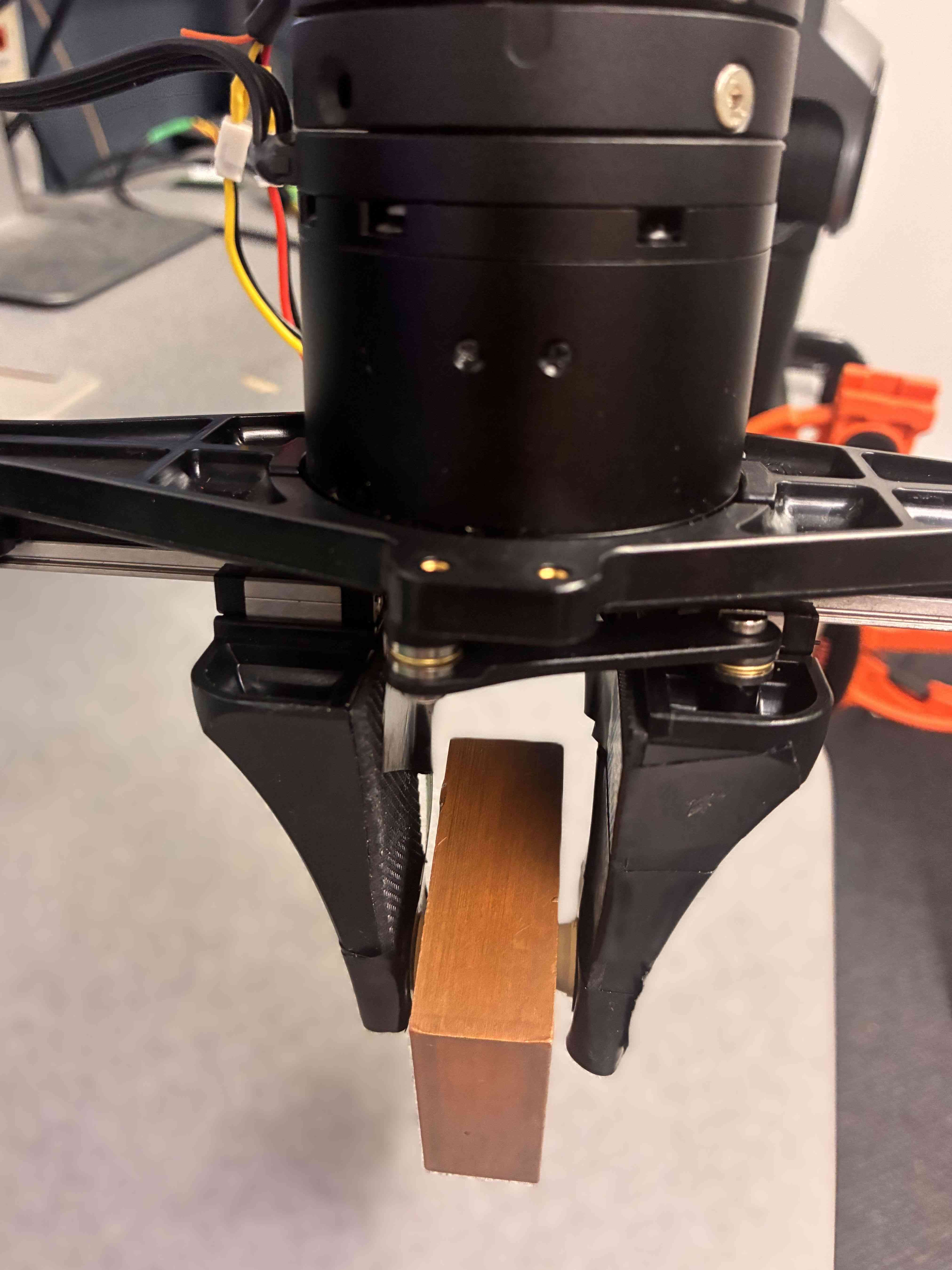}}
  \hspace{-3mm}\vspace{-2mm}
  \subfigure[]{\label{Fig_Copper2}
  \includegraphics[height=2.85cm]{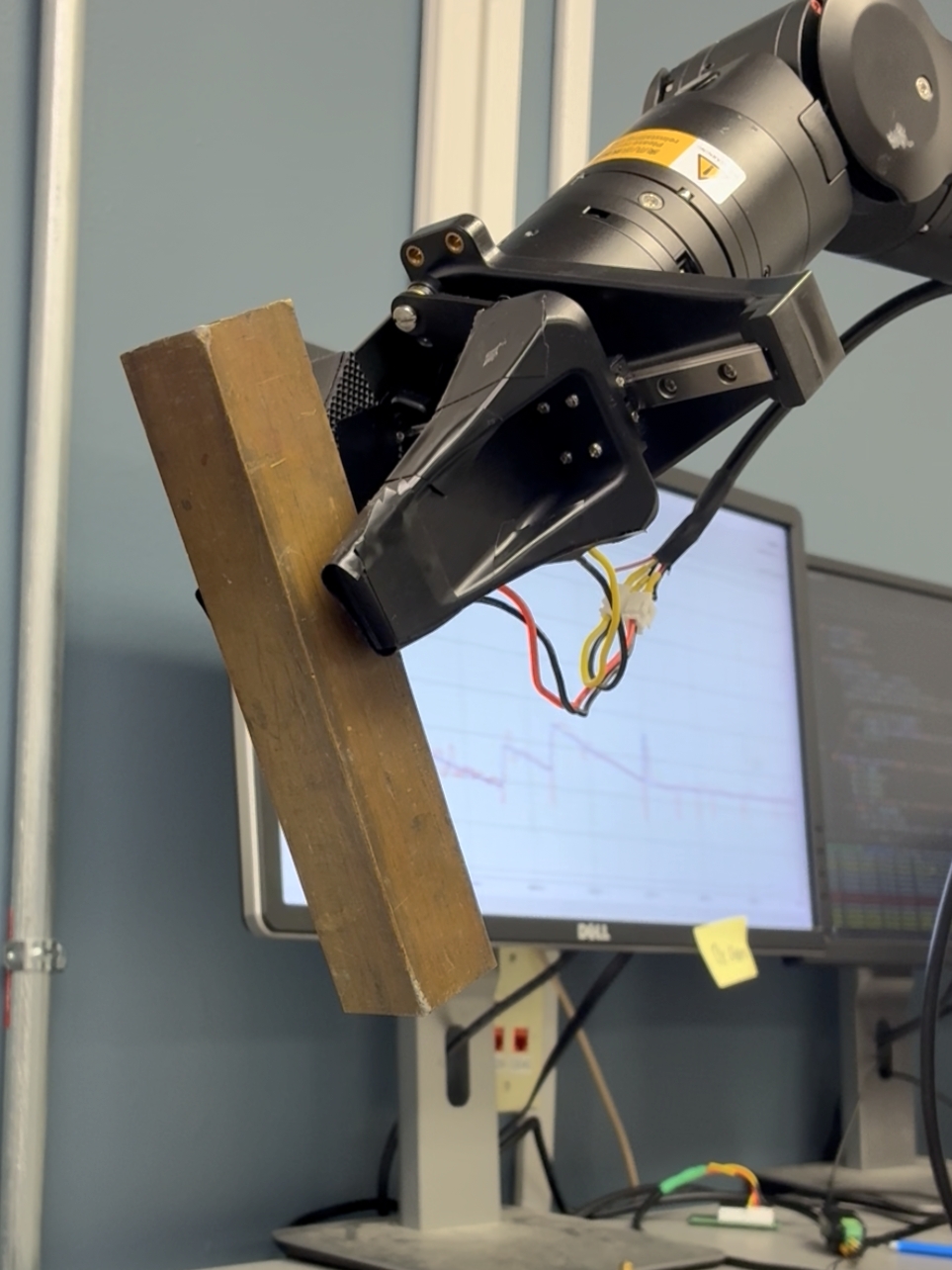}}
  \hspace{-3mm}
  \subfigure[]{\label{Fig_AL}
  \includegraphics[height=2.85cm]{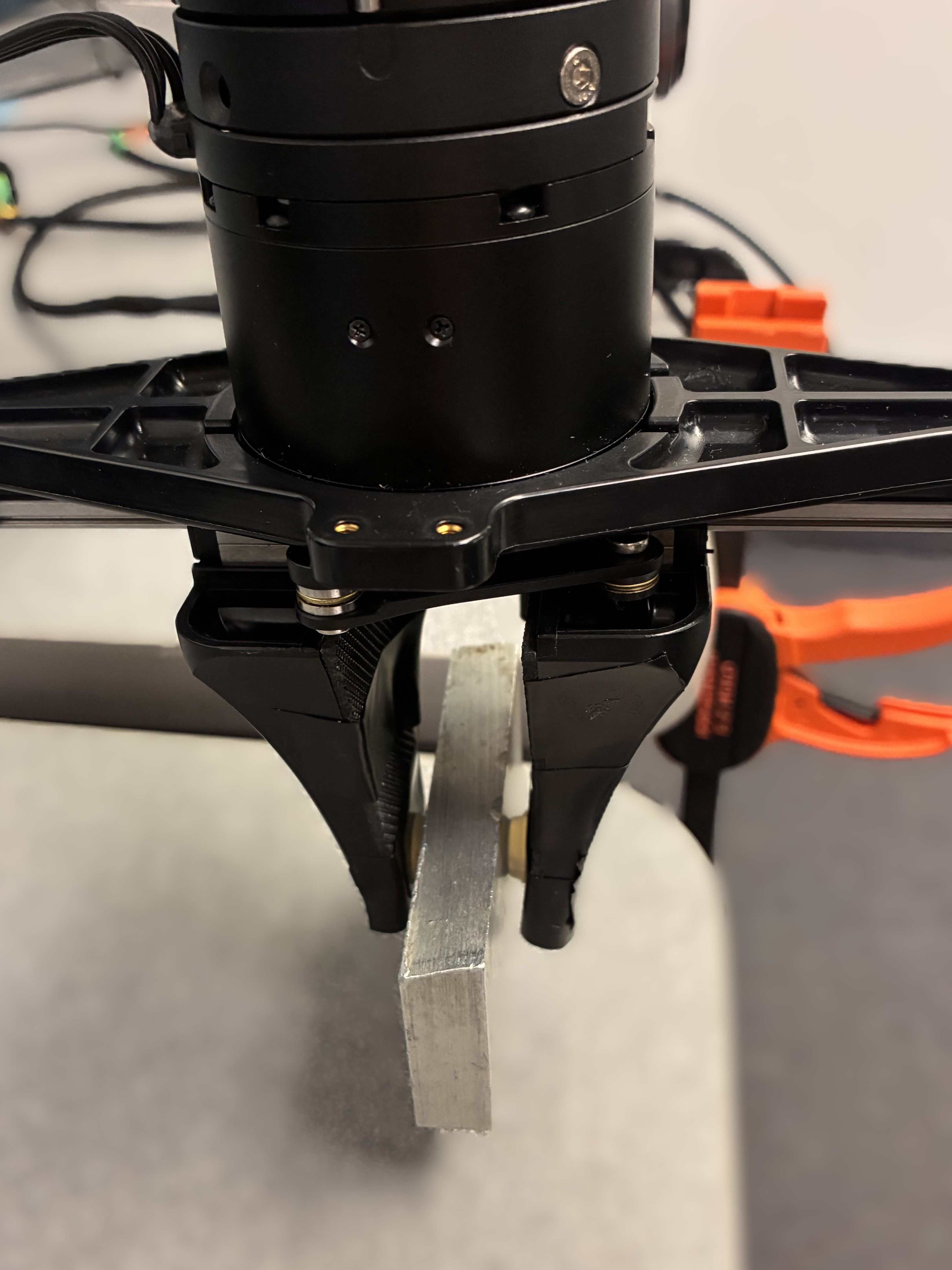}}
  \hspace{-3mm}
  \subfigure[]{\label{Fig_3D}
  \includegraphics[height=2.85cm]{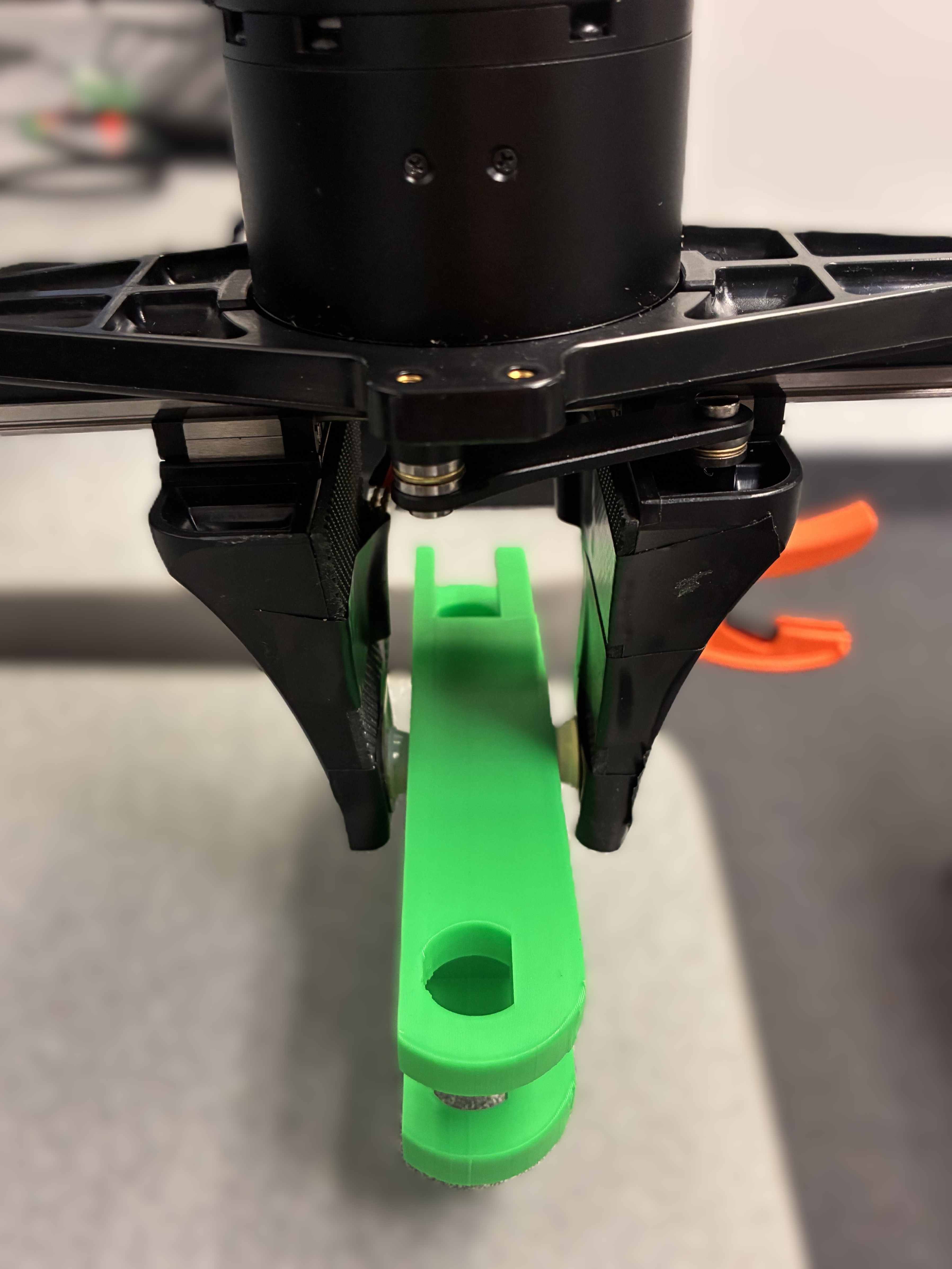}}
  \hspace{-3mm}
  \subfigure[]{\label{Fig_Foam}
  \includegraphics[height=2.85cm]{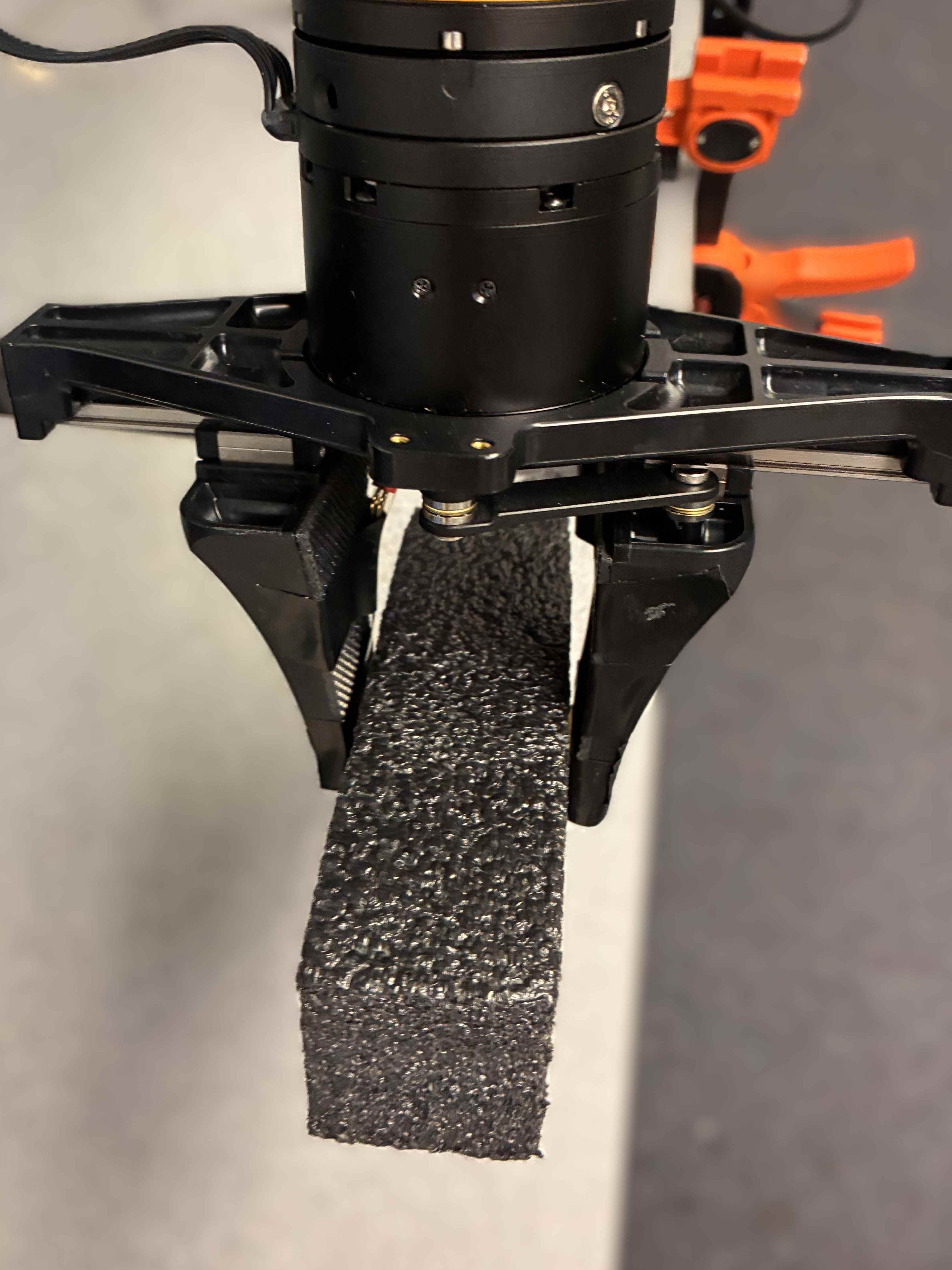}}
  \hspace{-3mm}
  \subfigure[]{\label{Fig_Silicon}
  \includegraphics[height=2.85cm]{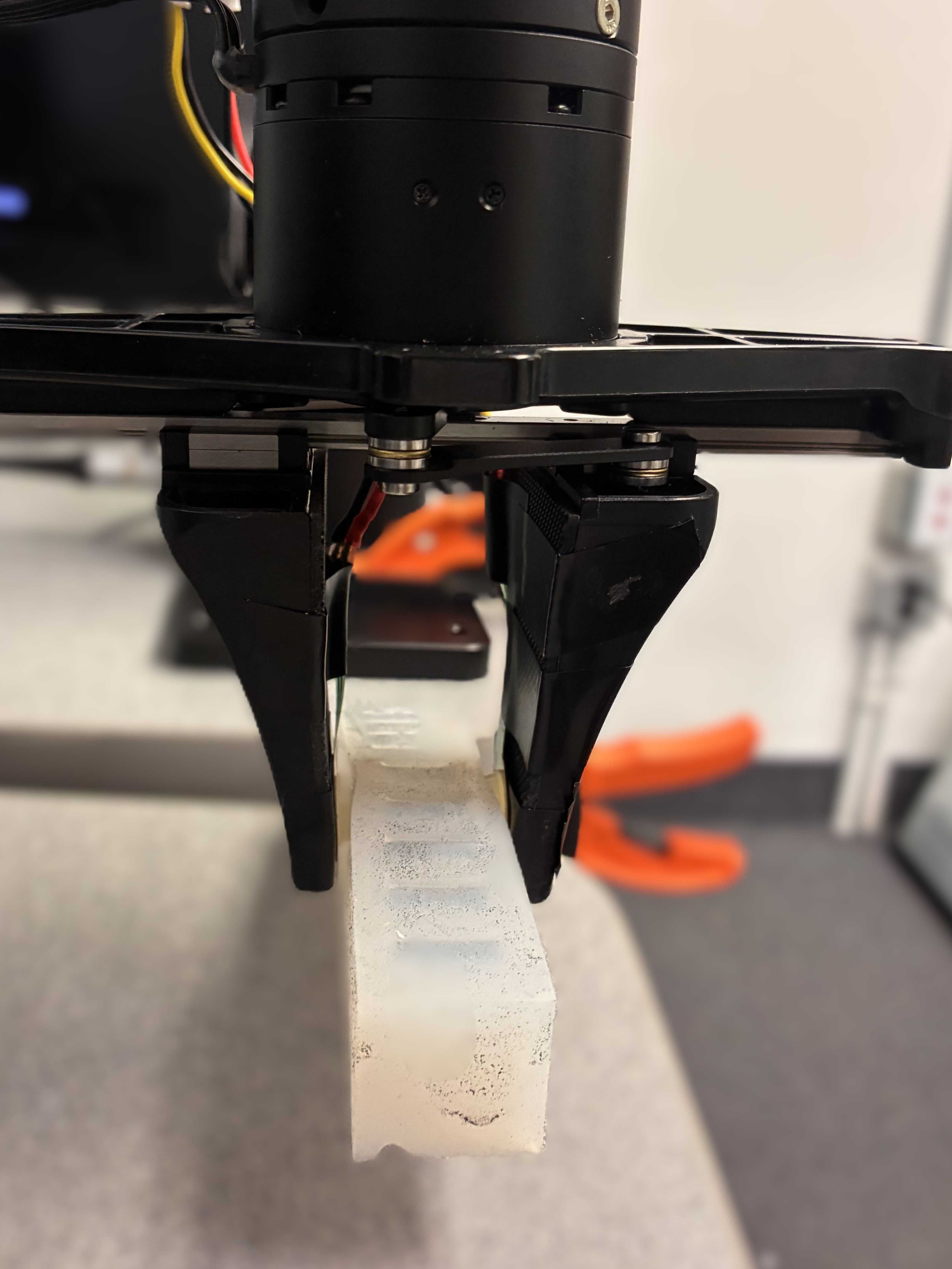}}
  \hspace{-3mm}
  \subfigure[]{\label{Fig_Rock}
  \includegraphics[height=2.85cm]{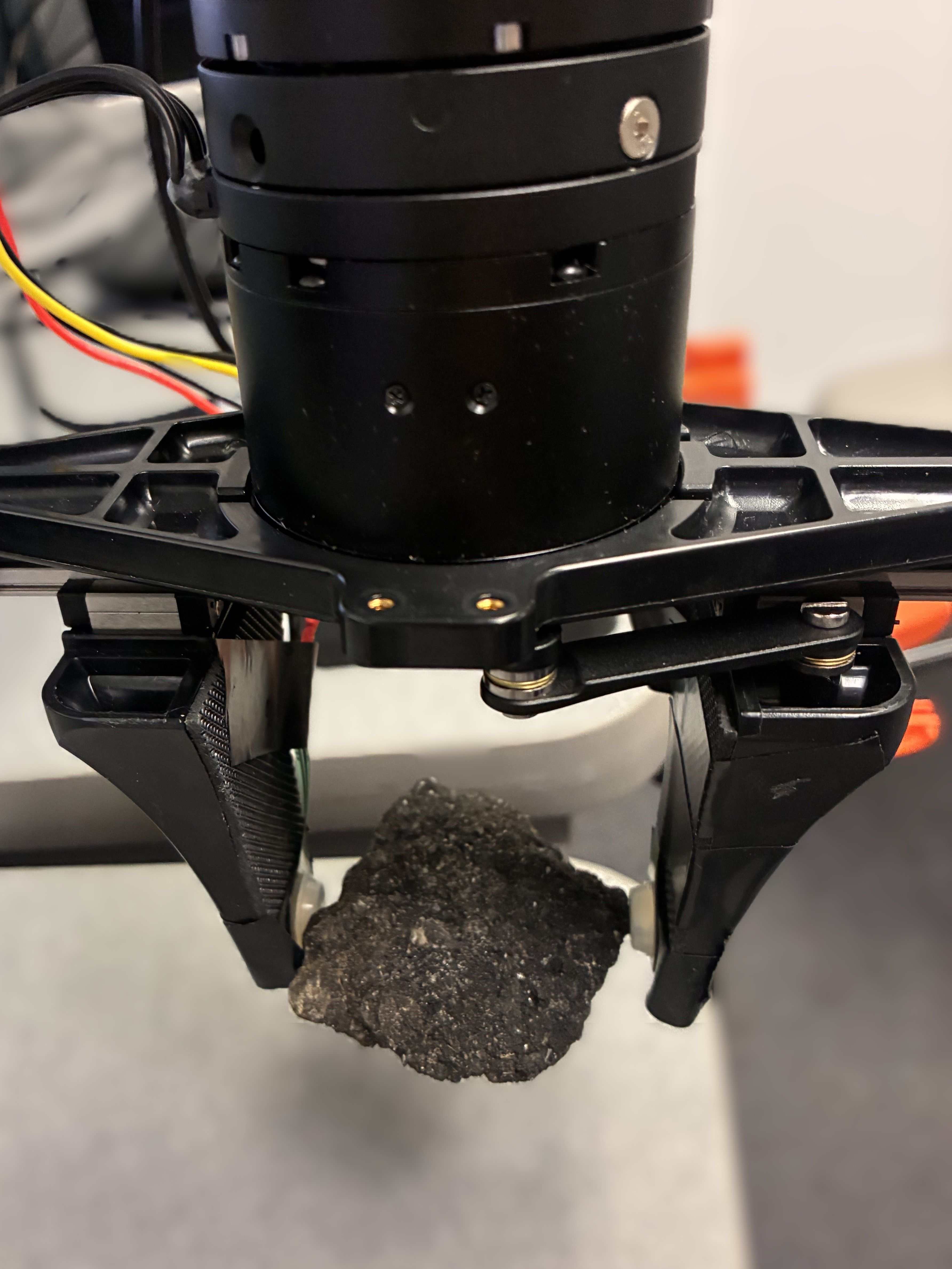}}
  \hspace{-3mm}
  \subfigure[]{\label{Fig_Card}
  \includegraphics[height=2.85cm]{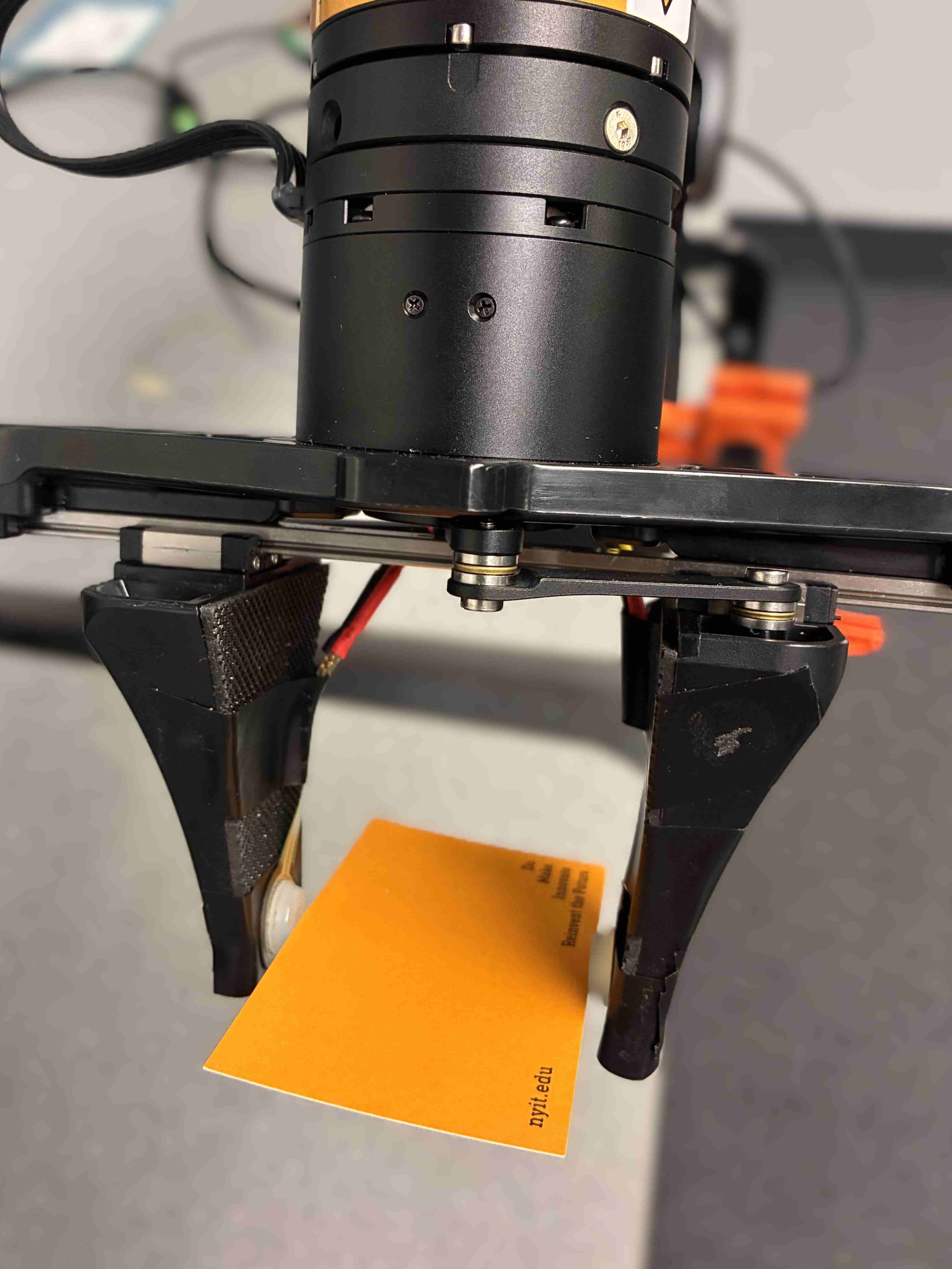}}
  \subfigure[]{\label{Fig_Chips}
  \includegraphics[height=2.85cm]{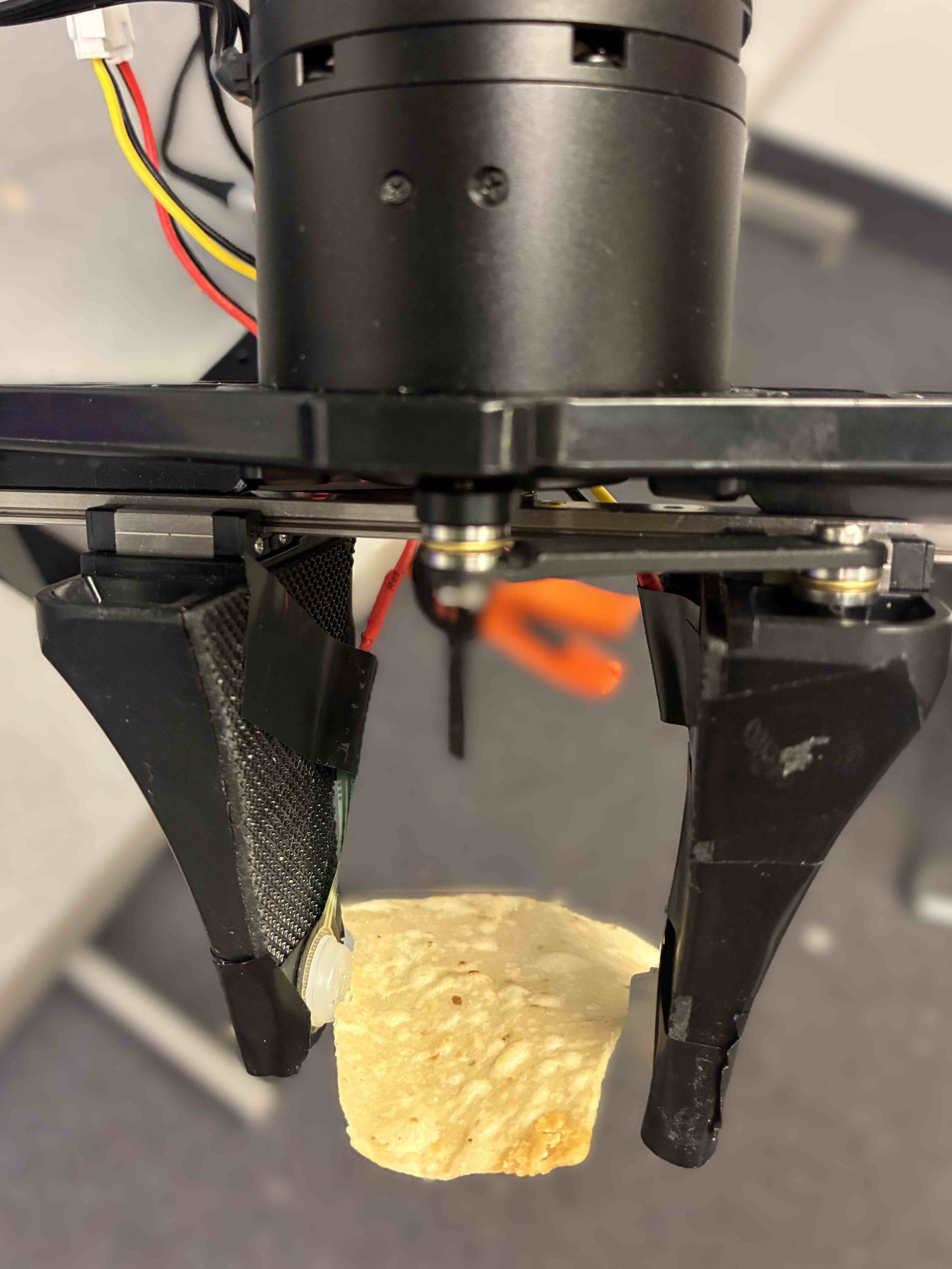}}
  \hspace{-3mm}
  \subfigure[]{\label{Fig_Maple}
  \includegraphics[height=2.85cm]{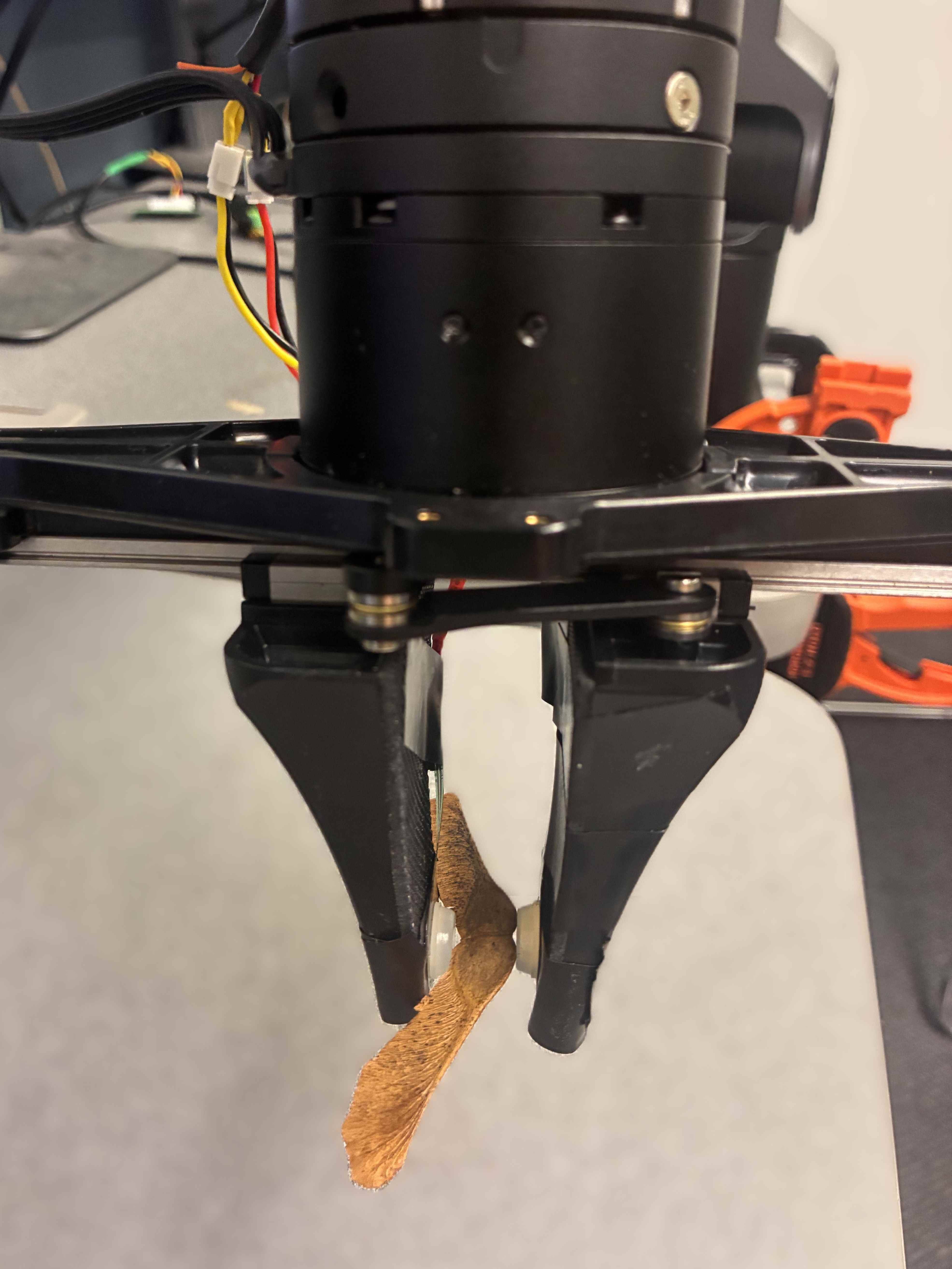}}
  \hspace{-3mm}
  \subfigure[]{\label{Fig_Coin}
  \includegraphics[height=2.85cm]{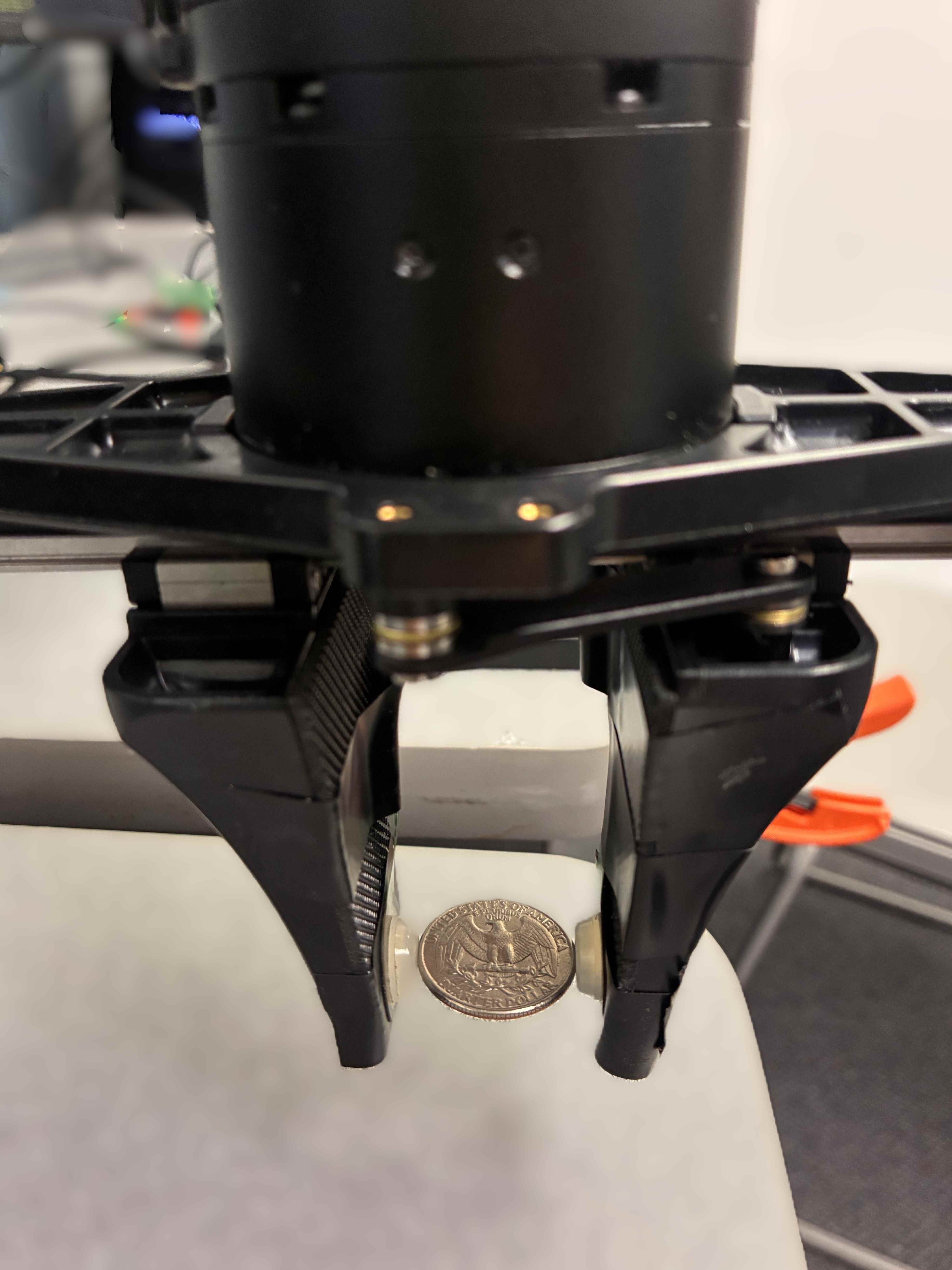}}
  \hspace{-3mm}
  \subfigure[]{\label{Fig_Nut}
  \includegraphics[height=2.85cm]{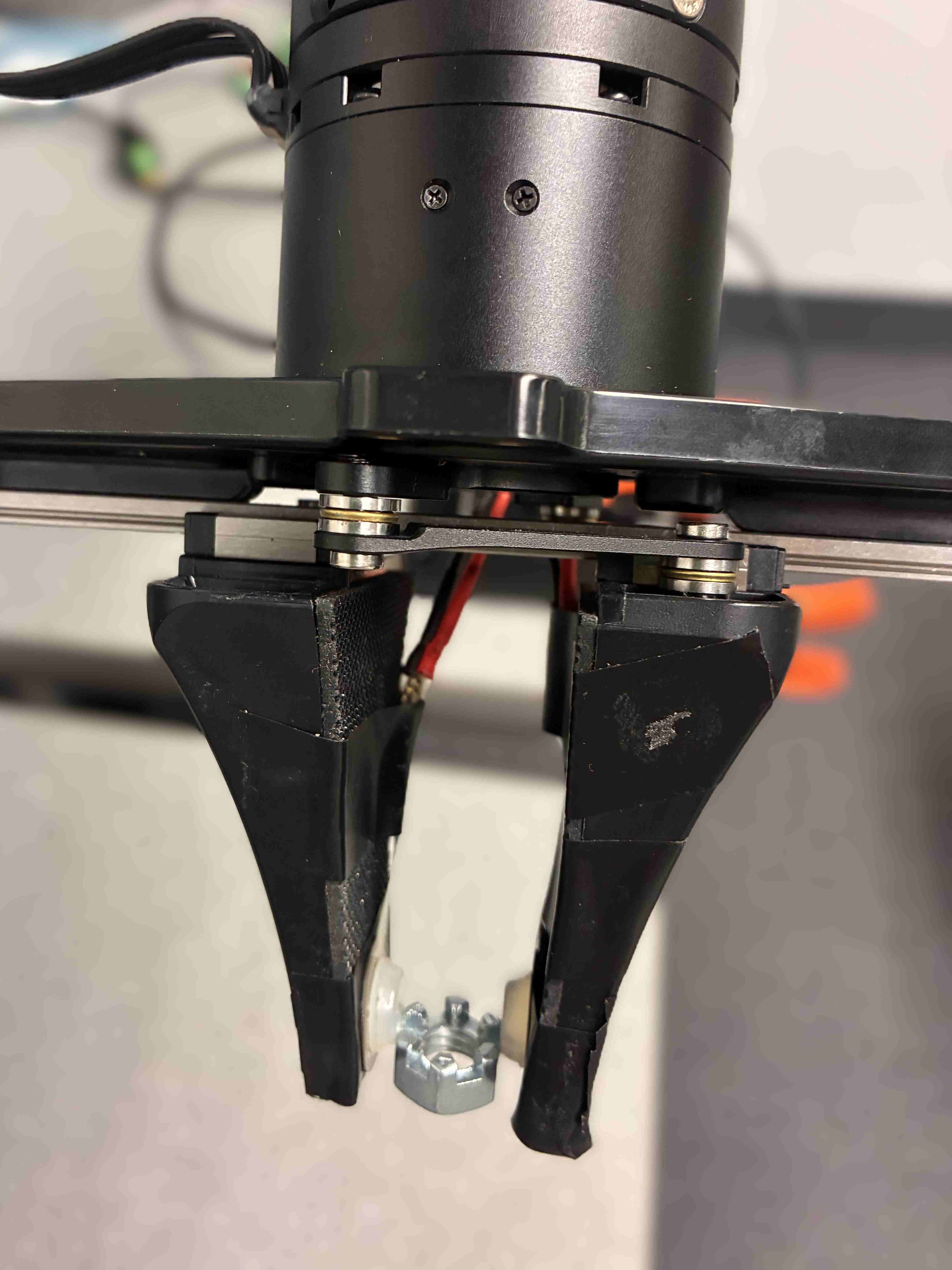}}
  \hspace{-3mm}
  \subfigure[]{\label{Fig_Tomato}
  \includegraphics[height=2.85cm]{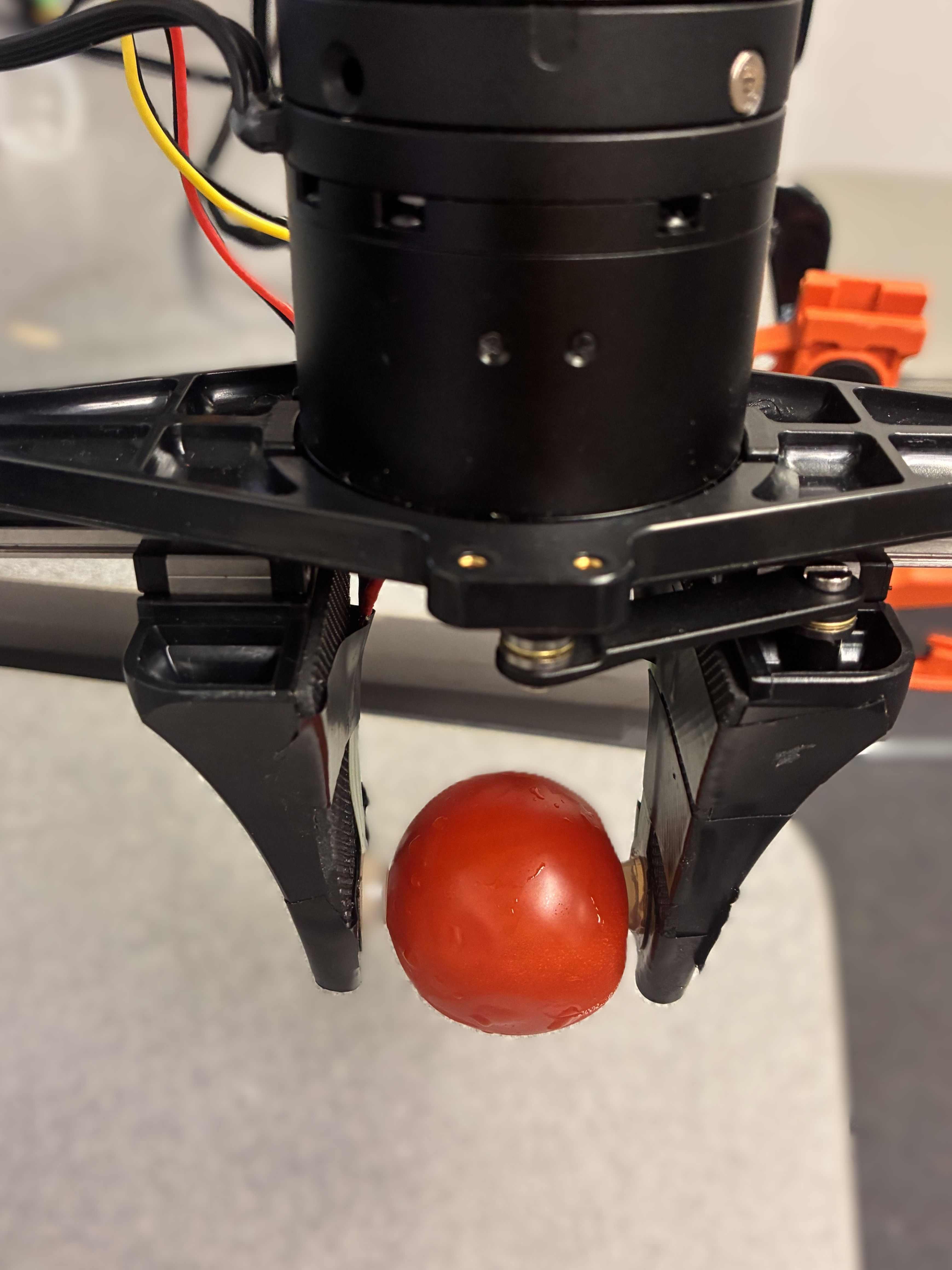}}
  \hspace{-3mm}
  \subfigure[]{\label{Fig_Blue}
  \includegraphics[height=2.85cm]{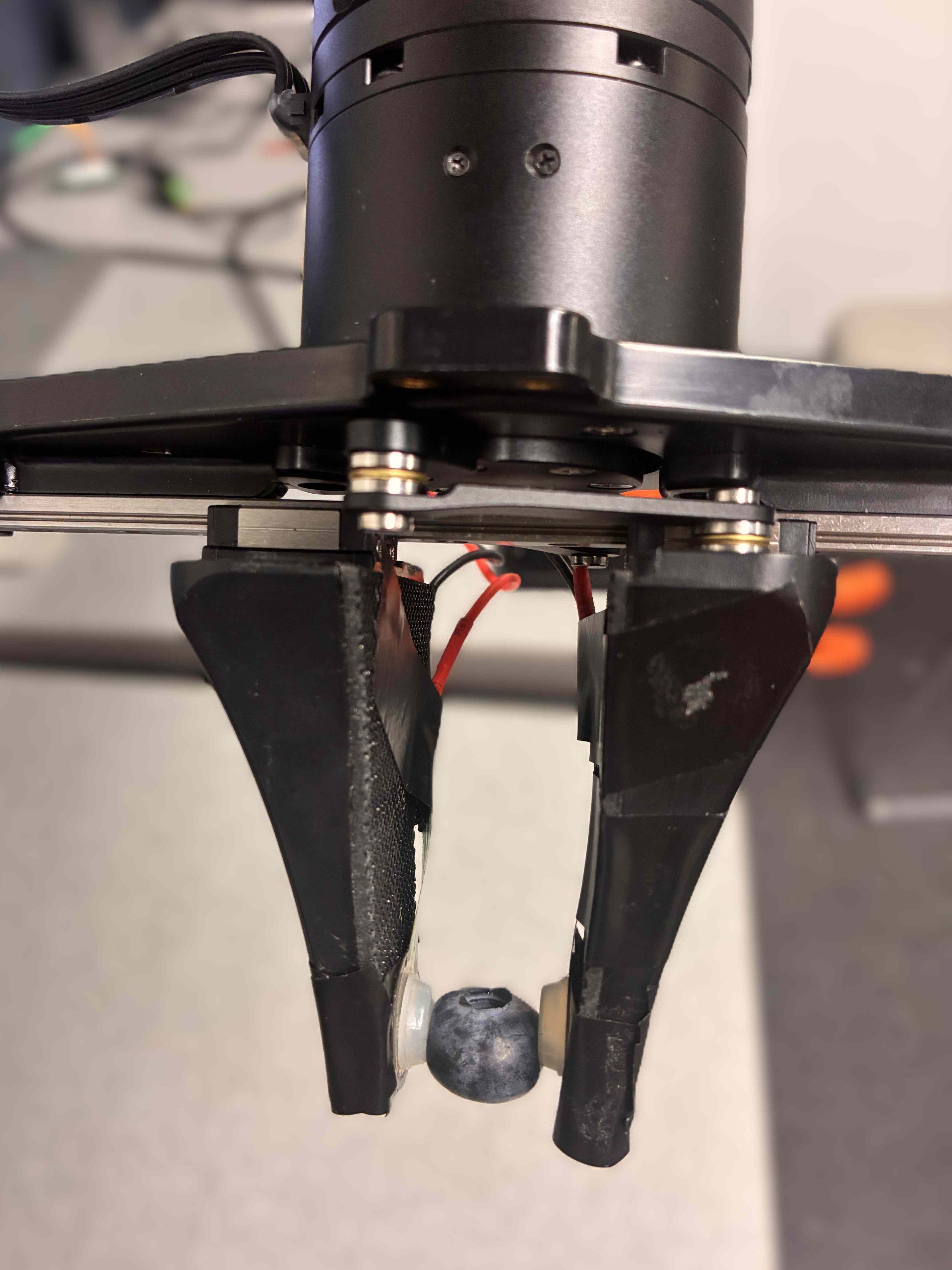}}
  \hspace{-3mm}
  \subfigure[]{\label{Fig_Marsh}
  \includegraphics[height=2.85cm]{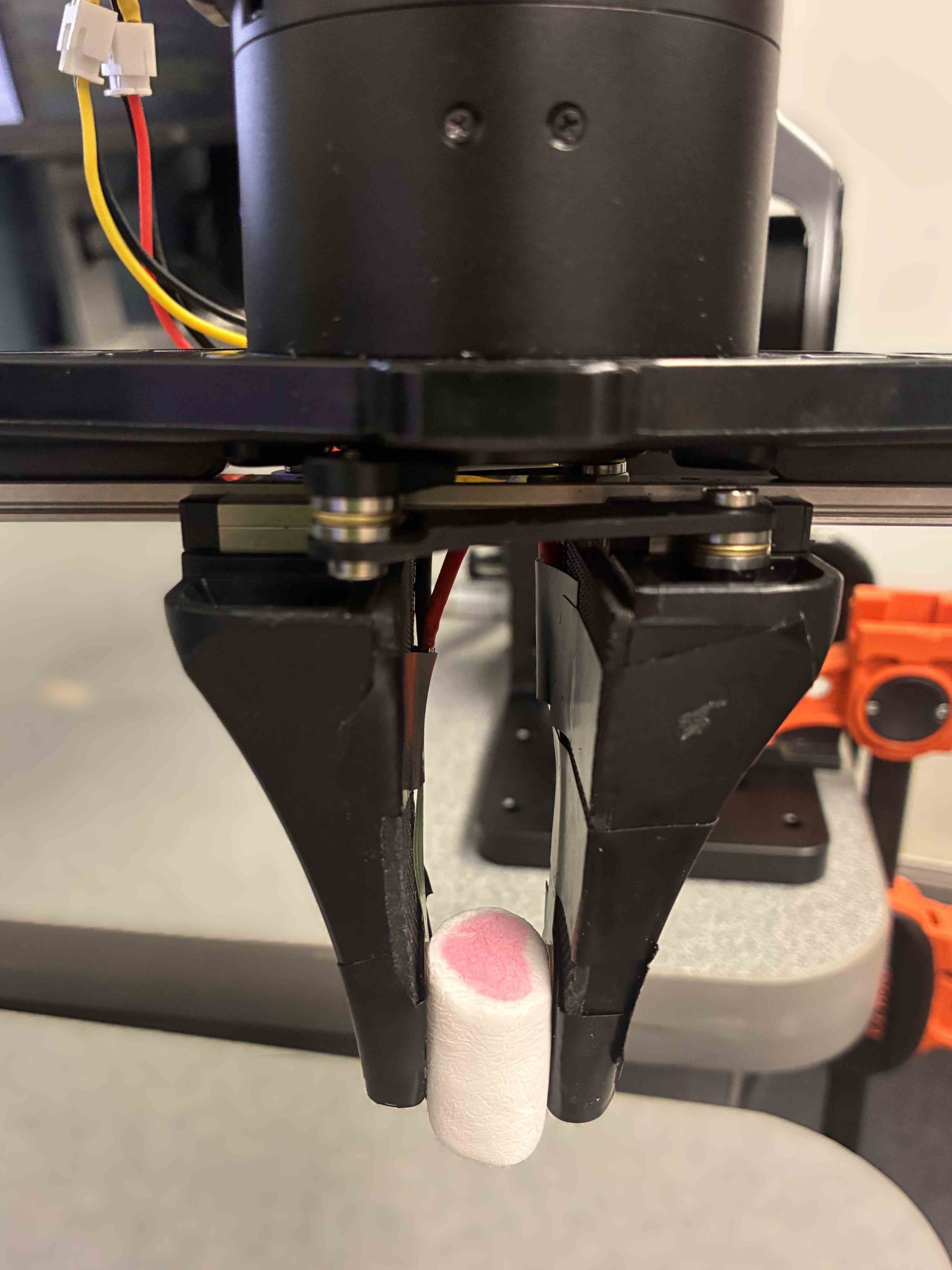}}
    \hspace{-3mm}
  \subfigure[]{\label{Fig_Tofu}
  \includegraphics[height=2.85cm]{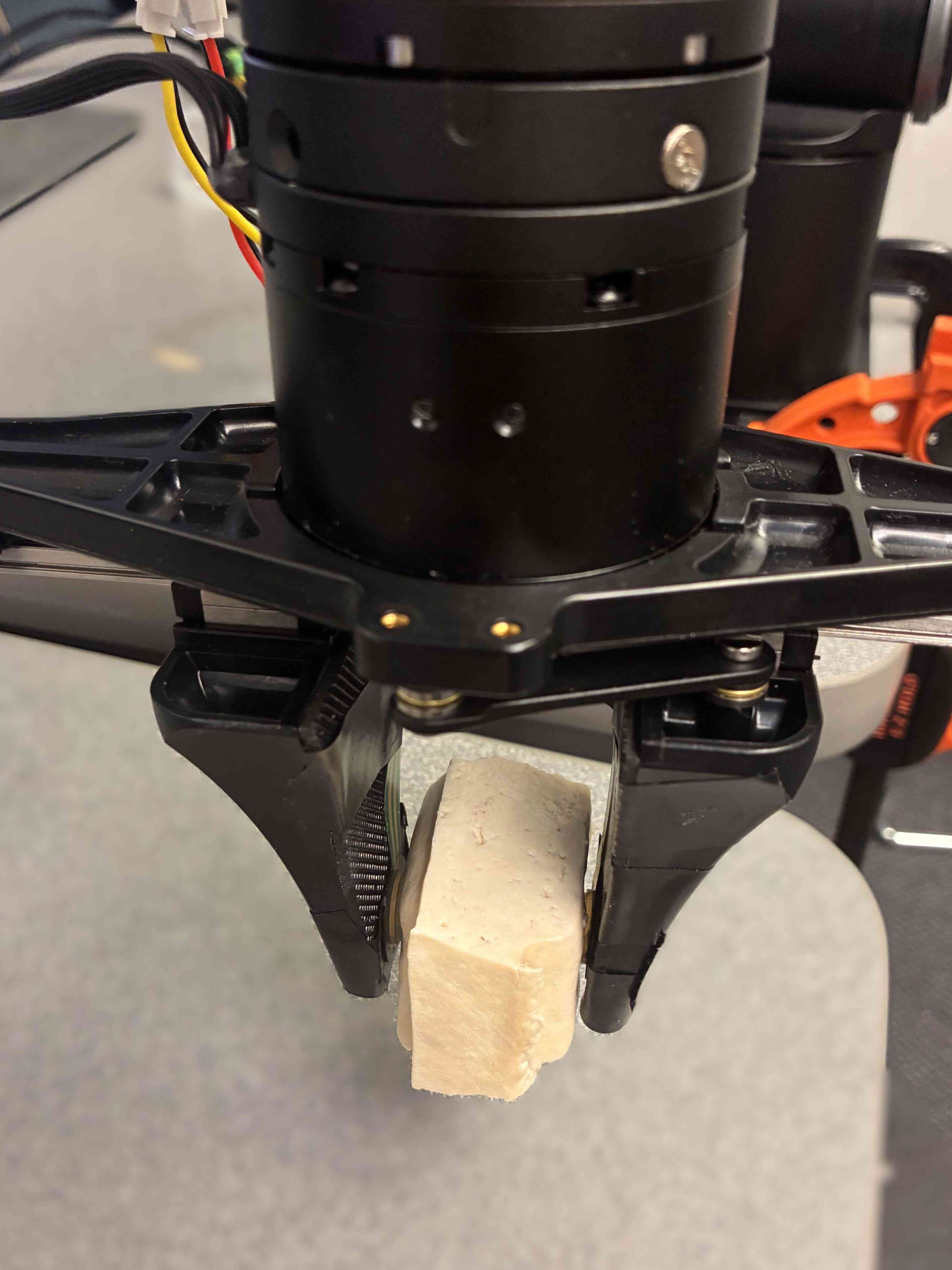}}
  \vspace{-2mm}
  \caption{Tactile grasping of different objects. (a) Copper block 1. (b) Copper block 2. (c) Aluminium block. (d) 3D-printed part. (e) Foam block. (f) Silicon block. (g) Rock. (h) Business card. (i) Chips from Moe's Southwest Grill. (j) Dry maple seed. (k) US quarter coin. (l) Hex nut fastener. (m) Tomato. (n) Blueberry. (o) Marshmallow. (p) Tofu.}
  \label{Fig_Object}
  \vspace{-3mm}
\end{figure*}

\renewcommand{\arraystretch}{1.35}
\setlength{\tabcolsep}{0.03in}
\begin{table*}
\centering
\caption{Grasping Object Properties}
\label{Tab_Objects}
\begin{tabular}{|c|c|c|c|c|c|c|c|c|c|c|c|c|c|c|c|c|c|c|c|}
\hline\hline

{Object} & \multicolumn{3}{c|}{Cylinder} & {Copper$1$} & {Copper$2$} & {Alum.} & {$3$D} & {Foam} & {Silicon} & {Rock} & {Card} & {Chips} & {Seed} & {Coin} & {Nut} & {Berry} & {Tomato} & Marsh. & Tofu\\ \hline

$m$ (g)& \multicolumn{3}{c|}{$29.4$}& $439.3$& $798.5$& $101.3$& $20.9$& $5.7$& $76.1$& $69.2$& $1.2$& $1.5$& $0.2$& $5.6$& $5.4$ & $2.2$ & $39.7$& $2.2$ & $38.7$  \\ \hline

$W_{\text{obj.}}^{\text{true}}$(mm)& $21.4$& $25.3$& $32.8$& $19.0$ & $35.0$& $12.8$& $25.2$& $37.5$& $24.4$& $51.0$& $50.4$& $36.1$& $8.8$& $23.7$& $14.3$& $16.0$ & $38.1$& $17.6$ &  $33.4$\\ \hline

Contact& Surf.& Line& Point& Surf.& Surf.& Surf.& Irreg.& Surface& Surf.& Irreg.& Line& Line& Line & Line & Surf.& Irreg.& Irreg.& Surf. & Surf. \\ \hline

Roughness& High& Low& N/A& Low& Low& Low& Low& Med.& Med.& High& N/A& N/A& N/A& N/A& Low& Low& Low & Med. & Med. \\ \hline

Hardness & High & High & High & High & High & High & High & Med. & Med. & High & Med. & Med. & High & High & High & Low & Low & Low & Low \\ \hline

$W_{\text{obj.}}$(mm)& $23.8$ & $27.4$ & $34.0$ & $21.8$ & $34.7$ & $14.8$ & $27.2$ & $36.9$ & $23.9$ & $49.5$ & $48.9$ & $38.5$ & $12.2$ & $25.8$ & $9.6$ & $17.6$ & $38.7$ & $11.9$ & $27.6$  \\ \hline

$\kappa_{\text{static}}$ (N/m) & $2040$ & $3331$ & $1375$ & $4000$ & $3637$ & $3113$ & $3146$ & $574$ & $415$ & $2515$ & $886$ & $1487$ & $1112$ & $2343$ & $2547$ & \color{blue}$118$ & $277$ & \color{blue}$89$ & \color{blue}$20$\\ \hline

$F_{x, {\min}}$(N)  & $2.5$ & $1.8$ & $1.0$ & $2.3$ &$2.9$& $1.8$ & $2.6$ & $0.9$ & $0.8$ & $1.1$ & $0.8$ & $1.0$ & $0.9$ & $1.2$ & $2.1$ & $0.8$ & $0.8$ & $0.7$ & $0.5$\\ \hline

$F_{x, {\max}}$(N)  & $7.3$ & $6.6$ & $3.7$ & $6.8$ & $7.6$ & $7.1$ & $7.3$ & $3.3$ & $3.0$& $4.3$ &$2.6$ & $3.8$ & $3.6$ & $4.6$ & $4.6$ & $1.9$ & $2.6$ & $1.5$ & $1.3$ \\ \hline\hline
\end{tabular}
\end{table*}

\begin{figure}
 \centering
 \includegraphics[width=7.5cm]{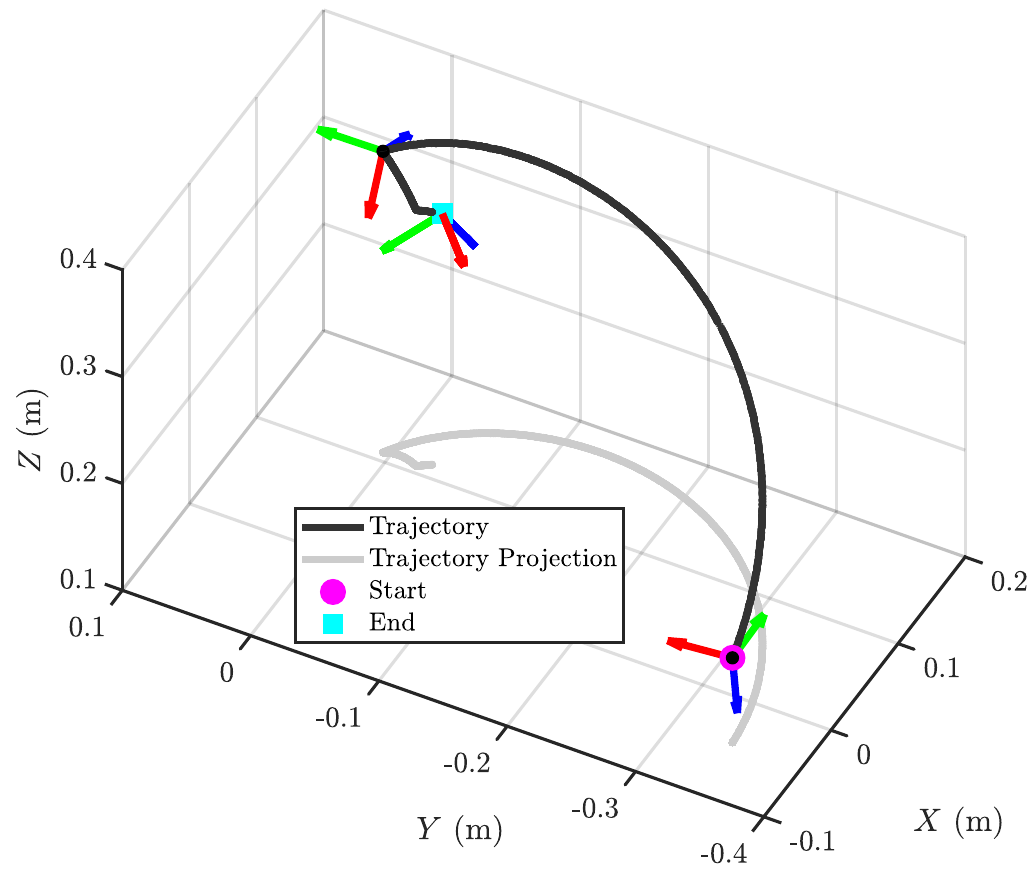}
 \vspace{-3mm}
 \caption{The trajectory of the arm movement.}
 \label{Fig_Arm_Traj}
 \vspace{-3mm}
\end{figure}

\begin{figure*}
  \centering
  \subfigure[Alum. Block]{\label{Fig_min_AL}
  \includegraphics[height=3.625cm]{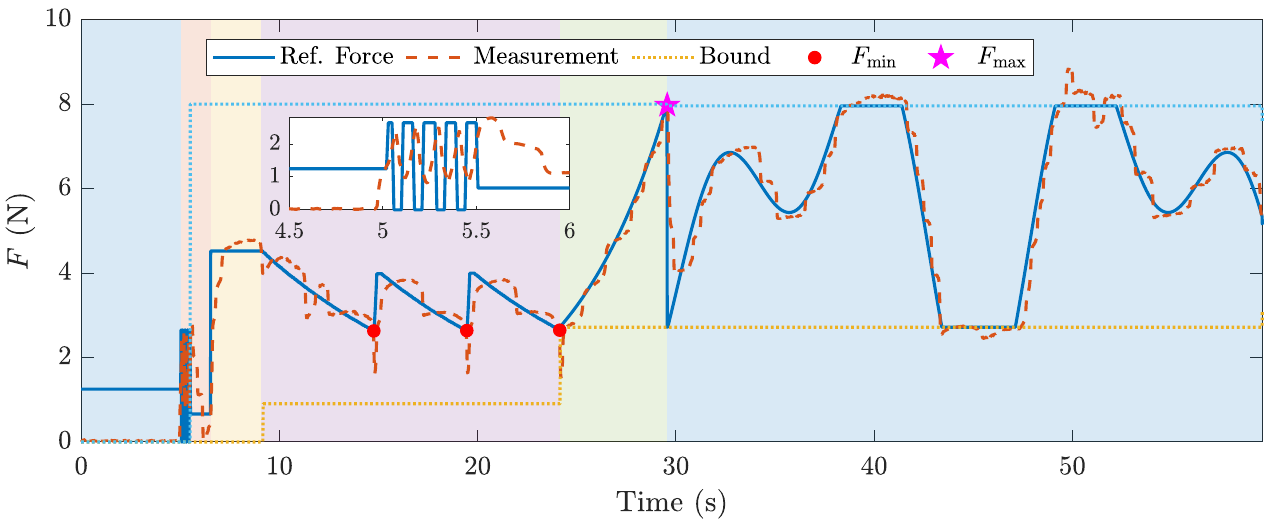}}
  \hspace{-4mm}
  \subfigure[Tomato]{\label{Fig_min_Tomato}
  \includegraphics[height=3.625cm]{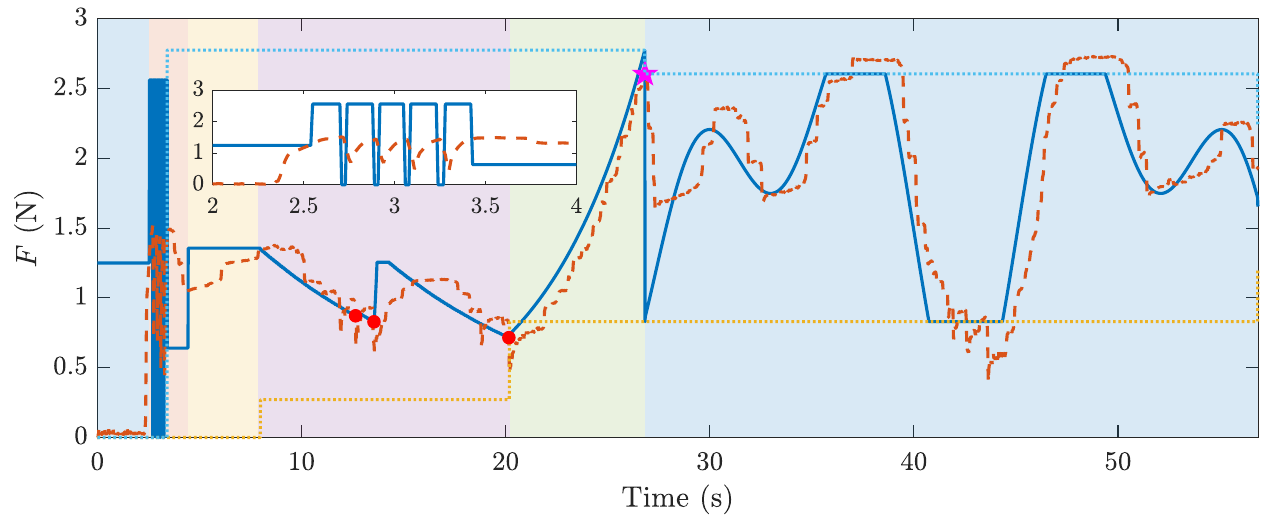}}
  \subfigure[Foam]{\label{Fig_min_Foam}
  \includegraphics[height=3.2cm]{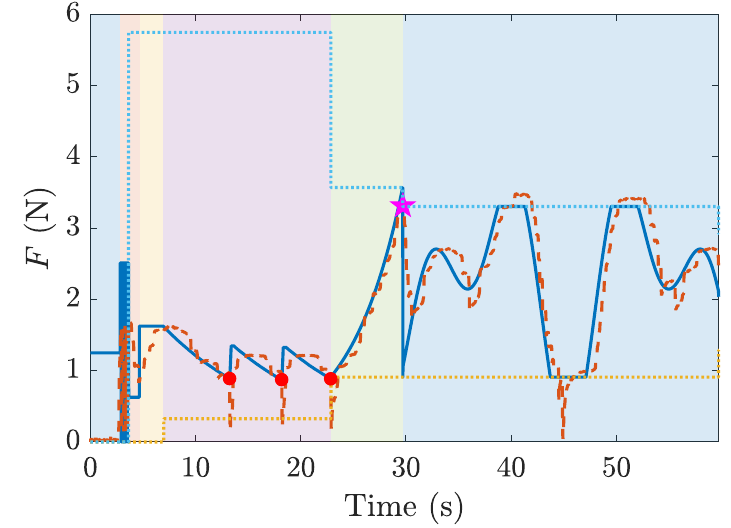}}
  \hspace{-4mm}
  \subfigure[Silicon]{\label{Fig_min_Silicon}
  \includegraphics[height=3.2cm]{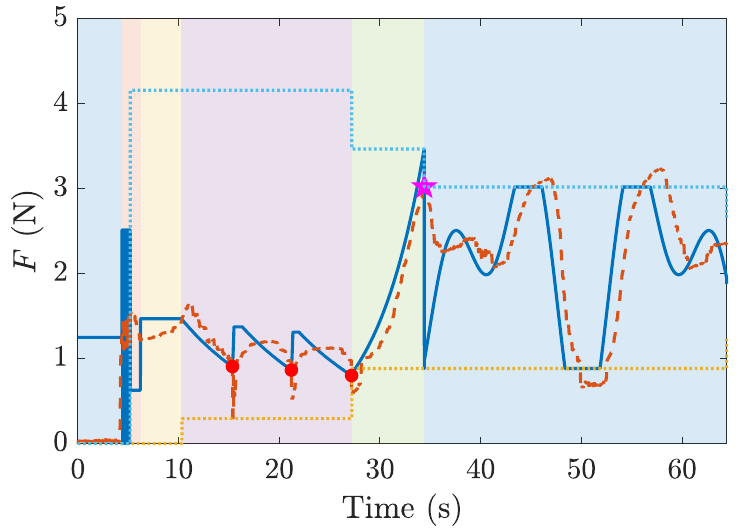}}
   \hspace{-4mm}
  \subfigure[Blueberry]{\label{Fig_min_Blue}
  \includegraphics[height=3.2cm]{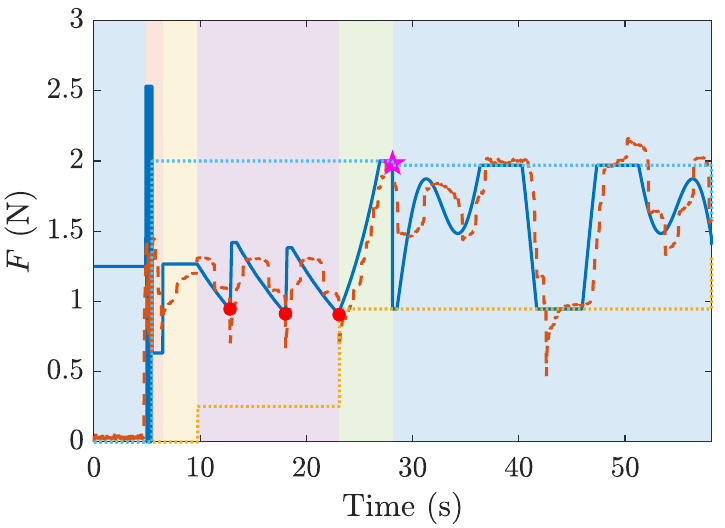}}
  \hspace{-4mm}
  \subfigure[Coin]{\label{Fig_min_Coin}
  \includegraphics[height=3.18cm]{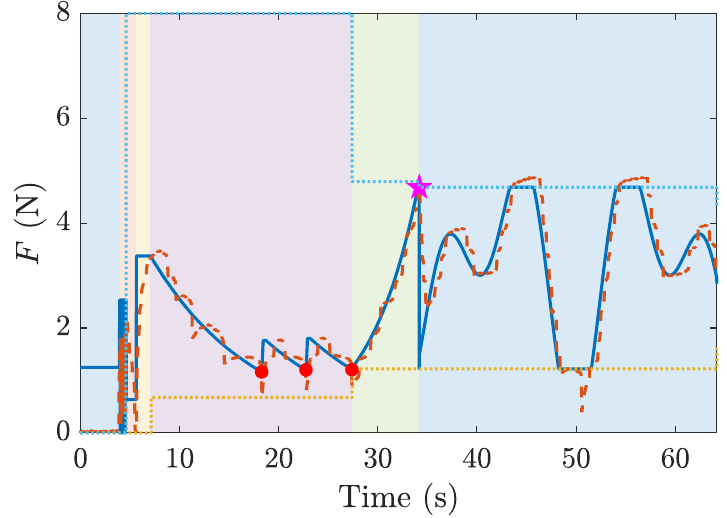}}
  \vspace{-1.5mm}
  \caption{Online object probing and constraints identification for various objects. (a) Alum. block. (b) Tomato. (c) Foam. (d) Silicon. (e) Blueberry. (f) US quarter coin. The entire process is divided into 6 sub-phases indicated by the colors, which are in the order of \texttt{approaching}, \texttt{probing}, \texttt{contact force applying}, \texttt{minimum force finding}, \texttt{maximum force verification}, and \texttt{secure holding}. Note after \texttt{contact force applying}, objects are in the air lifted by the robotic arm. The \enquote{$\bullet$} indicates the minimum forces, and \enquote{$\star$} the maximum forces.}
  \label{Fig_Probing}
  \vspace{-2mm}
\end{figure*}

\subsection{Results on Tactile Grasping}

\subsubsection{Objects Probing and Constraints Identification}


The proposed FSM probes the object's geometry and mechanical properties and actively identifies secure-and-gentle force constraints. We test the FSM on various objects with different sizes, weights, stiffnesses, and surface roughnesses, ranging from items as fragile and light as a marshmallow ($m=2$~g) to those as hard and heavy as a metal block. The entire process is divided into multiple phases, including \texttt{approaching}, \texttt{probing}, \texttt{contact force applying}, \texttt{minimum force finding}, \texttt{maximum force verification}, and \texttt{secure holding}. In the first two phases, the size and stiffness of the object are evaluated, and the initial force is calculated. In the \texttt{contact force applying} phase, contact between the gripper and the object is established. Afterward, the object is lifted into the air for transportation, and the secure-and-gentle force constraints are identified and verified online. Preliminary PID control is used for grasping force control. Fig.~\ref{Fig_Object} shows snapshots of the dynamic grasping of different objects. The gripper is commanded to grasp the object and move it to a different location; the trajectory and pose of the gripper are shown in Fig.~\ref{Fig_Arm_Traj}. In Fig.~\ref{Fig_Probing}, we show the results for several representative objects: an aluminum block, a tomato, a foam block, a silicone block, a blueberry, and a US quarter coin.

\begin{figure}
  \centering
  \includegraphics[width=8.5cm]{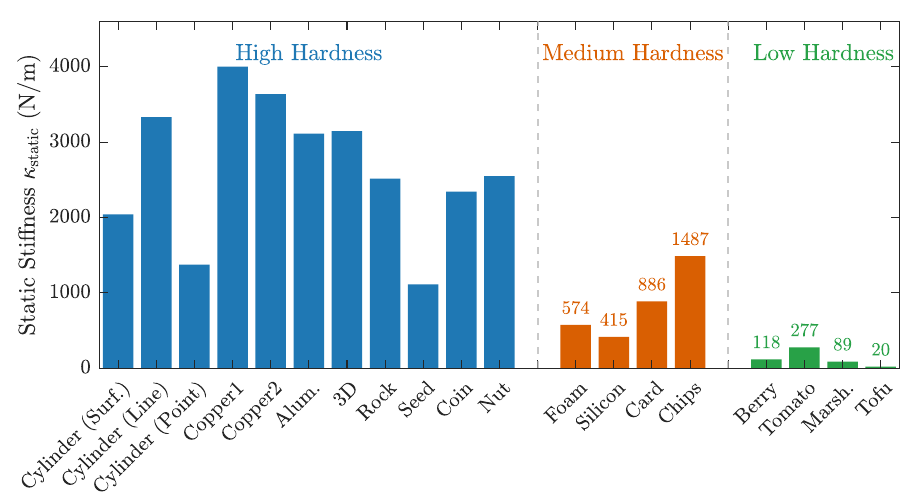}
  \vspace{-2mm}
  \caption{Experimentally evaluated stiffness of objects comparison. }
  \label{Fig_Stiffness}
  \vspace{-3mm}
\end{figure}

\begin{figure*}
  \centering
  \subfigure[]{
  \label{Fig_NMPC_AL}
  \includegraphics[height=2.9cm]{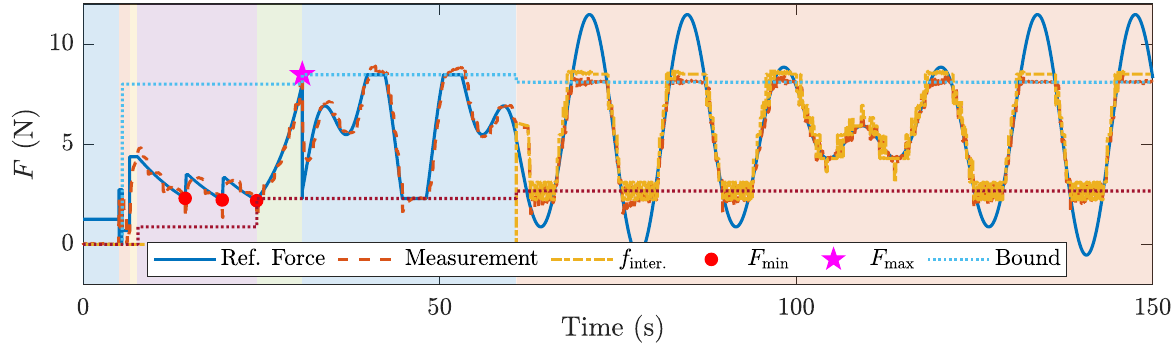}}
  \hspace{-3mm}\vspace{-2mm}
  \subfigure[]{
  \label{Fig_NMPC_AL_Error}
  \includegraphics[height=3cm]{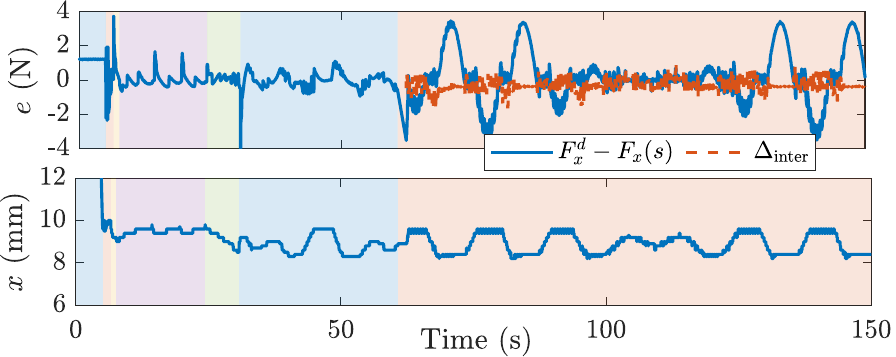}}
    \subfigure[]{
  \label{Fig_NMPC_Tomato}
  \includegraphics[height=3cm]{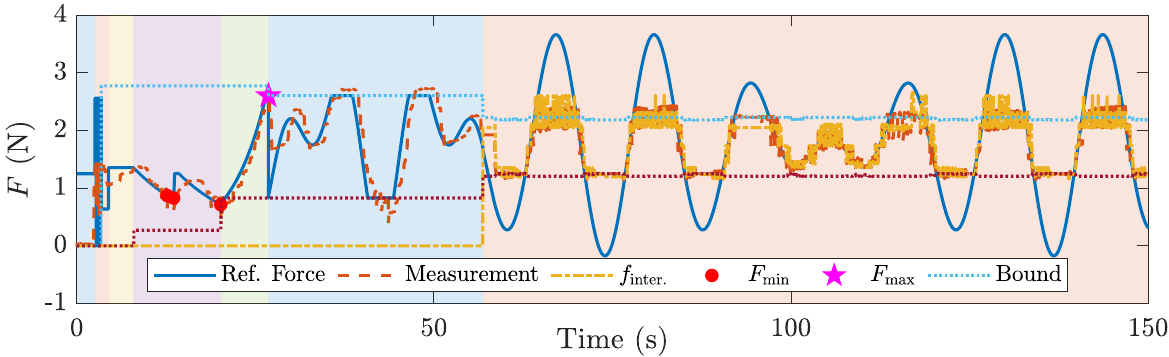}}
  \hspace{-3mm}\vspace{-2mm}
  \subfigure[]{
  \label{Fig_NMPC_Tomato_Error}
  \includegraphics[height=3cm]{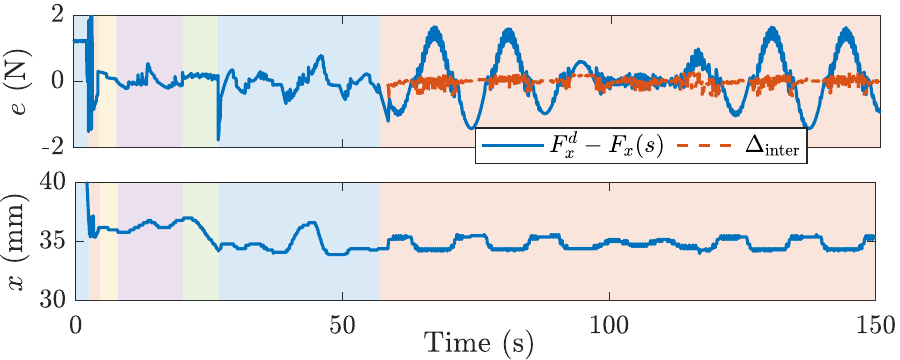}}
  \subfigure[]{
  \label{Fig_NMPC_Silicon}
  \includegraphics[height=3cm]{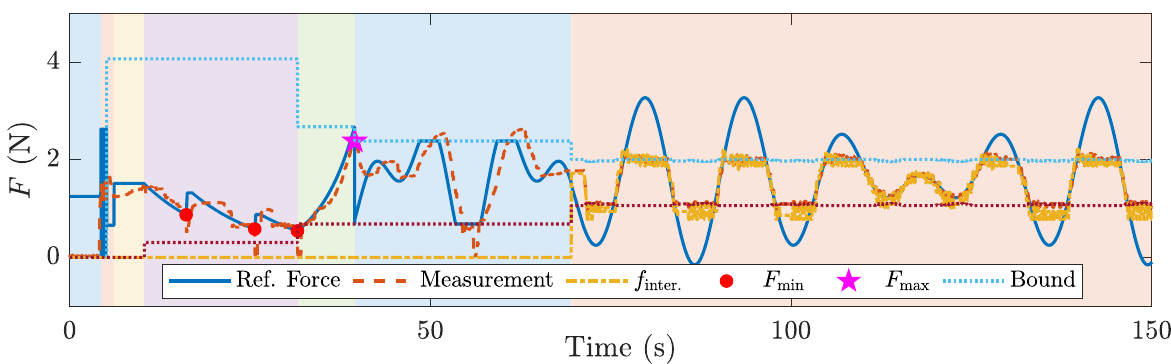}}\vspace{-2mm}
  \hspace{-3mm}
  \subfigure[]{
  \label{Fig_NMPC_Silicon_Error}
  \includegraphics[height=3cm]{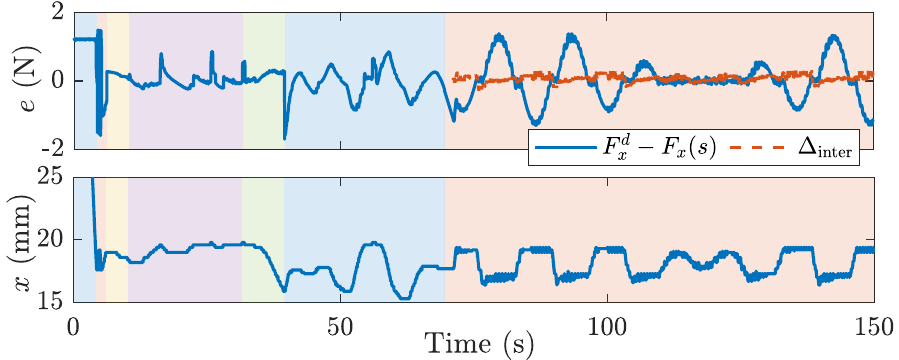}}
    \subfigure[]{
  \label{Fig_NMPC_Maple}
  \includegraphics[height=3cm]{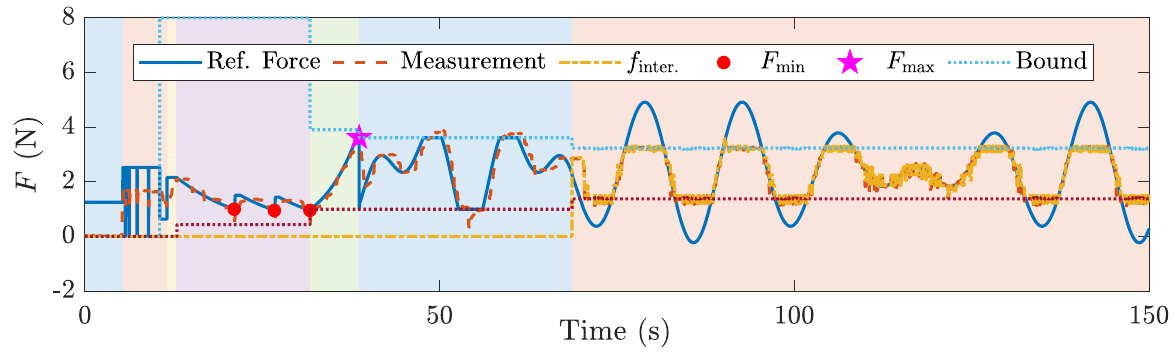}}
  \hspace{-3mm}
  \subfigure[]{
  \label{Fig_NMPC_Maple_Error}
  \includegraphics[height=3cm]{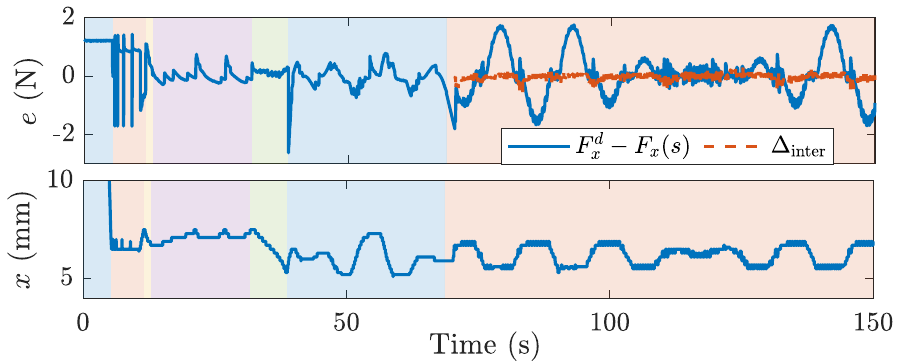}}
  \vspace{-2mm}
  \caption{Online object dynamic learning and nonlinear MPC control. (a) Alum. block. (b) Tomato.  (c) Silicon. (d) Maple seed.}
  \label{Fig_NMPC}
  \vspace{-2mm}
\end{figure*}

For instance, in Fig.~\ref{Fig_min_AL}, after initial contact is established, the gripper squeezes the object 5 times to evaluate its stiffness. The initial contact force for lifting the aluminum block is determined as $4.5$~N using~\eqref{Eq_Grasp_Force}. The grasping force is then commanded, and the robotic arm lifts the aluminum block at around $9$~s. At the same time, the initial constraints $F_{\text{obj}, \min}$ and $F_{\text{obj},\max}$ are updated based on the stiffness and allowable grasping depth. Since no slip occurs, the initial force is greater than $F_{\text{obj},\min}$; namely, $\mathds{C}_{\text{secure}}$ is inherently satisfied, which triggers constraint hunting to identify the true $F_{\text{obj},\min}$ and $F_{\text{obj},\max}$. In total, three trials of relaxation and reflex are performed. We reduce the grasping force at a low rate ($2\%$) to avoid the effects of PID overshoot and oscillation until a slip ($25\%$ error) occurs, at which point the current reference force is labeled as the lower bound (see \enquote{$\bullet$} in Fig.~\ref{Fig_min_AL}) and $\mathbb{D}_{F,\min}$ is updated. Immediately, the grasping force is increased based on the identified lower bound. After three trials, the lower bound is identified and $\mathds{C}_{\text{secure}}$ is updated. Furthermore, we verify the initially estimated maximum grasping force. The grasping force is gradually increased (at a rate of $2\%$). In the case of crushing or permanent deformation, the grasping force feedback will drop. Therefore, if a $20\%$ error is identified, the current feedback force is labeled as the true $F_{\text{obj},\max}$ and $\mathds{C}_{\text{gentle}}$ is updated; see \enquote{$\star$} in Fig.~\ref{Fig_min_AL}. We tested the FSM on various objects: a smooth-surfaced and soft tomato in Fig.~\ref{Fig_min_Tomato}, soft and light foam in Fig~\ref{Fig_min_Foam}, an extremely soft silicone block in Fig.~\ref{Fig_min_Silicon}, an extremely soft and extremely light blueberry in Fig.~\ref{Fig_min_Blue}, and a light but hard coin in Fig.~\ref{Fig_min_Coin}. These objects exhibit different grasping forces and distinct constraints. A summary of all objects' properties and identified constraints is shown in Table~\ref{Tab_Objects}, including a qualitative assessment of each object's hardness, surface roughness, and contact type, alongside the experimentally obtained stiffness and constraints. These objects present significantly distinct properties, thereby demonstrating the broad operating bandwidth of the proposed design and approach.

The evaluated stiffness is shown in Fig.~\ref{Fig_Stiffness}. We should note that the absolute value of the evaluated stiffness may not represent the true physical stiffness. However, the results align with qualitative assessments of object hardness based on everyday experience. Objects with high hardness display large stiffness, for instance, a copper block or a rock, while tofu and marshmallows have the lowest stiffness (less than 100). In the experiment, we normalized the stiffness to the range $[\kappa_{\min}, \kappa_{\max}]$ to ensure a valid initial force calculation in~\eqref{Eq_Grasp_Force}. Under this normalization, the initial force and maximum force might be overestimated, making the relaxation and reflex mechanism even more necessary.

With the identified constraints, we apply PID control to track a reference grasping force. The reference grasping force is filtered using the identified constraints. Under preliminary PID control, the grasping force follows the reference. However, because the reference changes rapidly and PID control does not account for the gripper-object interaction dynamics, overshoot and large tracking errors occur. This overshoot leads to violations of the secure-and-gentle constraints; see Figs.~\ref{Fig_min_AL} and~\ref{Fig_min_Tomato} at $t=50$~s and $t=45$~s, respectively. For light or rough-surfaced objects, small overshoots do not cause significant failures; see Figs.~\ref{Fig_min_Tomato} through~\ref{Fig_min_Coin}. However, for heavy and smooth objects, these violations may cause slippage and failure (see the discussion in Fig.~\ref{Fig_min_Copper2}), which necessitates constraint-embedded MPC.

\subsubsection{MPC-based Control}

Building upon the above preliminary results, we present the online interaction dynamics learning and control framework for dynamic grasping. The robotic arm is commanded to grasp and transport the unknown object from one location to another. Note that the object properties listed in Table~\ref{Tab_Objects} are not used. For each experimental trial, even for the same object, the object properties are re-evaluated, and the interaction dynamics model is re-learned online, complying with assumption $\mathcal{A}_2$ of zero prior knowledge.

Fig.~\ref{Fig_NMPC} shows the results for four representative objects: an aluminum block in Fig.~\ref{Fig_NMPC_AL}, a tomato in Fig.~\ref{Fig_NMPC_Tomato}, a silicone block in Fig.~\ref{Fig_NMPC_Silicon}, and a maple seed in Fig.~\ref{Fig_NMPC_Maple}. After the maximum force is verified (see \enquote{$\star$} in Fig.~\ref{Fig_NMPC_AL}, for instance), an excitation reference force is applied under PID control for $25$~s, during which interaction dynamics training data is collected. Out of $2500$ data points, we randomly select $1500$ points for online interaction model training according to~\eqref{Eq_gp_interaction}. The model is trained instantaneously on the GPU. With the case-specific GP-based gripper-object interaction dynamics established, we execute the proposed nonlinear MPC control to track the time-varying grasping force while updating the constraints $\mathds{C}_{\text{secure}}^{\text{Robust}}$ and $\mathds{C}_{\text{gentle}}^{\text{Robust}}$. As shown in Fig.~\ref{Fig_NMPC_AL}, once the interaction dynamics model is trained, the grasping force prediction becomes available and closely follows the actual measurement, yielding an accurate dynamics model for MPC prediction. During the MPC horizon, the reference force is unconstrained; however, the actual grasping force remains within the safe region and tracks the time-varying reference. Both the interaction dynamics prediction error and MPC control error are shown in Fig.~\ref{Fig_NMPC_AL_Error}. The model prediction error remains under $0.4$~N. The prediction uncertainties are reflected in the constraint updates, causing both the lower and upper bounds to tighten inward. The peak tracking error under PID control reaches $1$~N, whereas under the proposed nonlinear MPC, it stays below $0.5$~N.

\renewcommand{\arraystretch}{1.35}
\setlength{\tabcolsep}{0.035in}
\begin{table*}
\centering
\caption{Statistics of the  force error ($\times 10^{-2}$~N) under different control methods (Mean and one standard deviation) }
\label{Tab_Eror}
\begin{tabular}{|c|c|c|c|c|c|c|c|c|c|c|}
\hline\hline
Control & Alum. & Copper1 & Foam & Silicon & Tomato & Rock & 3D Part & Coin & Card  \\ \hline
PID & $35.2 \pm 40.1$ & $32.7 \pm 27.7$ & $24.4 \pm 20.0$ & $37.6 \pm 34.7$ & $24.8 \pm 22.7$ & $26.7 \pm 28.6$ & $38.2 \pm 50.0$ & $40.6 \pm 51.1$ & $37.8 \pm 29.5$  \\ \hline

LMPC & $25.7 \pm 20.6$ & $28.4 \pm 21.3$ & $13.1 \pm 10.9$ & $8.6 \pm 5.2$ & $11.9 \pm 9.4$ & $24.8 \pm 19.9$ & $\mathbf{25.8 \pm 23.1}$ & $\mathbf{17.5 \pm 15.3}$ &$ 6.49 \pm 4.5$ \\ \hline

NMPC & $\mathbf{23.6 \pm 17.7}$ & $\mathbf{24.1 \pm 19.0}$ & $\mathbf{13.4 \pm 10.1}$ & $\mathbf{5.1 \pm 3.9}$ & $\mathbf{8.3 \pm 6.5}$ & $\mathbf{18.2 \pm 15.0}$ & $27.7 \pm 21.3$ & $18.4 \pm 15.7$ & $\mathbf{5.3 \pm 5.6}$  \\ \hline\hline
\end{tabular}
\end{table*}

\begin{figure*}
  \centering
  \subfigure[Marshmallow]{\label{Fig_min_Marsh}
  \includegraphics[height=3.1cm]{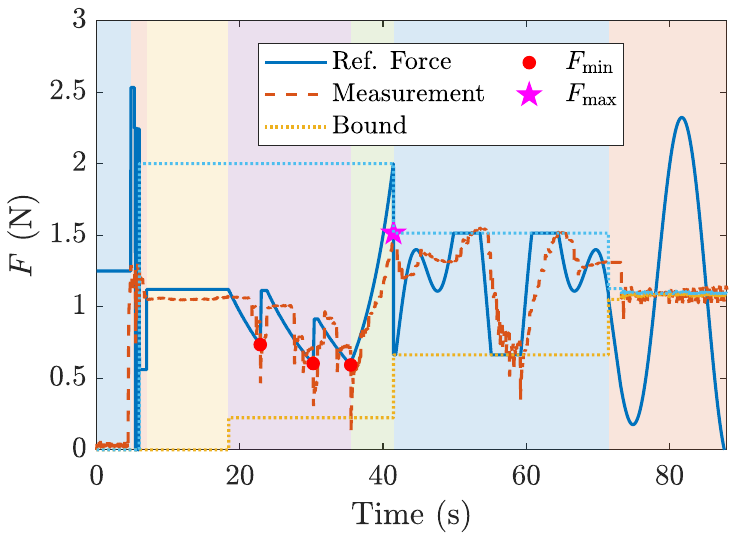}}
  \hspace{-3mm}
  \subfigure[Tofu]{\label{Fig_min_Tofu}
  \includegraphics[height=3.1cm]{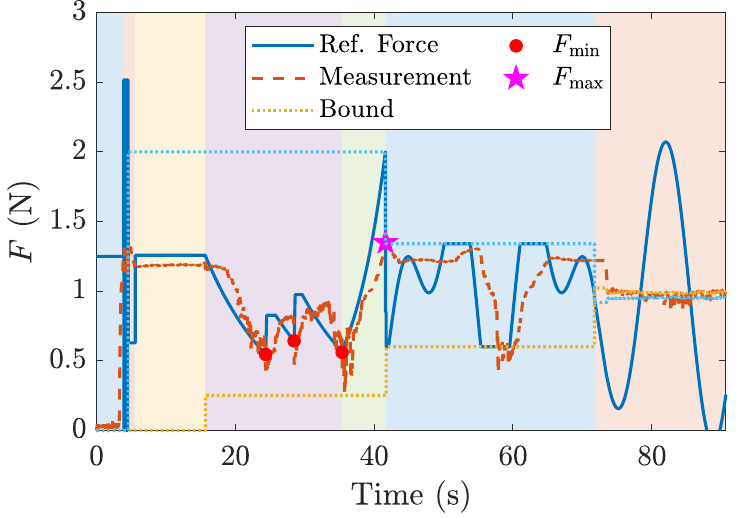}}
  \hspace{-3mm}
  \subfigure[Gripper Width, Tofu]{\label{Fig_min_Tofu1_x}
  \includegraphics[height=3.1cm]{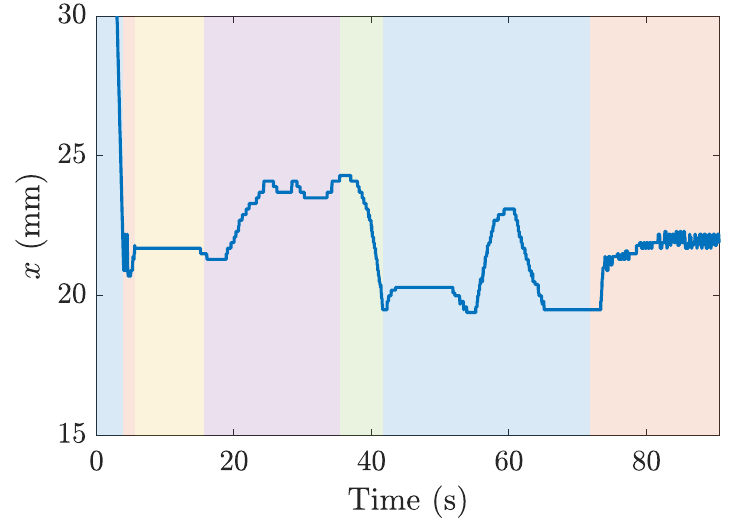}}
  \hspace{-3mm}
  \subfigure[Heavy Copper, $798.5$~g]{\label{Fig_min_Copper2}
  \includegraphics[height=3.1cm]{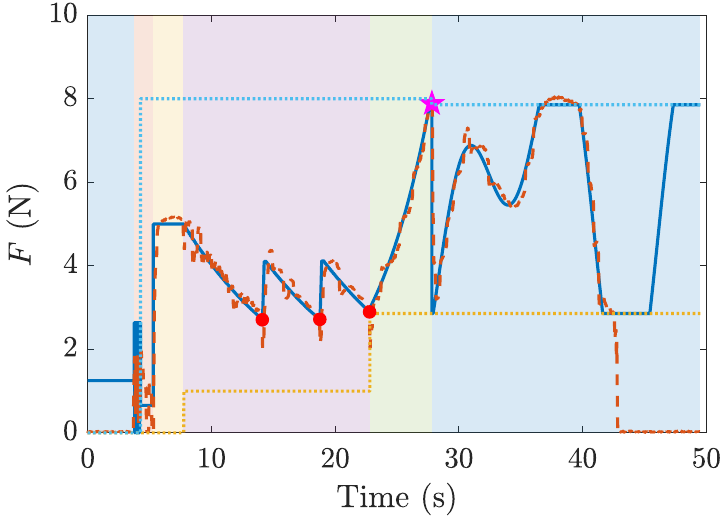}}

  \subfigure[Tofu with $\beta = 0$ ]{\label{Fig_NMPC_Tofu}
  \includegraphics[height=3.15cm]{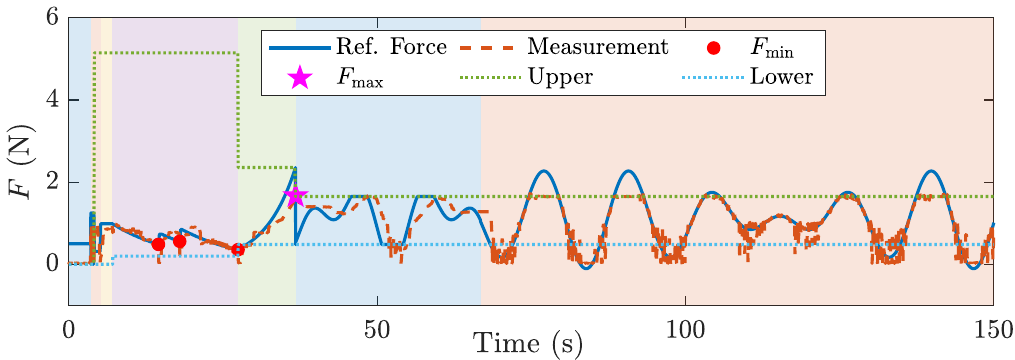}}
  \hspace{-3.5mm}
  \subfigure[Business Card, $0.3$~mm thick]{\label{Fig_NMPC_Card}
  \includegraphics[height=3.15cm]{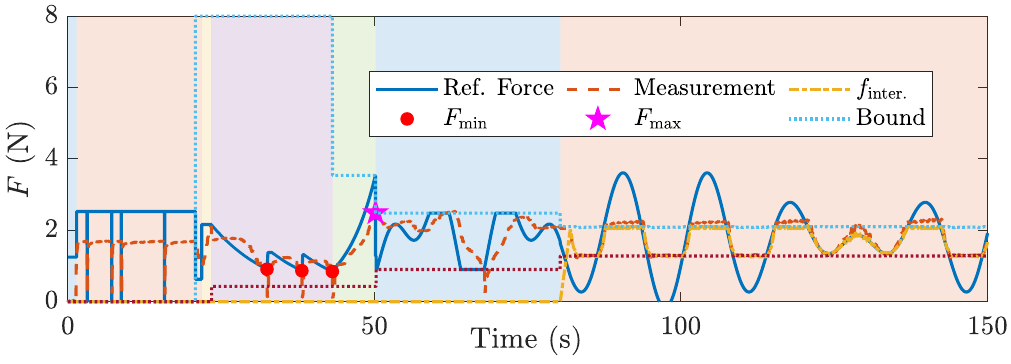}}
  \caption{Dynamic grasping of soft, hard, and permanently-deformable objects. (a) Marshmallow. (b) Tofu. (c) Gripper width when grasping tofu. (d) Heavy copper block. (e) Tofu with $\beta=0$. (f) Business card.}
  \label{Fig_Special_Case}
  \vspace{-3mm}
\end{figure*}

For the soft tomato and silicone block, both the lower and upper grasping force limits are shifted inward, leading to a narrow operating range. This result reflects the delicate nature of these two objects. The proposed online learning and MPC strategy achieve superior performance. The model learning error is less than $0.3$~N for the tomato (see Fig.~\ref{Fig_NMPC_Tomato_Error}) and less than $0.25$~N for the silicone block (see Fig.~\ref{Fig_NMPC_Silicon_Error}). The force tracking error under MPC control is less than $0.3$~N for the tomato and less than $0.2$~N for the silicone block. Noticeably, the error for the soft object is smaller than that for the hard object. This is because a hard object possesses high stiffness, which restricts the displacement of the gripper. As shown in Fig.~\ref{Fig_NMPC_AL_Error}, the gripper width variation is less than $1.3$~mm for the aluminum block, compared to $3$~mm of the silicone block in Fig.~\ref{Fig_NMPC_Silicon_Error}. The limited displacement causes lower data variance and high stiffness may induce potential overshoot.

The maple seed ($m=0.2$~g) represents a class of hard yet brittle objects and large force can cause catastrophic fracture. In Fig.~\ref{Fig_NMPC_Maple}, the high evaluated stiffness leads to an $8$~N maximum force. However, after verification, the constraints become more conservative, and $F_{\text{obj},\max}$ drops to $4$~N. Under line contact, the gripper width displacement remains within $1.6$~mm upon secure contact. The error is less than $0.35$~N for model learning error and  $0.6$~N for control. The relatively small modeling error occurs because the silicone pads on the fingers are compressed under line contact mode.

Due to page constraints, experimental results under the benchmark PID control and linear MPC are omitted here. However, in Table~\ref{Tab_Eror}, we summarize the error statistics ($|F_x^d-F_x(\bm s)|$) for dynamics-free PID control (with reference force filtering applied), linear MPC, and nonlinear MPC. All experiments were conducted starting from the \texttt{approaching} phase, and identical parameters were used across all controllers and objects. Overall, the errors under PID control are the largest. In most cases, the nonlinear MPC achieves the best performance. In particular, for soft objects such as the tomato and silicone block, material elasticity and properties vary under pressure or as compression becomes firm. PID control ignores these dynamic properties, resulting in tracking errors of $0.24$~N for the tomato and $0.37$~N for the silicone block, representing a $3\times$ and $7\times$ increase compared to the proposed controller. Linear MPC also outperforms PID control in these cases. For hard objects, such as the aluminum and copper blocks, object stiffness remains nearly constant within the applied force range; thus, linear and nonlinear MPC exhibit similar performance. Overall, control errors under MPC are smaller than those under preliminary PID control, as the latter lacks system dynamics modeling and, more critically, does not enforce safety constraints.

\section{Discussions}
\label{Sec_discussion}

In this section, we present results for several special cases. We further analyze the computational cost, parameter selection, and limitations of the framework.

\subsection{Special  Cases}

\subsubsection{Delicate Objects}


For delicate objects, secure and gentle grasping is of critical importance. For instance, tofu is an extremely soft material, and human grasping frequently results in structural damage. As reported in~\cite{Xu2024}, computationally expensive vision-tactile array sensing and control strategies were unable to successfully grasp soft objects such as boiled eggs or raw meat. Here, we consider the even more challenging tasks of grasping an extremely soft marshmallow and raw tofu.

\begin{figure}
  \centering
  \subfigure[]{
  \label{Fig_Tofu_dent}
  \includegraphics[height=2cm]{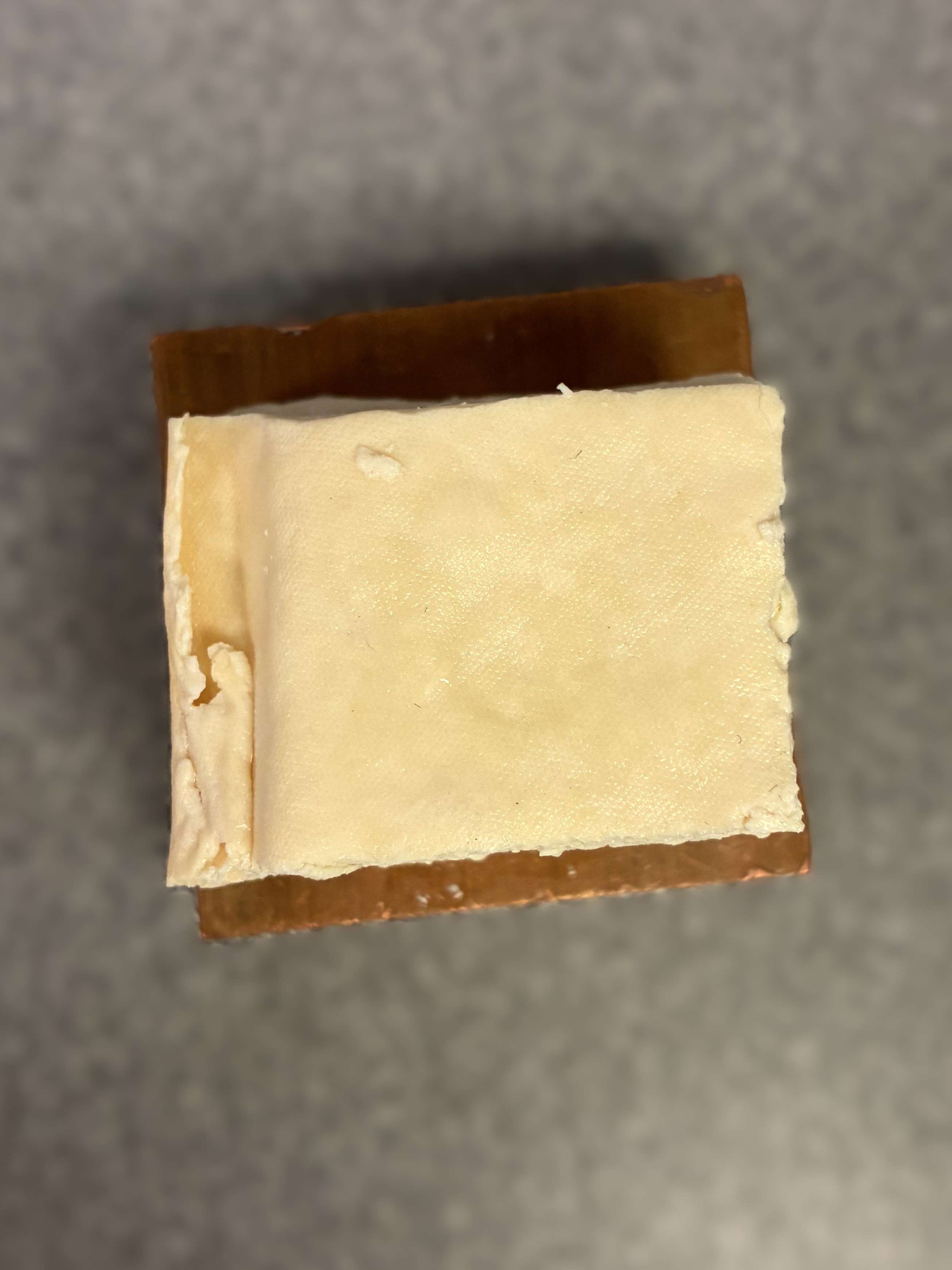}}
  \hspace{-3mm}
  \subfigure[]{
  \label{Fig_Marsh_dent}
  \includegraphics[height=2cm]{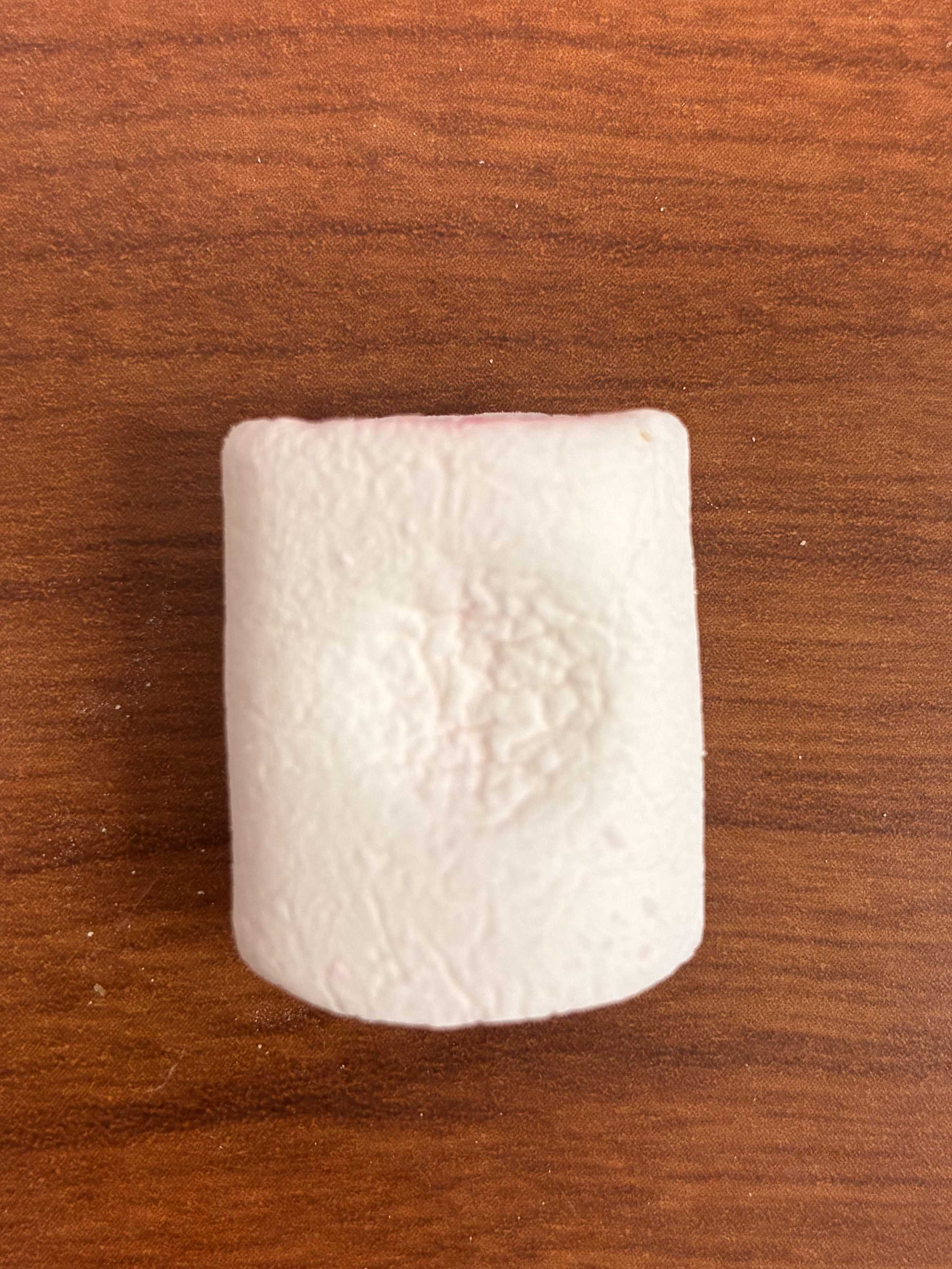}}
  \hspace{-3mm}
  \subfigure[]{
  \label{Fig_Card_Snap}
  \includegraphics[height=2cm]{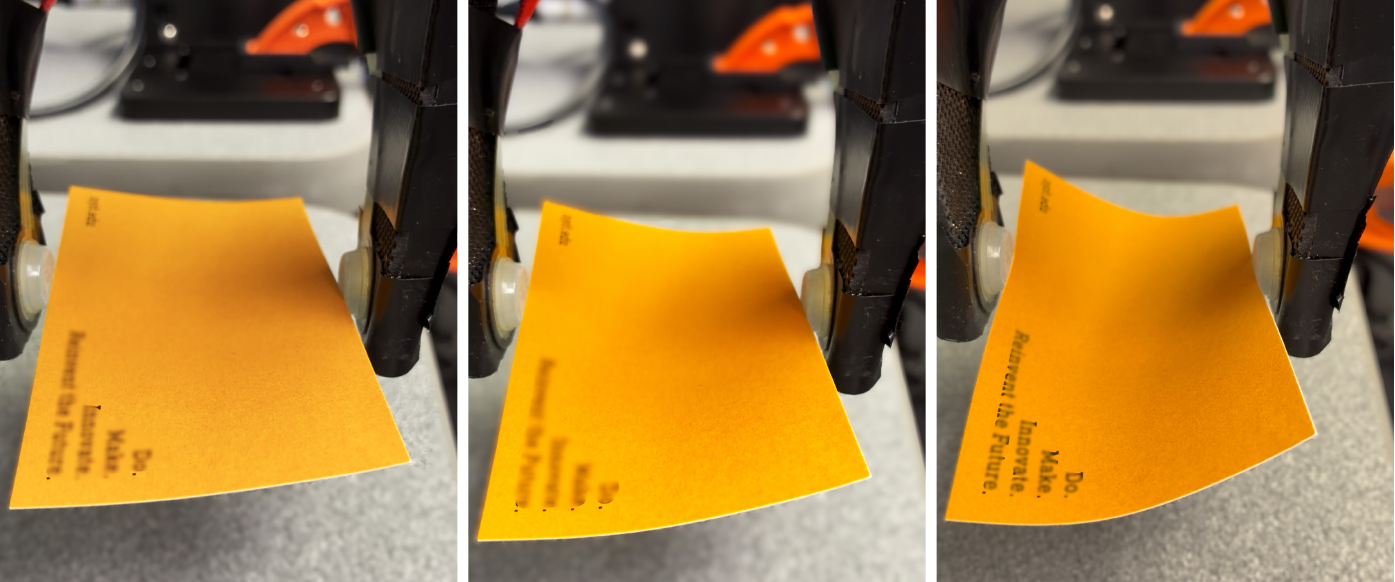}}
  \hspace{-3mm}
  \subfigure[]{
  \label{Fig_Card_Bend}
  \includegraphics[height=2cm]{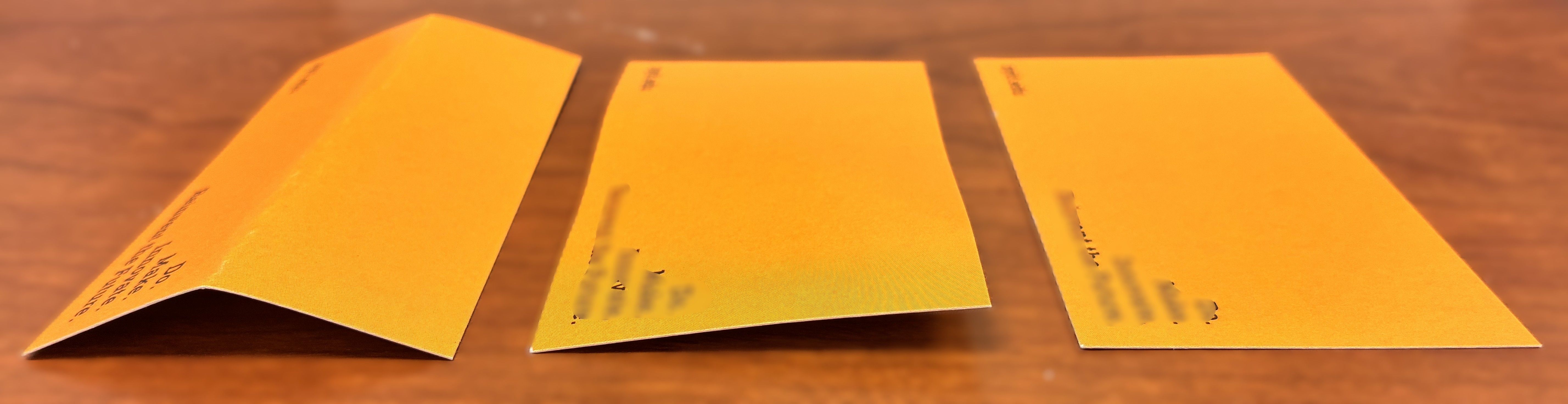}}
  \vspace{-2mm}
  \caption{The dent on the soft object (a) tofu and (b) marshmallow. (c) The snapshots of dynamic grasping of a business card when the applied force is increasing; see Fig.~\ref{Fig_NMPC_Card}. (d) The deformation of the card. From left to right: failed one with permanent deformation, the one under proposed control, and the original status. }
    \label{Fig_Special_Case}
    \vspace{-5mm}
\end{figure}

Fig.~\ref{Fig_min_Marsh} and Fig.~\ref{Fig_min_Tofu} show the robotic grasping of a marshmallow ($m=2.21$~g, $\kappa_{\text{static}}=89$~N/m) and raw tofu ($m=38$~g, $\kappa_{\text{static}}=20$~N/m). In Table~\ref{Tab_Objects}, the evaluated stiffnesses of the tofu and marshmallow are the lowest among all tested objects and fall below the lower limit ($\kappa_{\min}=200$), which violates assumption $\mathcal{A}_2$. Nevertheless, we demonstrate that the proposed approach can still successfully handle these delicate objects. Because the tactile sensor requires sufficient deformation to trigger activation, the marshmallow must be compressed somewhat firmly to be lifted (see the actual width versus measurement width in Table~\ref{Tab_Objects}). This explains why the evaluated stiffness of the marshmallow exceeds that of the tofu. To lift the extremely soft and light marshmallow, a force of $0.7$~N is required; see Fig.~\ref{Fig_min_Marsh}. Similarly, for the tofu, the minimum grasping force is $0.5$~N. For these two delicate objects, upper-bound verification becomes critically necessary. As shown in Fig.~\ref{Fig_min_Marsh} and Fig.~\ref{Fig_min_Tofu}, the proposed algorithm clips the stiffness to the lower bound of $200$, yielding an initial temporary maximum force of $2$~N, which caused catastrophic structural failure during experiments conducted without maximum force verification. In contrast, with verification active, as the grasping force increases toward $2$~N (around $t=40$~s), the gripper squeezes progressively more aggressively with positive acceleration (see Fig.~\ref{Fig_min_Tofu1_x}). However, the measured grasping force stops increasing. The identified maximum force ($1.3$~N) thus prevents potential damage. Meanwhile, the delicate and fragile nature of tofu and marshmallows introduces substantial uncertainties into the interaction dynamics learning, narrowing the safe operation region and leaving virtually no operational margin for grasping force adjustments. After secure grasping, a small indentation remains on the object (prior to elastic recovery); see Figs.~\ref{Fig_Tofu_dent} and~\ref{Fig_Marsh_dent}. To evaluate the performance of interaction dynamics learning and MPC control, we deliberately set $\beta=0$ and omit model uncertainties in constraint evaluation. Fig.~\ref{Fig_NMPC_Tofu} shows the dynamic grasping of the tofu, where the maximum grasping force is capped at approximately $1.7$~N following verification. The interaction force successfully tracks the reference while respecting the constraints, despite the extreme delicacy of the object. However, larger tracking errors occur when reaching the lower bound. This is because reducing the interaction force requires increasing the gripper width; for soft objects, a steady reaction force relies on maintaining a sufficiently firm engagement.

\subsubsection{Heavy Objects}


For heavy objects with smooth surfaces, the lower force bound plays a significant role in ensuring secure grasping. The weight ($m=798$~g) of the copper2 block falls outside the defined range of $(0, 500]$~g. In Fig.~\ref{Fig_min_Copper2}, the minimum required force is determined to be $3.2$~N, which is used as a constraint to saturate the reference force during PID control. However, PID control in principle does not enforce such hard constraints. Upon external disturbance and oscillation (induced by manually pushing the robotic arm) near the lower bound (see $t=42$~s), slippage occurs and contact is lost. These results highlight the necessity of constraint-enforced control. While tuning PID control gains might yield improved performance for a known object, an overly aggressive PID controller can easily cause catastrophic failures when handling delicate objects.

\subsubsection{Permanently Deformable Objects}


For certain rigid objects, high stiffness and elasticity are present along a specific grasping direction. However, if the grasping force becomes sufficiently large, the object undergoes permanent deformation, causing the reaction force to drop abruptly. Examples of such objects include thin cardboard, thin plastic cards, and thin aluminum or copper films. Fig.~\ref{Fig_NMPC_Card} shows the dynamic grasping of a business card ($0.1$~mm thickness, $m=1.2$~g), and Fig.~\ref{Fig_Card_Snap} presents snapshots of the experiment. The evaluated stiffness reaches $886$, categorizing the card as a medium-hard object. Consequently, the maximum grasping force prior to verification is calculated as $8$~N. Upon verification, the card begins to yield, and the reaction force increases only marginally despite aggressive reduction of the gripper width, similar to the response observed in Fig.~\ref{Fig_min_Tofu1_x}. Through this verification process, the maximum allowable force drops to a safe level of $2.6$~N. During the experiment, the card undergoes elastic deformation, with a control error of less than $0.2$~N. Compared to its initial state, only slight bending occurs; in contrast, without verification, the card under control suffers permanent plastic deformation, as shown in Fig.~\ref{Fig_Card_Bend}.

\begin{figure}
  \centering
  \includegraphics[width=8.5cm]{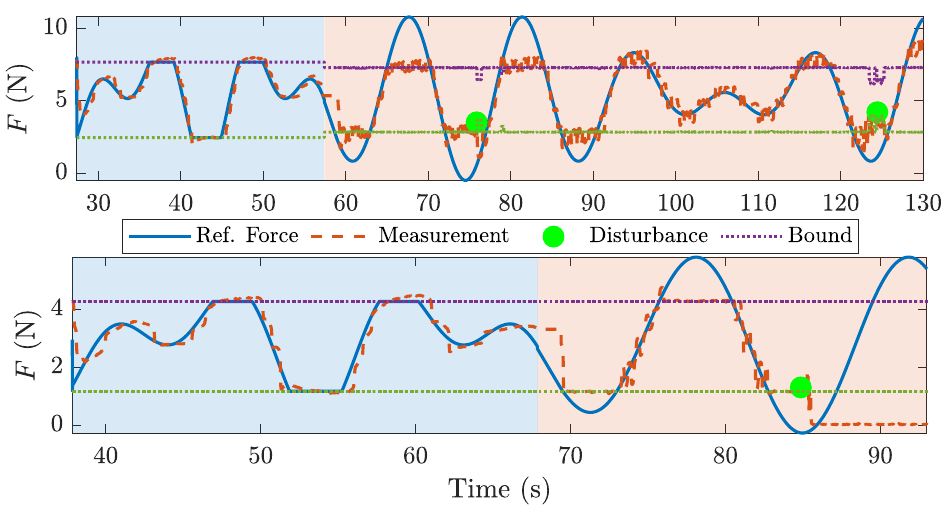}
  \vspace{-2mm}
  \caption{Demonstration of robust grasping under external disturbances. Upper: constraints with GP prediction uncertainties. Lower: constraints without GP prediction uncertainties.}
  \label{Fig_STD}
  \vspace{-2mm}
\end{figure}

\subsection{Robust Secure Grasping}


The proposed control under robust constraint enforcement demonstrates great potential for rejecting disturbances during object handling and grasping. To show the effectiveness of robust constraints, Fig.~\ref{Fig_STD} illustrates dynamic grasping under external disturbance. The robotic gripper grasps the aluminum block, and we strike the block using a rod to introduce disturbances. As shown in Fig.~\ref{Fig_STD}, when the grasping force reaches the lower bound, we apply the disturbances. The force bound varies in response to these external disturbances, and the grasping force accordingly increases to maintain contact. Conversely, under a constant grasping force limit, the aluminum block drops, resulting in failure. The secure-grasping constraint is identified online, and for the same object, this lower bound may vary dynamically. The robust constraints compensate for such variations caused by environmental uncertainties.

\begin{figure}[h]
  \centering
  \includegraphics[width=8.5cm]{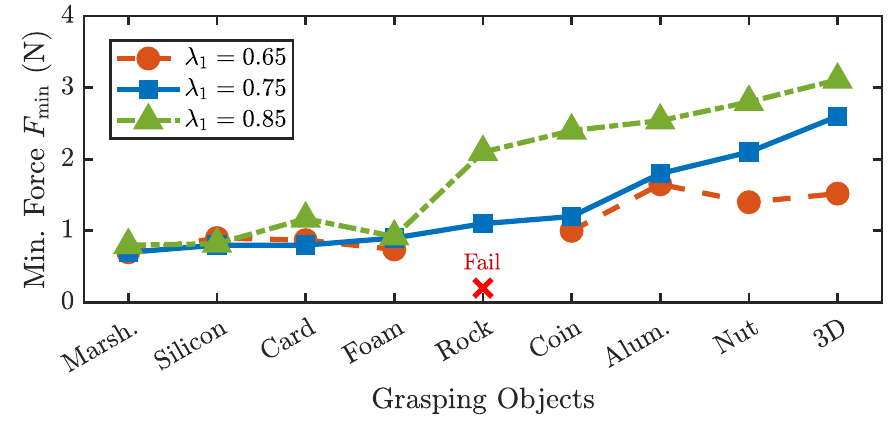}
  \vspace{-3mm}
  \caption{Minimum grasping force identified under different slip parameters. }
  \label{Fig_Fmin_Slip}
  \vspace{-3mm}
\end{figure}

\subsection{Parameters Selections}

The proposed method requires some pre-defined constants. The effects of some constraints are validated or corrected through online trials. For instance, the maximum force calculation in~\eqref{Eq_F_max_temp} depends on the user-defined parameters $\epsilon_{\max},  \xi $, which may generate a large force for a soft object. However, such a calculation is corrected online through maximum force online verification (see~\eqref{Eq_F_max_temp}). While out of many constants, the slip coefficient $\lambda_1$ in~\eqref{Eq_slip_event_trigger} is the most critical parameter.  A smaller $\lambda_1$ indicates a riskier grasping, and a larger $\lambda_1$ may result in over-conservativeness. During implementation, we used $\lambda_1=0.75$. To investigate these parameters on system performance. In Fig.~\ref{Fig_Fmin_Slip}, we show the minimum force identified online under different $\lambda_1$ values for some representative objects. For a more risky $\lambda_1$, the overall lower force bound becomes smaller. However, when the gripper-object contact topology becomes irregular, under slip, the contact may lose compared to a surface contact. For instance, when $\lambda_1=0.65$, for the rock with irregular-contact topology, the grasping fails (most time) due to loss of contact.

Similarly, the PID control also plays a significant role. An aggressive PID control may cause the identified constraints to be either more conservative or more risky. Meanwhile, under the assumption of unknown object properties, the PID controllers cannot be designed to be object-specific. To avoid an unrealistic estimation of the contact force bound. During the minimum force hunting and maximum force verification phases, the rate of the reference force is designed to be $2\%$,  followed by a rest period before the next increment/decrement. This reduces the influence of the PID control's transient performance.

\subsection{Computational Cost}

By leveraging computationally efficient mechatronics and modern computing resources, the computational cost of this work is mitigated through the following three aspects.

First, compared to~\cite{Romano2011TRO, Xu2024, kang2026Force, Yin2022Multimodal}, the proposed tactile-sensing mechatronics does not require an array of tactile sensor units. For each finger, only a single single-channel unit is needed, which significantly reduces latency and the overhead associated with lower-level raw data processing and transmission. In~\cite{Romano2011TRO}, multiple three-dimensional accelerometers are used to construct a tactile sensor array, yielding a fingertip force measurement frequency of only $5$~Hz. In contrast, the tactile force measurement in this work achieves a real-time sampling frequency exceeding $1000$~Hz.

Second, compared to visuo-tactile sensors, the proposed tactile-sensing system eliminates the need for computationally expensive computer vision or image processing. In~\cite{Tian2019Manipulation}, video processing techniques are used to handle tactile sensing. As a result, the control frequency of the MPC reaches only $1$~Hz, which is inadequate for real-time force control. In~\cite{Xu2024}, a visuo-tactile system is employed for robotic grasping. To reduce computational cost, state prediction is replaced by a neural network model; nevertheless, the overall control frequency remains low at $50$~Hz.

The primary computational burden of the proposed approach stems from online interaction dynamics learning and state prediction within the nonlinear MPC. A large data depth $d$ and dataset size $M$ (we set $M=1500$) in~\eqref{Eq_input_feature_vector} place heavy demands on online model training. In MPC optimization, both mean and variance evaluations are required, and constraints vary over the prediction horizon. To address this, we deploy GP model training, evaluation, and MPC optimization on a GPU by leveraging its parallel computing capabilities. Consequently, a real-time control frequency of $100$~Hz is achieved with no observable delay. The linear MPC-based control can also be executed on a resource-constrained machine (Intel Core i7-1260P x 16) at $50$~Hz, albeit with a smaller dataset size and a 10-second latency during GP model training. Importantly, under the ROS protocol, the grasping force command from the preceding step persists during this 10-second delay, preventing system failure. Therefore, the proposed method remains feasible even in computationally constrained setups.

\subsection{Limitations}

There are several limitations of the proposed approach compared to the work in the literature as follows.
\begin{itemize}



  \item One of the major limitations of the proposed work is the assumption that the object initially resides within the gripper's workspace. Under $\mathcal{A}_1$, the proposed framework does not require a vision system for object search. However, in practice, the object may lie outside the workspace of the gripper, necessitating active object search.

  \item Furthermore, for irregular objects, the grasp location on the object is critical. Even when the object is within the gripper's workspace, contact should be established at a location that maximizes the contact area. The contact area may shrink during online minimum force hunting, potentially leading to a loss of contact. Currently, the proposed algorithm lacks a module to adjust contact locations dynamically.

  \item Compared to~\cite{Xu2024}, the tactile sensing module measures only normal forces, leaving friction forces and torques unknown. Under this setup, the identified minimum force depends on user-defined parameters (e.g., $\lambda_1$), which may not closely reflect the ground truth.

\end{itemize}


Future work should focus on addressing these limitations. In particular, to relax $\mathcal A_1$, tactile fingers could be utilized to search for the object and reconstruct its three-dimensional surface contour. Additionally, the initial contact establishment can be optimized.

\section{Conclusions}
\label{Sec_conclusion}

This work proposes a tactile sensing-based dynamics learning and control framework for the robotic grasping of unknown objects. The framework encompasses multiple modules, including compact intelligent tactile mechatronics, object property probing, constraint identification, online grasping dynamics learning, and predictive control. The tactile sensing system measures gripper-object interaction forces, while the object probing module, resembling human grasping behavior, actively evaluates the object's stiffness, size, slip, and deformation. Fundamentally, we identify case-specific intrinsic constraints on the grasping force to achieve secure and gentle grasping of unknown objects. The interaction dynamics are formally learned online using GP regression, and predictive control is applied to dynamically regulate the grasping force while enforcing these constraints. The system safely grasps objects as delicate as raw tofu and marshmallows, as well as objects as rigid as metal blocks. Crucially, the proposed grasping system operates without offline training for object identification, complex tactile mechatronics, or vision systems, enabling computationally efficient control for real-time applications. Extensive experimental results validate the proposed framework across diverse scenarios. In the future, we plan to pursue two major directions: leveraging the tactile fingers to reconstruct 3D surface contours for optimizing contact establishment, and extending the proposed approach to high-DoF multi-fingered dexterous hands integrated with multi-axis tactile arrays to map complex shear forces and torsional sliding moments.

\bibliographystyle{IEEEtran}
\bibliography{HanRef}

\end{document}